\documentclass[twoside,11pt]{article}

\usepackage{blindtext}

\usepackage{amsmath}
\usepackage{anyfontsize}
\usepackage{tikz}
\usepackage{xcolor}
\usepackage{colortbl}
\usepackage{titletoc}
\usepackage{float}
\usepackage{tcolorbox} 
\usepackage{listings}

\definecolor{customredlight}{HTML}{F05039}
\definecolor{customred}{HTML}{F5977D}
\definecolor{custombluelight}{HTML}{3D65A5}
\definecolor{customblue}{HTML}{7A92C9}

\newcommand{\ltwo}[1]{\lVert #1 \rVert_2}

\newcommand{\R}{\mathbb{R}}
\newcommand{\C}{\mathbb{C}}

\definecolor{ocra}{RGB}{193, 143, 50}
\definecolor{boldcolor}{gray}{0}
\newcommand{\lightbold}[1]{\textbf{\textcolor{boldcolor}{#1}}}

\usepackage[normalem]{ulem}

\usepackage{jmlr2e}
\usepackage{lastpage}
\usepackage{subcaption}

\ShortHeadings{A Neural Kernel Theory of Symmetry Learning}{A. Perin and S. Deny}
\firstpageno{1}

\begin{document}

\jmlrheading{26}{2025}{1-\pageref{LastPage}}{12/24; Revised
5/25}{6/25}{24-2175}{Andrea Perin and St\'ephane Deny}

\title{On the Ability of Deep Networks to Learn Symmetries from Data -- A Neural Kernel Theory}

\author{\name Andrea Perin \email andrea.perin@aalto.fi \\
       \addr Department of Computer Science\\
       Aalto University\\
       Espoo, Finland
       \AND
       \name St\'ephane Deny \email stephane.deny.pro@gmail.com \\
       \addr Department of Computer Science\\
       Department of Neuroscience and Biomedical Engineering\\
       Aalto University\\
       Espoo, Finland}

\editor{Aapo Hyvärinen}

\maketitle

\begin{abstract}Symmetries (transformations by group actions) are present in many datasets, and leveraging them holds considerable promise for improving predictions in machine learning. In this work, we aim to understand when and how deep networks---with standard architectures trained in a standard, supervised way---\emph{learn} symmetries from data. Inspired by real-world scenarios, we study a classification paradigm where data symmetries are only \emph{partially observed} during training: some classes include all transformations of a cyclic group, while others---only a subset. We ask: under which conditions will deep networks correctly classify the partially sampled classes?\smallskip 

In the infinite-width limit, where neural networks behave like kernel machines, we derive a \emph{neural kernel theory of symmetry learning}. The group-cyclic nature of the dataset allows us to analyze the Gram matrix of neural kernels in the Fourier domain; here we find a simple characterization of the generalization error as a function of class separation (signal) and class-orbit density (noise). This characterization reveals that generalization can only be successful when the local structure of the data prevails over its non-local, symmetry-induced structure, in the kernel space defined by the architecture. This occurs when (1) classes are sufficiently distinct and (2) class orbits are sufficiently dense.\smallskip 

We extend our theoretical treatment to any finite group, including non-abelian groups. Our framework also applies to equivariant architectures (e.g., CNNs), and recovers their success in the special case where the architecture matches the inherent symmetry of the data. Empirically, our theory reproduces the generalization failure of finite-width networks (MLP, CNN, ViT) trained on partially observed versions of rotated-MNIST. We conclude that conventional deep networks lack a mechanism to learn symmetries that have not been explicitly embedded in their architecture \emph{a priori}. In the future, our framework could be extended to guide the design of architectures and training procedures able to learn symmetries from data. \smallskip 

All code is available at \href{https://github.com/Andrea-Perin/gpsymm}{\texttt{https://github.com/Andrea-Perin/gpsymm}}.
\end{abstract}


\section{Introduction}
The ability to make accurate predictions depends on how well one understands the structure of a problem. Physics, in particular, has achieved remarkable success in predicting natural phenomena by capturing their structure in compact mathematical equations. A key aspect of this structure lies in the concept of \emph{symmetry} \citep{Noether1918,gross1996role}. Symmetries describe transformations by group actions that do not affect the identity of objects. For example, a chair remains a chair whether it is presented upright or upside down (a transformation in $SO(3)$). In the context of deep learning, nowadays widely used for prediction tasks, a crucial question is whether deep networks have mechanisms to recognize and exploit data symmetries for effective predictions. For example, can deep networks learn to predict the identity of objects independently of their viewpoint? And importantly, do they need to be exposed to \emph{all object classes in all possible poses} during training, or is exposure to some classes in some poses sufficient for them to \emph{capture the concept} of pose invariance?

\medskip

The field of geometric deep learning develops theories and methods that enable neural networks to take advantage of problem symmetries \citep{bronstein2021geometricdeeplearninggrids}.  In particular, equivariant neural networks \citep{cohenICML2016} can ensure representation equivariance and classification invariance to prespecified symmetries. However, equivariant architectures require one to know the symmetry of the problem in advance. Here, we aim to understand whether conventional deep networks---which have not been explicitly designed to capture a prespecified symmetry---can \emph{learn symmetries directly from data}.

\medskip

We focus on a supervised classification paradigm where the symmetries of the data are only partially observed during training : for some classes, all possible transformations of a cyclic group are observed during training, while for other classes only a subset of transformations is observed. This scenario replicates realistic real-world learning problems. For example, a child during their development sees a few objects in all possible 3D poses (e.g., the toys they can manipulate), and many objects in only some poses (e.g., heavy furniture). \emph{Should we expect a deep network trained on such a data diet to generalize to the partially sampled classes (e.g., recognize a piece of furniture seen from an unusual viewpoint at test time)?} In this work, we study the conditions under which deep networks correctly extrapolate the symmetry invariance to the partially sampled classes.

\medskip

To characterize the generalization capabilities of deep networks in the presence of data symmetries, we rely on the equivalence between infinitely wide neural networks and kernel machines \citep{Neal1996, lee2018deep, jacotNIPS2018}. Typical neural kernels (MLP, CNN) greatly simplify when computed over a dataset generated by a cyclic group action (they become circulant), allowing an interpretable analysis in the Fourier domain of when and how symmetries are correctly learned. \emph{We find that the generalization behavior of networks is predicted by a simple ratio of inverse kernel frequency powers computed over orbits of the cyclic group.}  
Our analysis of this formula makes clear that deep networks (as described by their kernel equivalent) are \emph{a priori} unable to leverage data symmetries for generalization. Successful generalization is yet possible, in cases where the local structure of the data prevails over its non-local, symmetry-induced structure. This happens in particular (1) when classes are sufficiently well separated in kernel space, and (2) when the symmetric structure of the data is sufficiently local in kernel space. Importantly, while there is no guaranteed equivalence between finite-width networks and their infinite-width counterpart, our spectral kernel theory captures well the behavior of normally trained finite-width networks in all our experiments on rotated-MNIST, a version of the well-known handwritten character dataset \citep{mnist} augmented with rotations.
\medskip

\paragraph{Outline.} In Section~\ref{sec:symmetry_learning}, we briefly review prior theoretical work and practical methods developed for deep learning in relation to symmetries. In Section~\ref{sec:illustration}, we illustrate a simple symmetry learning problem, where deep networks with common architectures (MLP, CNN, ViT) are trained and evaluated on partial views of rotated-MNIST. In Section~\ref{sec:theory}, we analyze theoretically how kernel machines in general---and neural networks in particular---behave on datasets presenting symmetries. We build our theory through scenarios of increasing complexity, from a simple Gaussian kernel applied to a circular dataset, to deep neural kernels with or without equivariant architectures (MLPs and CNNs) applied to an affine group transformation such as rotated-MNIST. In all cases, we show both theoretically and empirically the inadequacy of conventional neural architectures trained with supervision to learn symmetries that have not been embedded in their kernel design \emph{a priori}. In Section~\ref{sec:discussion}, we discuss how our work provides theoretical tools which could be helpful in identifying and discovering architectures and training procedures able to learn symmetries from data. In Appendix \ref{app:nonabelian}, we extend our theory to any finite group, including non-abelian groups.

\section{Prior Work on Deep Learning, Symmetries, and Neural Kernels}\label{sec:symmetry_learning}

The interplay between deep learning and symmetries has been studied extensively, both empirically and theoretically, for finite-width networks and in the infinite-width limit (where kernel analogies apply). Here we attempt a condensed review of these efforts. In the following, our operational definition of \emph{dataset symmetries} is a set of transformations by group action that do not change the class identity of objects present in a dataset. These group transformations may act directly in the native space of the dataset (e.g., image translations) or in a latent space affecting the dataset (e.g., images of 3D-rotating objects).\medskip

\lightbold{Empirically, the inability of conventional neural networks to capture symmetries present in datasets has been observed in many different contexts}. Conventional networks have been shown to fail to extrapolate a simple periodic function \citep{ziyin2020neural}. In vision, studies have investigated the generalization capabilities of deep networks to recognize objects undergoing changes in pose \citep{strikewithapose,madan2022ood, Abbas_Deny_2023, siddiqui2023investigatingnature3dgeneralization}, size \citep{ibrahim2023robustness}, mirror symmetry \citep{sundaram2022symmetry}, lighting conditions \citep{madan2021small}, and shown a substantial degradation of network performance in these conditions. Recently, \cite{ollikka2025humans} also showed that humans beat state-of-the-art deep networks and most vision-language models at recognizing objects in unusual poses. In language, compositional generalization has also been framed as capturing permutation symmetries, which traditional language models have been shown to fail at. \emph{In our work, we study from a theoretical point of view why networks fail to generalize on data symmetries despite being exposed to them partially during training. To do so, we characterize network behavior on symmetric datasets in the kernel limit.}\medskip

\lightbold{Network architectures have been designed to be equivariant to prespecified symmetries} \citep{gensNIPS2014,cohenICML2016, Worrall2017, cohen2019gauge, cohen2020generaltheoryequivariantcnns, bekkers2021bsplinecnnsliegroups,weiler2021coordinateindependentconvolutionalnetworks}, or conserve certain physical quantities \citep{greydanus2019hamiltonian,cranmer2020lagrangian,finzi2020generalizing, vanderouderaa2024noethersrazorlearningconserved}. Network architectures have also been designed to respect the topology of a problem \citep[for a review, see][]{hajij2023topologicaldeeplearninggoing}. Relaxed and adaptive equivariance schemes have also been proposed \citep{elsayed2020revisiting,zhou2021metalearning,d2021convit,wang2022approximately,yehPMLR2022,kaba2023symmetry,vanderouderaa2023learninglayerwiseequivariancesautomatically}. The failure of traditional convolutional architectures to be fully equivariant has been observed \citep[e.g.,][]{azulay2019deep}, and some measure of approximate equivariance proposed \citep{gruver2023the}. Others have investigated how equivariant representations affect the capacity of group-invariant linear readouts \citep{Farrell2021}. Beyond equivariance, symmetries inherent to the architecture of deep networks have also been studied, and shown to affect their learning dynamics and solutions \citep{tanakaNIPS2021, pmlr-v139-simsek21a,ainsworth2023git}. \emph{In our work, we study the ability of neural networks to generalize on datasets presenting symmetries to which they have \emph{not} been designed to be equivariant.}\medskip

\lightbold{The effects of training networks on symmetry-augmented versions of a dataset have also been scrutinized}. Recently, the efficiency of training networks on augmented data vs. enforcing equivariance in the architecture has been carefully characterized at scale \citep{brehmer2024doesequivariancematterscale} and on a variety of datasets \citep{vadgama2025probingequivariancesymmetrybreaking}. \cite{moskalevPMLR2023a} clarify that training on augmented datasets does not produce genuine equivariance, in the sense that the trained networks may not be equivariant outside the training distribution. Recent theoretical results show, however, that training on a perfectly augmented dataset produces emergent equivariant representations in \emph{ensembles} of networks, in and outside the training set, both in the finite \citep{nordenfors2024ensemblesprovablylearnequivariance} and infinite-width limit \citep{gerkenPMLR2024a}. \emph{In our work we ask a related but different question: we study the generalization capabilities of networks trained on a dataset containing unknown symmetries---where no systematic augmentation scheme or equivariant architecture has been enforced. The symmetries in the training dataset may be well represented for some classes, but not others. The question we ask is then: how well will networks generalize the symmetry invariance to the partially sampled classes?}\medskip

\lightbold{Beyond equivariant and augmentation methods to deal with prespecified symmetries, approaches to learning the symmetries present in data have been proposed} \citep{culpepper2009learning, jaderberg2015spatial, sohl2010unsupervised, dupont2020equivariant,benton2020learnedaugmentations, connor2020representing, zhou2021metalearning, kellerNEURIPS2021, rey2023equivariantrepresentationlearningpresence, sanborn2023bispectral, connor2024learning,   vanderlinden2024learningsymmetriesweightsharingdoubly,yang2024latentspacesymmetrydiscovery} \lightbold{and studied theoretically} \citep{Anselmi2019,pfauNEURIPS2020, anselmi2023data,marchetti2024harmonicslearninguniversalfourier}.  Self-supervised learning approaches build invariance \citep{ zemelNIPS1990,chen2020simple,zbontar2021barlow,ibrahim2022robustselfsupervisedlearninglie} or equivariance \citep{garrido2024learningleveragingworldmodels} to predefined symmetries by feeding augmentations of the dataset to two identical versions of the network and matching their latent representations. Disentanglement has also been framed as the problem of learning symmetries from data \citep{higgins2018definitiondisentangledrepresentations,mercatali2022symmetry}. Subsequent work has shown, however, that topological defects result from attempting to learn disentangled representations of even the simplest symmetries such as affine transformations \citep{bouchacourt2021addressingtopologicaldefectsdisentanglement,esmaeili2023topologicalobstructionsavoid}. \emph{In our work, we focus on characterizing the symmetry learning ability of conventional architectures (MLP, CNN) trained with traditional supervised learning. However, our theoretical framework could potentially be extended to characterize the ability of these other methods to capture the symmetries of data.}\medskip

\lightbold{Beyond symmetries, the interplay between machine learning and data geometry has been studied for many types of data structure: compositional} \citep{sabour2017dynamic,schott2022visual,liang2024how,wiedemer2024compositional,lippl2024doescompositionalstructureyield,yang2025swingbydynamicsconceptlearning}\lightbold{, hierarchical} \citep{saxe2019mathematical, mel2021theory}\lightbold{, data lying on a manifold} \citep{goldt2020modeling,gerace2022probing} \lightbold{or on separate manifolds} \citep{chung2018classification,cohen2020separability,sorscher2022neural}. \emph{Our work focuses on datasets presenting symmetries (i.e., group-action induced structures), an important type of structure found in the physical world and present in many datasets.} \medskip

\lightbold{Two main kernel theories of deep learning exist, which identify deep networks with kernel machines} \citep{Neal1996, jacotNIPS2018}.\footnote{We refer the reader unfamiliar with neural kernel theories to the excellent lecture notes of Adityanarayanan Radhakrishnan available online on this topic (\url{https://aditradha.com/lecture-notes/}). We also provide an intuitive explanation of neural kernel theories in Appendix \ref{app:intro_ntk}.} Both theories apply in the limit where the networks have infinite width (infinite number of neurons within a layer or infinite number of channels in convolutional networks), and both identify training a network with doing simple Gaussian Process regression over a fixed kernel determined by the network architecture.
The Neural Network Gaussian Process (NNGP) theory \citep{Neal1996} characterizes the distribution of functions produced by random draws of the network parameters (its weights and biases), assuming a distribution from which the parameters are drawn (typically i.i.d. Normal). Conditioned on a training set, the NNGP characterizes the distribution of solutions (i.e., functions passing through the training input-label pairs) under the assumed distribution of parameters, and can be used to make predictions on a testing set.
The Neural Tangent Kernel theory (NTK) \citep{jacotNIPS2018} describes the distribution of solutions obtained by training all layers with gradient descent, from random initial conditions. It is allegedly the most realistic theory (although see \cite{avidan2025connectingntknngpunified} for a more nuanced account) and the one we use in this work. Nonetheless, we repeated our analyses using the NNGP kernel instead of the NTK and it did not change our conclusions (not shown). Neural kernel theories assume networks to have infinite width, which is a big approximation in practice. For example, \cite{fortNEURIPS2020} show that the kernel of finite-width networks changes rapidly during the first few epochs of training to a more favourable kernel for the task, breaking the frozen kernel assumption which is only provable in the infinite-width limit. \emph{However, in all our experiments, we find that kernel theories adequately capture the generalization behavior of finite-width, normally trained networks on datasets presenting symmetries.} 
\medskip

\lightbold{Finally, in an insightful line of work,} \cite{bordelonPMLR2020a,canatar2021spectral} \lightbold{relate generalization properties of infinite-width deep networks to spectral properties of their kernel}. They do not, however, explore the specific interplay between spectral kernel properties and data symmetries. \emph{In our work, we find that the symmetric nature of the dataset considerably simplifies the spectral description of neural kernels, allowing better interpretability of the factors on which generalization depends.} \\

We next propose a simple empirical study to illustrate the problem of symmetry learning and highlight the failure of a range of common architectures on this problem.

\section{Illustration of a Symmetry Learning Problem}\label{sec:illustration}

\begin{figure}[!b]
  \centering
  \includegraphics[width=1\columnwidth]{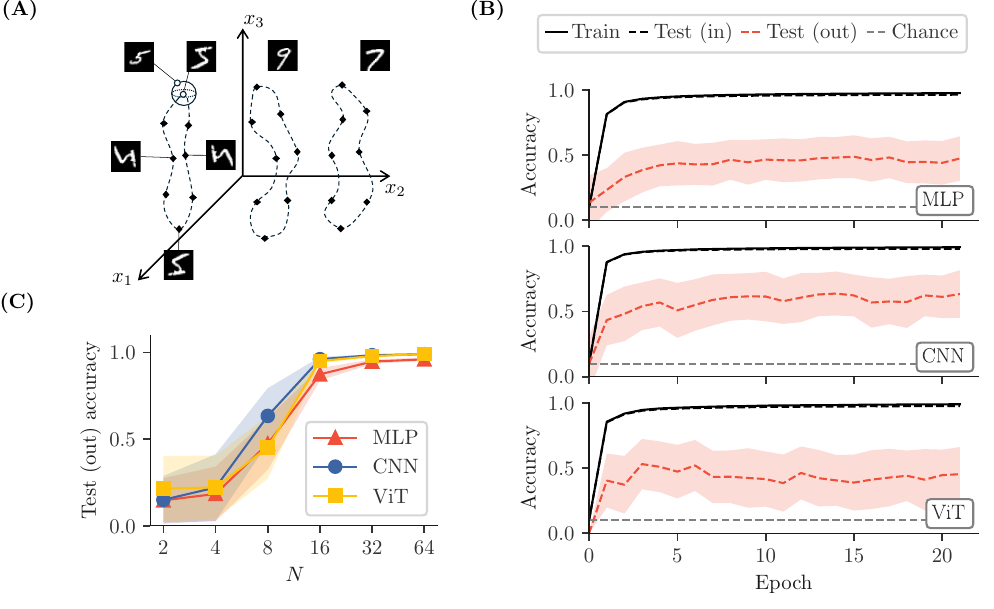}
  \caption{\textbf{Common deep network architectures fail to extrapolate symmetries from a partially observed version of rotated-MNIST.} \textbf{A}: A conceptual sketch of the learning task. For samples in the leave-out class (digit ``5'' in this example), the upright pose is not included in the training set (empty circles). For samples in the remaining classes, all poses are included in the training set (full circles). A model that can generalize rotations should classify the missing upright samples correctly. \textbf{B}: The accuracies of the three models tested respectively on the training set, \emph{in-test} set (normal i.i.d. test set) and \emph{out-test} set (i.e., left-out pose of the leave-out class). Chance levels (10\%) are also reported as a baseline. Error shades represent 95\% confidence intervals computed over all tested leave-out classes (n = 10). \textbf{C}: Out-test accuracy for the three models as a function of the number of angles in the rotation orbit. Error shades as in \emph{B}.} 
  \label{fig:symm_learning}
\end{figure}
To illustrate the problem of symmetry learning, we train a range of networks with different architectures (MLP, ConvNet, ViT-S \citep{dosovitskiy2020image}) on partially observed versions of rotated-MNIST (details of the architectures in App.~\ref{app:arch}).

We start by constructing the rotated-MNIST dataset: for all samples in MNIST, we generate $N$ rotated samples, where the rotation angles are multiples of $2\pi/N$.
We choose a digit class, which we refer to as the \emph{leave-out class} (Fig.~\ref{fig:symm_learning}A).
We define the ``training'' and the ``in-test'' splits as a 90/10 random split of the whole rotated MNIST dataset, \emph{except} the upright samples of the leave-out class.
These samples, instead, constitute what we call the ``out-test'' split.\footnote{Note that predicting the labels of the out-test split is an out-of-distribution (OOD) generalization task. In this work, we study whether networks are able to perform this OOD task in virtue of the symmetric nature of the dataset.}

We report the accuracies of the three models in the case of 8 rotation angles in Fig. \ref{fig:symm_learning}B.
The in-test accuracy closely tracks the training set accuracy; however, the out-test accuracy is considerably lower for all the considered models. Training for very long times ($>$1000 epochs) did not improve out-test accuracy (no grokking happens). 

We find that increasing the number of angles leads to better out-test performance across models (Fig. \ref{fig:symm_learning}C).
Successful generalization is achieved when the number of angles sampled is sufficiently large. \\

In the following, we propose a theory to understand the generalization abilities of deep networks on problems involving learning from partially observed symmetries. Fully developed, our theory gives a precise account of the generalization behavior of conventionally trained networks on rotated-MNIST (Section \ref{sec:multi} and Fig. \ref{fig:multiseedclass}).

\section{A Neural Kernel Theory of Symmetry Learning}\label{sec:theory}

Essentially, the argument of this paper relies on the two following observations: (1) in the Neural Tangent Kernel (NTK) limit, training deep networks is analog to performing kernel regression; (2) kernel regression has a greatly simplified expression on datasets presenting symmetries, allowing a simple geometric interpretation of the factors leading to generalization (or lack thereof). We provide in Appendix \ref{app:intro_ntk} a brief and intuitive description of NTK. We present below the main theoretical result of the paper, which consists in the simplified expression for kernel regression on a symmetric dataset.

\subsection{General Definitions}
We start by recalling basic definitions relating to kernel methods.
\begin{definition}[Kernel Function]
A kernel function \( k: \mathcal{X}\times\mathcal{X}\to \R \) takes two inputs \( x, x' \in \mathcal{X} \) and returns a scalar value representing their similarity. It can be defined as the inner product between the images of the inputs under a certain feature map \( \varphi: \mathcal{X}\to\mathcal{H} \):
\[
k(x, x') = \langle \varphi(x), \varphi(x') \rangle
\]
Where:
\begin{itemize}
    \item \( \varphi(x) \) is a mapping from the original input space \(\mathcal{X}\) to a higher-dimensional feature space \(\mathcal{H}\) (possibly infinite-dimensional).
    \item \( \langle \cdot, \cdot \rangle \) represents the inner product in \(\mathcal{H}\).
\end{itemize}
Knowing the kernel function allows one to skip explicit computation of the mapping \(\varphi(x)\) and directly obtain the inner product.

 \end{definition}

 \begin{definition}[Gram Matrix]
 Given a dataset \( \{ x_1, x_2, \dots, x_N \} \) with \( N \) data points, and a kernel function \( k(x, x') \), the \emph{Gram matrix} \( K \) is defined as the matrix of pairwise kernel evaluations between all pairs of data points:
\[
K_{ij} = k(x_i, x_j)
\]
This matrix captures the pairwise relationships between all data points in the feature space induced by the kernel.
\end{definition}
\begin{definition}[Circulant Gram Matrix]
A Gram matrix \( K \) is said to have a \textit{circulant structure} if each row of the matrix is a cyclic shift of the previous row. Specifically, the matrix \( K \) is circulant if its entries satisfy the following condition:
\[
K_{ij} = c_{(i-j) \mod N}
\]
where \( N \) is the number of data points (or the size of the matrix), and \( c_{k} \) represents the entries of a vector \( \mathbf{c} = \{ c_0, c_1, \dots, c_{N-1} \} \), which is periodic with period \( N \). This means the matrix \( K \) can be written in the form:
\[
K = \begin{bmatrix}
c_0 & c_1 & c_2 & \dots & c_{N-1} \\
c_{N-1} & c_0 & c_1 & \dots & c_{N-2} \\
c_{N-2} & c_{N-1} & c_0 & \dots & c_{N-3} \\
\vdots & \vdots & \vdots & \ddots & \vdots \\
c_1 & c_2 & c_3 & \dots & c_0
\end{bmatrix}
\]
In other words, the matrix entries are determined by a single vector \( \mathbf{c} \), and each row is a shifted version of this vector, where the shifts are cyclic (i.e., they wrap around the matrix). 
\end{definition}
\subsection{Spectral Error}
When the Gram matrix is circulant over a dataset made of two interleaved classes, we can derive a simplified formula for kernel regression.
\begin{proposition}\label{prop:spectral_error}
Consider a kernel function and its Gram matrix on a dataset. The dataset has an even number of points $2N$, where the data points are from two classes ordered in an interleaved manner with labels \( +1 \) and \( -1 \). Assume that the Gram matrix is circulant on this ordered dataset. Recalling that a circulant matrix is diagonalized by the \emph{Discrete Fourier Transform} (DFT), kernel regression on a missing point from this dataset will result in the following error, referred to as the \emph{spectral error}:
\begin{align}
    \varepsilon_s = \frac{\lambda_N^{-1}}{\langle \lambda^{-1} \rangle}, \label{eq:spectral_error}
\end{align}
where \( \lambda_N \) denotes the eigenvalue of the largest frequency of the Gram matrix and \( \langle \lambda^{-1} \rangle \) denotes the average of the inverse eigenvalues over all frequencies (from $-N$ to $N$). 
\end{proposition}
The result above, proven in App.~\ref{app:spectral_error}, will be leveraged throughout this study to gain insights into the geometrical quantities that matter for generalization of deep networks on symmetric datasets. 
Essentially, we will show that---on symmetric datasets---various kernels (RBF, MLP, CNN) lead to circulant (or approximately circulant) Gram matrices, allowing to use the spectral error formula (Eq.~\ref{eq:spectral_error}) to interpret the geometric factors on which generalization depends. Note that the spectral error can also be extended to multiple missing points (see App.~\ref{app:spectral_error_multiple_points}) but it leads to a less interpretable formula, so we do not consider this case. \\ \\
The spectral error formula can also be generalized to \emph{any finite group, including non-abelian groups}, see Theorem \ref{thm:nonabelian} in App.~\ref{app:nonabelian}. The resulting formula is necessarily more complex, requiring notions of representation theory and non-commutative harmonic analysis (generalized Fourier transform). For the sake of simplicity, we limit our treatment to the cyclic group in what follows.

\subsection{A Simple Case Study: Gaussian Kernel Regression on a Circular Dataset}\label{sec:gaussian_toy}

We first study the behavior of the Gaussian (i.e., RBF) kernel on a simple symmetric dataset in $\R^3$ (Fig. \ref{fig:spectrum_sketch}). We will see later that this case study captures all the relevant phenomenology for understanding the behavior of deep networks on datasets presenting symmetries.

\paragraph{Circular dataset in $\R^3$.} Consider the following dataset of $2N$ points (see Fig. \ref{fig:spectrum_sketch}A for an illustration):
\begin{align*}
    \mathcal{D} \equiv \{ (g^i. x^A, +1)\}_{i=0}^{N-1} \cup \{ (g^i.x^B, -1)\}_{i=0}^{N-1} \subset \R^n\times \R, \end{align*}
where $g$ is a representation of the generator of the cyclic group of order $N$, acting on $\R^3$ by usual matrix multiplication. This dataset is composed of two orbits of the same group, obtained for two different \emph{seed} samples, $x^A$ and $x^B$, which are labeled +1 and -1 respectively.
We \emph{order} the dataset so that points labeled +1 are interleaved with those labeled -1, while preserving individual orbit ordering.\footnote{While this ordering does not affect the results of kernel regression on a missing point (kernel methods are order-independent), it will allow us to diagonalize the Gram matrix via the Fourier transform, instead of a permutation of it, making the spectral formula directly interpretable in terms of the eigenvalues associated to each Fourier component.}
For notational convenience, we denote $x_i^{\bullet} = g^i. x^{\bullet}$. The ordered dataset can be written:
\begin{align*}
    \mathcal{D}_o & \equiv \{ (g^0.x^A, +1), (g^0.x^B, -1), (g^1.x^A, +1), \cdots (g^{N-1}.x^B, -1)\}\\
    & \equiv \{ (x_0^A, +1), (x_0^B, -1), (x_1^A, +1), \cdots (x_{N-1}^B, -1)\}.
\end{align*}

We consider the following $\R^3$ representation of the generator $g$:
\begin{align*}
    g \equiv \begin{bmatrix}
        1 & 0 & 0\\
        0 & \cos\theta & -\sin\theta\\
        0 & \sin\theta & \cos\theta
\end{bmatrix},
\end{align*}
where we define $\theta=2\pi/N$.
We consider the seed points
\begin{align}
    x^A \equiv \frac{\Delta}{2} \begin{bmatrix}
            1 \\
            0 \\
            0 
    \end{bmatrix} + \begin{bmatrix}
            0 \\
            1 \\
            0 
    \end{bmatrix}, \quad
    x^B \equiv -\frac{\Delta}{2} \begin{bmatrix}
            1 \\
            0 \\
            0 
    \end{bmatrix} + \begin{bmatrix}
            0 \\
            \cos(\theta/2) \\
            \sin(\theta/2)
    \end{bmatrix},\label{eq:low_d_orbits}
\end{align}
over which we make $g$ act by usual matrix multiplication.
Successive applications of $g$ onto the seed points generate the dataset. Note that the second seed point $x^B$ is chosen such that the two orbits are \emph{geometrically interleaved} (as depicted in fig \ref{fig:spectrum_sketch}A), which means that points of orbit $A$ are equidistant from their two nearest neighbors in orbit $B$ and vice-versa. This geometric interleaving is a necessary condition for the Gaussian kernel Gram matrix to be circulant (see below).

\begin{figure}[t]
    \centering
    \includegraphics[width=.95\textwidth]{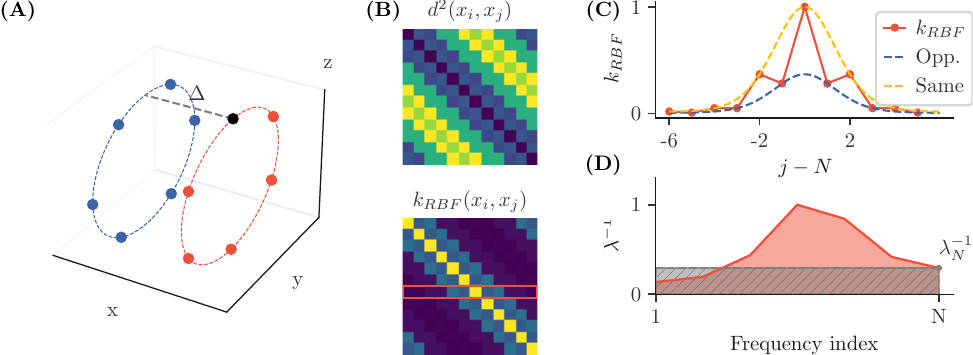}
    \caption{\textbf{Case study: Analysis of the generalization behavior of the Gaussian kernel on a circular-symmetric dataset.}
    \textbf{A}: A circular dataset is made of two sets of interleaved points in $\R^3$ belonging to two classes (denoted in red and blue). One point (in black) is left out during training, and we regress on it. 
    \textbf{B}: The pairwise distance matrix between the points is circulant (above), leading to a circulant kernel matrix (below) obtained by applying the Gaussian (RBF) kernel function to it elementwise.
    We select the $N$-th row from it, highlighted by the red rectangle; due to the circularity of the matrix, this comes without loss of information.
    \textbf{C}: A plot of the selected row of the kernel matrix. The values of the kernel alternate between two limiting curves, representing respectively the ``same label" and the ``opposite label'' kernel values.
    This alternation, due to the separation between classes, is akin to the presence of a high frequency component in the kernel function itself, were it computed over the classes perfectly interleaved ($\Delta$ = 0).
    \textbf{D}: The prediction error of a Gaussian kernel on a missing point of a circular-symmetric dataset is a simple function of its spectrum (Eq.~\ref{eq:spectral_error}).
    The reciprocal (inverse) of the positive half of the DFT of the selected row of the kernel matrix is shown.
    The grey area corresponds to the numerator of the spectral error, while the red area corresponds to the denominator.
    The ratio of the two corresponds to the prediction error incurred by kernel regression using the chosen kernel.
    The presence of the aforementioned high frequency component in the kernel is manifested in the low value of $\lambda_N^{-1}$, leading to small error (i.e., increased class separability).
    }
    \label{fig:spectrum_sketch}
\end{figure}

\begin{figure}[t!]
   \centering
   \includegraphics[width=0.95\linewidth]{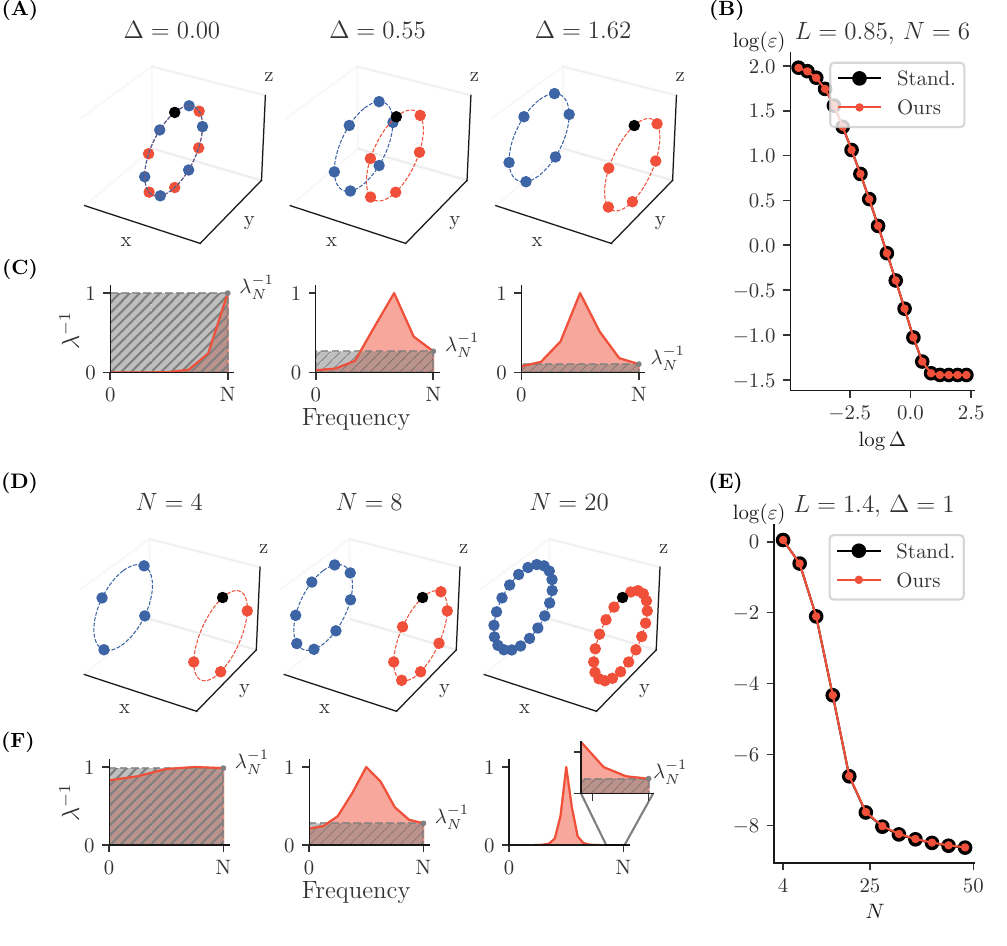}
   \caption{
   \textbf{Further geometric interpretations of the Gaussian kernel generalization behavior on a circular-symmetric dataset.}
   \textbf{A}: We progressively increase the separation $\Delta$ between two circular classes of points (respectively in blue and red). The leave-out point used for testing is in black.
    \textbf{B}: The prediction error (in black) of the RBF kernel on the leave-out point decreases as a function of $\Delta$, as predicted by our spectral formula (in red). given in Eq.~\ref{eq:spectral_error}. 
\textbf{C}: Inverse spectra of the kernel matrix for different $\Delta$. The grey (respectively, red) area corresponds to the numerator (resp., denominator) of Eq.~\ref{eq:spectral_error}. As $\Delta$ increases, the ratio between grey and red areas (i.e., the spectral error) progressively decreases, as a consequence of the last frequencies getting larger.
    \textbf{D, E, F}: Same as \emph{A,B,C} where instead of increasing $\Delta$ we increase the number of points forming each class. The ratio between grey and red areas progressively decreases with number of points, as a consequence of the middle inverse frequencies progressively diverging. }
    \label{fig:case_study}
\end{figure}

\paragraph{Gaussian (RBF) kernel regression on a circular dataset.} We consider the following regression problem: remove any one point from the dataset $\mathcal{D}_o$, and estimate the value of $y$ at the missing point.
We solve this regression problem using \textit{Gaussian process (GP)} kernel regression. We select the \textit{Gaussian or Radial Basis Function (RBF)} kernel,
\begin{align*}
    k_{RBF}(x_i, x_j) = \exp(-L^2 \ltwo{x_i - x_j}^2),
\end{align*}
where $L$ denotes the kernel's length scale.
\\

\begin{lemma}\label{lem:stationary_kernel}
The Gram matrix of a stationary kernel (a kernel that only depends on the Euclidean distance between pairs of input points) over a dataset made of two geometrically interleaved cyclic orbits (as defined above) is circulant. 
\end{lemma}
\begin{proof} 
First consider that the pairwise distance between points $x_i$ and $x_j$ in $\mathcal{D}_o$ is only function of their absolute index difference $|i-j|$. 
Since a \textit{stationary kernel} only depends on the Euclidean distance between pairs of input points, it follows that a stationary kernel produces a \textit{circulant} Gram matrix over $\mathcal{D}_o$.
\end{proof}

The RBF kernel is a stationary kernel. The Gram matrix of the RBF kernel being circulant over $\mathcal{D}_o$, Proposition \ref{prop:spectral_error} applies, and prediction at the missing point is given by the spectral error (Eq. \ref{eq:spectral_error}). This affords us a geometric interpretation of the generalization behavior of the kernel in the spectral (Fourier) domain.

\paragraph{Geometric interpretation} The formula of the spectral error lends itself to a simple geometric interpretation (Fig. \ref{fig:spectrum_sketch}).
Consider the spectrum obtained by applying the DFT to any one row of the kernel matrix.\footnote{Such row is composed entirely of real entries.
As a consequence, the spectrum is symmetric; we thus only show the positive half of the spectrum in the figure depictions, ranging from 0 to $N-1$.}
We plot the inverse spectrum $\lambda^{-1}$.
In this picture, the ratio in Eq.~\ref{eq:spectral_error} corresponds to a ratio between areas, respectively those of a rectangle with height $\lambda_N^{-1}$ for the numerator, and the area under the $\lambda^{-1}$ curve for the denominator.\\

\textbf{Let us now study the impact of the distance between orbits, $\Delta$, on the prediction error (Fig. \ref{fig:case_study}A-C)}. We vary $\Delta$ while keeping the number of points in one orbit, $N$, and the kernel's length scale, $L$, fixed.
We compute the prediction error by solving kernel regression for varying $\Delta$, and compare it with our formula for the spectral error (Fig.~\ref{fig:case_study}B).
As expected by their mathematical equivalence, the two agree across the investigated range of values, and decrease as a function of $\Delta$.
This decrease aligns with the following intuition: pulling the classes apart makes prediction over the missing point by a local kernel easier.

In Fig.~\ref{fig:case_study}C, we show the geometric interpretation of the spectral error.
Increasing $\Delta$, by increasing the power of the highest frequency of the kernel matrix, skews the ratio in favor of the denominator in Eq.~\ref{eq:spectral_error} (denoted by the red area).
For $\Delta=0$, the points are interleaved in the $yz$ plane, and a kernel that would successfully generalize on the missing point would need to be dominated by its highest frequency (oscillating) component. This is not the case for the RBF kernel, whose spectrum is instead dominated by the low frequency components due to its local nature.
As $\Delta$ increases, however, the kernel's spectrum effectively gains power in the highest frequency component, thus ``aligning" its properties to the requirements imposed by the regression problem.

In summary, an increase in the distance between orbits $\Delta$ leads to an increase in the power of the highest frequency term $\lambda_N$ of the kernel matrix, which leads in turn to better kernel generalization on the missing point.\\

\textbf{We now turn to study the impact of the number of angles in an orbit $N$ on the prediction error (Fig. \ref{fig:case_study}D-F).}
We compute the prediction error by solving kernel regression for varying $N$, and compare it with our formula for the spectral error (Fig.~\ref{fig:case_study}E).
As expected by their mathematical equivalence, the two agree across the investigated range of values, and decrease as a function of $N$.
This decrease aligns with the following intuition: increasing the point density of an orbit makes prediction over a missing point easier.

In Fig.~\ref{fig:case_study}F, we show the geometric interpretation of the spectral error.
Increasing $N$ skews the ratio in favor of the denominator of Eq.~\ref{eq:spectral_error} (denoted by the red area).
The DFT projects an $n$ dimensional signal over $n$ discrete frequencies.
For a local kernel, such as the RBF, which has most of its power slotted onto its first few low frequencies, sampling more points and then taking the DFT effectively means expanding the number of frequencies with low power, leading to the corresponding spectrum elements becoming vanishingly small.
This in turn causes the average of the inverse spectrum to diverge, making the denominator of Eq. \ref{eq:spectral_error} grow.\footnote{We note, however, that the presence of a nonzero $\Delta$ implies that the \textit{highest} frequency terms in general, and $\lambda_N$ in particular, remain large, leading to the ``explosion in the middle" of the inverse spectrum, which keeps the numerator in Eq.~\ref{eq:spectral_error} small.}

In summary, we see that increasing the number of angles along an orbit $N$ leads to an increase in the average inverse spectrum $\langle \lambda^{-1} \rangle$ (denominator of Eq. \ref{eq:spectral_error}), and thus to a decrease in generalization error on the missing point. 

These geometric insights will prove useful in interpreting the generalization capabilities of deep networks on high dimensional datasets possessing symmetries.

\subsection{Multi-layer Perceptrons on Rotated-MNIST}\label{sec:mlps_rotmnist}

Building on the insights from the previous section, we now turn to a more realistic setting. We study the predictions of Multi-layer Perceptrons (MLP) in the Neural Tangent Kernel limit (NTK), on partial views of rotated-MNIST, a high-dimensional dataset containing a rotational symmetry.
\paragraph{Datasets and architectures} We consider the MNIST dataset, and we augment it by means of (discrete) rotations.
More precisely, for each MNIST digit, $N$ images are generated by rotation in increments of $2\pi/N$, capturing the full rotational range.
This way, we have access to a complete \textit{rotational orbit} for each sample.
This rotation transformation corresponds, up to pixel discretization effects, to the action of a particular representation $g$ of the cyclic group of order $N$ (specifically, the regular representation), acting on the vector space $\R^{28\times 28}$ in which the greyscale MNIST images live.
Additionally, we normalize the digits to be on the sphere, making their orbits comparable with each other.

We then create datasets by forming \textit{pairs} of these orbits, each pair consisting of digits from different classes.
Given these datasets, we define a regression problem by leaving out one of the points in the two orbits, and asking to find a predictor for this missing sample. Note that unlike in the more realistic scenario presented in Fig.~\ref{fig:symm_learning}, our classes here are constructed from a single seed sample and their orbit, and we only consider two classes. However, we will show in Section \ref{sec:multi} that our conclusions on this simplified setup also apply to the multi-seed-per-class and multi-class problems, under further approximations.

We now focus on neural networks as our predictors.
Specifically, we consider respectively a 1-hidden-layer and a 5-hidden-layer MLPs with ReLU activations (see App.~\ref{app:arch} for more details on the architectures) in the NTK limit. In this limit, the predictions of the networks can be written as the result of kernel regression, where the kernel function is determined by the architecture of the network. Therefore, we can replicate the analysis performed in the previous section, where we replace the simple Gaussian kernel with a neural kernel.
In the following, we refer to our kernel function as $k_{NTK}$, and the respective kernel matrix as $K^{NTK}$.
In practice, we use the \href{https://github.com/google/neural-tangents}{\texttt{neural-tangents}} library \citep{Novak2020Neural} to compute the neural kernels.\\

Importantly, while this setup is partly analogous to the case study of the previous section (RBF kernel over a circular dataset), there are some key assumptions and approximations that need to be justified before we can apply the theory to this case.\\

Firstly, unlike the RBF kernel, the MLP kernel is not stationary, so Lemma \ref{lem:stationary_kernel} does not apply to the MLP kernel. However, we can show that the MLP kernel, in virtue of being a \emph{dot product} kernel \citep{Neal1996}, also produces a circulant Gram matrix over an orbit of the cyclic group:

\begin{lemma}
    The Gram matrix of a dot product kernel is circulant over a single orbit of the cyclic group.
\end{lemma}
\begin{proof}
The cyclic group acts on a seed point linearly via its (matrix) representation $R$; due to the properties of the cyclic group, such representation is orthogonal.
The dot product between data points thus only depends on their index difference, as indexed by the group action:
\begin{align*}
    (x_i^A)^T x_j^A = (R^i x_0^A)^T(R^j x_0^A) = (x_0^A)^T R^{-i} R^j x_0^A = (x_0^A)^TR^{j-i}x_0^A
\end{align*}

As the kernel of a dot product only depends on the scalar product of its inputs, it thus also only depends on the index distance between the data points:
\begin{align*}
    K_{ij} = k(x_i^A, x_j^A) = k((x_i^A)^Tx_j^A) = k((x_0^A)^TR^{j-i}x_0^A),
\end{align*}
and thus the $(i,j)$ entry of the Gram matrix depends only on the difference between the two indices $(i-j)$, modulo $N$ because of the properties of the rotation matrix.
The Gram computed over a single orbit is thus circulant. 
\end{proof}

Next, we remind the reader that the datasets we consider here are composed of two orbits, obtained by applying the cyclic group action to two different MNIST digits. Unlike in the synthetic dataset of the previous section, these two orbits are not geometrically interleaved, as we have no control over the position of the seed points given by the MNIST digits. This leads us to make the following approximation and adjustment.

\paragraph{Approximation: Circularity of the kernel matrix over the dataset made of two orbits} MNIST digits from different classes are in general not \textit{geometrically} interleaved, i.e., a point in one orbit is not equidistant from the ``neighboring" points in the other orbit. This can be written as the following dot product condition:
\begin{align*}
    \textcolor{custombluelight}{x_i^A \cdot x_i^B} \neq \textcolor{customblue}{x_i^A \cdot x_{i+1}^B.}
\end{align*}
Furthermore, the angular distance between first neighbors in each orbit is not necessarily the same, i.e., 
\begin{align*}
    \textcolor{customredlight}{x_i^A \cdot x_{i+1}^A} \neq \textcolor{customred}{x_i^B \cdot x_{i+1}^B.}
\end{align*}

Indeed, even though the seed images are all normalized to be on the sphere, the angular distance between first neighbors on an orbit still depends on how close the seed image is to the stabilizer of the group action.
Take for instance a seed image with perfect central symmetry (e.g., a perfect `o'); all rotations of the seed image are the seed image themselves, and thus the angular distance is 0.
Conversely, for a ``Dirac delta"-like image (i.e., an image that is nonzero at a single point, assuming infinite pixel resolution), all rotated images are orthogonal to each other, setting their angular difference to 1.
The angular distance between first neighbors on a same orbit thus sits between 0 and 1, depending on the smoothness of the image with respect to the group transformation.

As a consequence, the following NTK matrix is in general not circulant:\footnote{We can note that the first diagonal (denoted in black) is constant in virtue of the fact that we normalize each data point to be on the sphere.}
\[
K^{\text{NTK}} = k_{\text{NTK}} \left(
\renewcommand{\arraystretch}{1.5} \begin{bmatrix}
    x_0^A \cdot x_0^A & \textcolor{custombluelight}{x_0^A \cdot x_0^B} & \textcolor{customredlight}{x_0^A \cdot x_1^A} & \textcolor{customblue}{x_0^A \cdot x_1^B} & \textcolor{customred}{x_0^A \cdot x_2^A} & \cdots \\
    \textcolor{custombluelight}{x_0^B \cdot x_0^A} & x_0^B \cdot x_0^B & \textcolor{customblue}{x_0^B \cdot x_1^A} & \textcolor{customred}{x_0^B \cdot x_1^B} & \textcolor{custombluelight}{x_0^B \cdot x_2^A} & \cdots \\
    \textcolor{customredlight}{x_1^A \cdot x_0^A} & \textcolor{customblue}{x_1^A \cdot x_0^B} & x_1^A \cdot x_1^A & \textcolor{custombluelight}{x_1^A \cdot x_1^B} & \textcolor{customredlight}{x_1^A \cdot x_2^A} & \cdots \\
    \textcolor{customblue}{x_1^B \cdot x_0^A} & \textcolor{customred}{x_1^B \cdot x_0^B} & \textcolor{custombluelight}{x_1^B \cdot x_1^A} & x_1^B \cdot x_1^B & \textcolor{customblue}{x_1^B \cdot x_2^A} & \cdots \\
    \textcolor{customred}{x_2^A \cdot x_0^A} & \textcolor{custombluelight}{x_2^A \cdot x_0^B} & \textcolor{customredlight}{x_2^A \cdot x_1^A} & \textcolor{customblue}{x_2^A \cdot x_1^B} & x_2^A \cdot x_2^A & \cdots \\
    \vdots & \vdots & \vdots & \vdots & \vdots & \ddots \\
\end{bmatrix}
\renewcommand{\arraystretch}{1} 
\right)
\]

\begin{figure}[t]
    \centering
    \includegraphics[width=.95\linewidth]{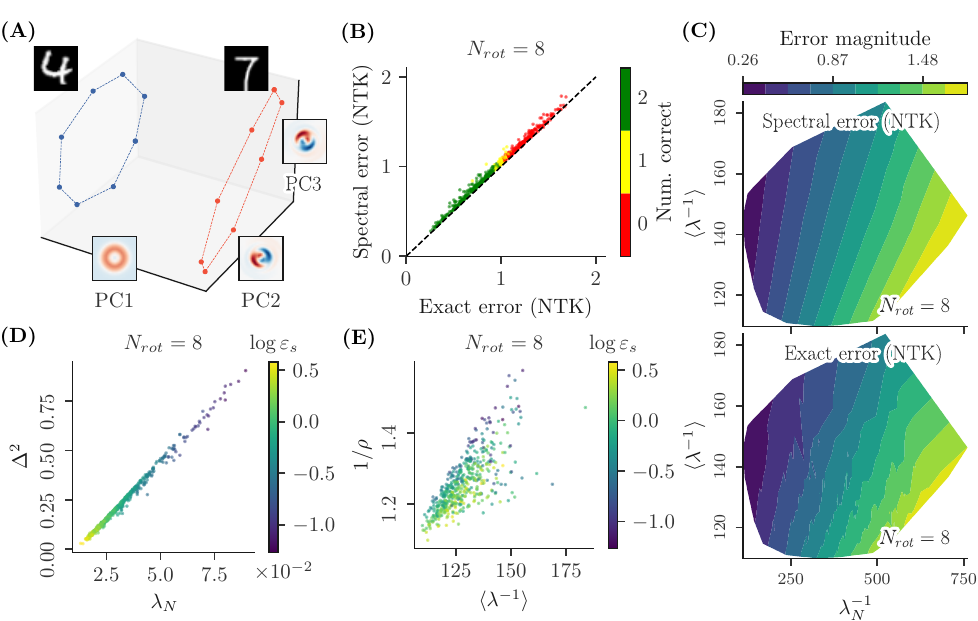}
    \caption{\textbf{Analysis of the prediction behavior of a MLP on pairs of orbits from rotated-MNIST.}
    \textbf{A}: We take samples of two different MNIST digits and generate their rotation orbits. The task is to predict the label value of a leave-out point in one of the two orbits. We use angle steps of 45 degrees, amounting to $N_{rot}=8$. We show the points in a reduced-dimensional space obtained by performing PCA on the dataset. 
    \textbf{B}: Scatter plot of the NTK error computed in the standard way (exact) against our spectral error (Eq.~\ref{eq:spectral_error}). Each dot corresponds to a different dataset, obtained by randomly picking pairs of digits of different classes. The color coding corresponds to the number of classification errors incurred by the symmetrized NTK regression, obtained by excluding a point from class A, then from class B, and counting the number of classification errors (0, 1 or 2), understood as a disagreement in sign between NTK prediction and label of the missing point.
    \textbf{C}: Comparison between the spectral and exact NTK error across different values of $\lambda^{-1}_N$ and $\langle \lambda^{-1}\rangle$. 
    \textbf{D}: Comparison between the values of $\lambda_N$ and $\Delta^2$. The former is the highest frequency component of the neural kernel matrix (inverse of the numerator in Eq.~\ref{eq:spectral_error}), while the latter is the distance in pixel space between the averages (centroids) of the two orbits.
    \textbf{E}: Comparison between the values of $\langle \lambda^{-1}\rangle$ and $1/\rho$ (see main text). }
    \label{fig:highd_study}
\end{figure}

To retrieve the setup outlined in the previous section, we resort to an \textit{ex-post} circularization procedure.
We consider the matrix $\Tilde{K}^{NTK}$ that is obtained by taking the diagonal-wise average of the kernel matrix $K^{NTK}$:
\begin{align*}
    \Tilde{K}^{NTK}_{ij} = \frac{1}{2N} \sum_{k} K^{NTK}_{(i+k)\% (2N), (j+k)\% (2N)}.
\end{align*}
Such matrix is by definition circulant.
The practical meaning of this procedure is twofold: the averaging of \textit{even} diagonals imposes that both classes' feature embeddings are ``equally spaced'' along the orbit, while the averaging of \textit{odd} diagonals (together with the intrinsic symmetry of the kernel matrix) imposes that the feature embeddings are interleaved in the sense that $\Tilde{K}^{NTK}_{i, i+1} = \Tilde{K}^{NTK}_{i, i-1}$. 
We stress that this circularization procedure is not justified a priori. However, we empirically show that it is a realistic approximation for analyzing the interplay between symmetric datasets and deep networks.

\paragraph{Adjustment: Symmetrization of the NTK error w.r.t. to class} Because of the asymmetry between orbits, removal of a point from one orbit is not in general equivalent to removal of a point from the other orbit. Yet, our formula for the spectral error, operating on the circularized kernel matrix, is intrinsically symmetric.
To restore this interchangeability in the standard NTK error, we consider in the following a \textit{symmetrized NTK error}, which we obtain by averaging the NTK prediction errors that arise by removing a point from either one orbit or the other. 
It is this symmetrized NTK error that we compare with our spectral error.

\paragraph{Results} Equipped now with a circulant kernel matrix, we can once again refer to the spectral error formula of Eq.~\ref{eq:spectral_error} for the estimation of the prediction error associated with performing kernel regression under the kernel matrix $\Tilde{K}^{NTK}$.

We find that the spectral error tracks the symmetrized NTK error across all sampled pairs (Fig.~\ref{fig:highd_study}A-B).
This agreement holds across the whole range of explored values for both $\lambda_N$ and $\langle \lambda^{-1}\rangle$ (Fig.~\ref{fig:highd_study}C), and regardless of the number of rotation angles (App.~\ref{app:many_angles}).
Moreover, this agreement qualitatively holds for finite-width, trained networks (Fig. \ref{fig:all_ntks}B and App.~\ref{app:more_mlps}), as well as for deeper (5-layer) MLPs (Fig. \ref{fig:all_ntks}C and App.~\ref{app:more_mlps}).
This suggests that the circularization procedure preserves the overall structure of the problem, and allows us to study generalization through a geometric lens as in the previous section.

\begin{figure}
    \centering
    \includegraphics[width=1\linewidth]{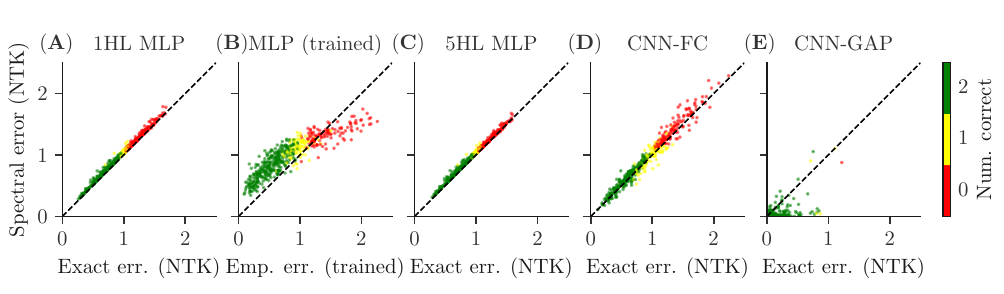}
    \caption{\textbf{Spectral error matches exact NTK error across various architectures and for finite-width networks, on rotated-MNIST orbit pairs}: \textbf{(A)} a MLP with 1 hidden layer, \textbf{(B)} a finite-width 1-hidden layer MLP trained with Adam, \textbf{(C)} a MLP with 5 hidden layers,  \textbf{(D)} a ConvNet with a fully-connected last layer, \textbf{(E)} a ConvNet with global average pooling at the last layer. 
    In this case, the assumptions of the theory are too crude to capture empirical phenomenology (see text).
    Color coding as in Fig.~\ref{fig:highd_study}B.
    }
    \label{fig:all_ntks}
\end{figure}

\paragraph{Geometric interpretation} The quantities involved in the spectral error (Eq.~\ref{eq:spectral_error}) are computed in kernel space, not in input (i.e., dataset) space, which makes their interpretation difficult. However, for simple kernels such as the MLP kernel, we will now show empirically that these quantities correlate with geometric quantities in input space, allowing us to understand the geometric factors underlying generalization.

Firstly, the highest frequency eigenvalue $\lambda_N$ correlates with the orbit separation in input space $\Delta$ (Eq.~\ref{eq:low_d_orbits}).
Indeed, the values of $\lambda_N$ and the euclidean distance between orbit averages in input space are, for the explored ranges, in an approximate linear relationship (Fig.~\ref{fig:highd_study}D).
\emph{As a consequence, one can see how the separation between the orbits is a key indicator of prediction error for the MLP kernel.}

Secondly, we can map the average reciprocal spectrum $\langle \lambda^{-1}\rangle$ to a measure of orbit density in input space, computed as the \emph{inverse of the orbit radius}. The orbit radius is computed as the distance of any orbit sample to the orbit centroid, in the high-dimensional dataset space.\footnote{The orbit radius is a measure of how close the seed image is from the stabilizer of the group action. For example, a seed image that is rotation invariant, such as a perfect `o', is situated on the stabilizer of the rotation group action, leading to a null radius.}
We see that these two quantities are in an approximately linear relation (Fig.~\ref{fig:highd_study}E).

These relations enable a geometric interpretation of the quantities that determine generalization success: generalization success depends on the distance between classes $\Delta$ and the density of orbits. Note, however, that the observed mapping between spectral quantities of the kernel and geometric quantities in input space is empirical and not guaranteed to hold for all architectures. In particular, we show later that this equivalence breaks for a convolutional neural network with global average pooling at the last layer, making the interpretation of the spectral formula terms less direct.

\subsection{Extension to Multiple Seeds per Class and Multiple Classes}\label{sec:multi}

Until now, we have only considered a simplified scenario in which the datasets are composed of a single orbit per class, and where there are only two classes.
We now ask whether the theory can also be adapted to more a realistic scenario, where datasets are composed of multiple seed images per class, and where there are multiple classes, only one of which has a missing angle during training.
This setting is the one we presented at the start of this study (Fig~\ref{fig:symm_learning}).
Below, we show that our framework can be extended to this scenario.

\paragraph{Extension to multiple seed points per class}\label{sec:multiseed}
First, we describe how our framework can be extended to multiple seeds per class and two classes.
Since our spectral formula only applies to datasets made of two orbits (one for each class), we make the simplifying assumption that pairs of orbits interact \emph{linearly} to predict the average error on a missing point. In other words, we average the spectral error computed over each possible pairs of orbits taken across classes.

Specifically, for all pairs of seeds that can be formed across the two classes, we compute a spectral error according to Eq.~\ref{eq:spectral_error}.
By doing so, we are considering the case of a missing point in either of these two orbits.
We then average the error over all the pairs, obtaining a single, average spectral error.

We compare this average spectral error against the exact NTK error, computed over the entire dataset.
We obtain the exact error by running NTK regression on the dataset that contains all orbits of both classes, excluding one angle from all orbits of one class.
On said angles, we compute the NTK regression error, and average it over missing points from all orbits of that class.
We then compute its symmetrized version by swapping the role of the two classes, and averaging the two error values.

We test the agreement between the pair-averaged spectral error and symmetrized NTK error across many trials, where each trial corresponds to a different dataset, obtained by taking a number of seeds ($N_{\mathrm{seed}}=13$) from each of two different, randomly picked MNIST digit classes. We limit the datasets to contain 13 seeds per class for computational reasons.
We find empirically that the average spectral error correctly predicts exact NTK error across a large number of such trials (Fig.~\ref{fig:multiseedclass}A, extended version in App.~\ref{app:multi_plots}). We note that the spectral error does not capture well the magnitude of the exact error anymore (in reason of the additional assumptions needed in the multi-seed case). However, the spectral error still correlates well with the exact error across datasets.

In conclusion, our spectral theory, which simply considers the average pairwise interaction between orbits of different classes, predicts well the generalization behavior of infinite-width networks on a dataset comprising multiple seeds per class.

\begin{figure}
    \centering
    \includegraphics[width=\linewidth]{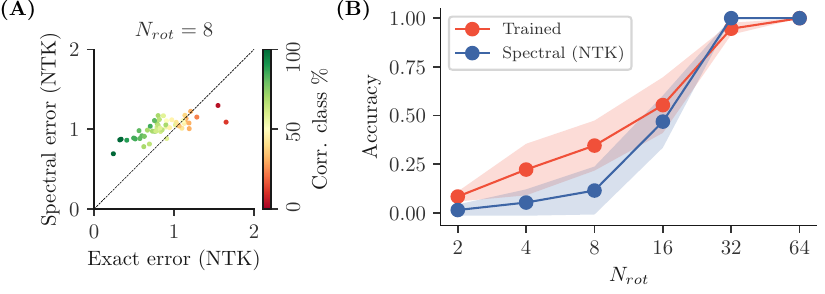}
    \caption{
    \textbf{Application of the spectral theory to a MLP trained on a subsection of rotated-MNIST comprising multiple seeds per class and multiple classes.}
    \textbf{A}:
    For datasets comprised of two classes of rotated-MNIST, comprising multiple seeds each ($N_{\mathrm{seed}}=13$), we compare the average spectral error, obtained by averaging over all pairings of orbits in the dataset, with the symmetrized NTK prediction error, computed over all possible missing points. Each dot in the scatter plot represents a different dataset, drawn by randomly selecting two classes from MNIST, and randomly selecting 13 seed images per class. 
    The color coding reflects the percentage of seeds (of both classes) for which the NTK regression gives a correct prediction, understood as agreeing in its sign with the label of the missing points.
    \textbf{B}:
    On a multi-class (10 classes of MNIST), multi-seed-per-class ($N_{\mathrm{seed}}=13$) version of rotated-MNIST, we compare the generalization accuracy predicted by our multi-class adapted spectral error, with the one of a normally trained MLP (2 hidden layers, trained with a cross-entropy loss). 
    As the number of points in the orbits increases, both trained and spectral accuracies increase on the classification task, progressively and similarly, suggesting that no mechanism for symmetry learning is present for finite-width trained networks that would be unaccounted for by the spectral theory.
    }
    \label{fig:multiseedclass}
\end{figure}

\paragraph{Extension to multiple classes}\label{sec:multiclass}
We now describe how our framework can be extended to the 10 classes of MNIST.
We stress that, in principle, NTK regression is not suited to predict the results of a network that is trained with a cross-entropy loss on multiple classes.
However, we show that we can adapt our spectral theory to this scenario, and qualitatively model the results of such training.
We do so by employing a \textit{one-versus-many} strategy.

Consider an orbit from class A.
Form all pairs with orbits of another class B (there are $N_{\mathrm{seed}}=13$ such pairs).
Average the spectral prediction obtained from all these pairwise comparisons.
Repeat this comparison with all other classes C, D, E, etc.
We thus obtain a prediction for a missing point in the orbit of class A against every other \emph{class}.
If \emph{all} of these class-wise predictions are correct, we consider the network prediction on this orbit of class A to be correct.
We extend this procedure to all orbits of class A, and count the percentage of correctly classified orbits for that class.
We repeat this procedure for every possible leave-out class, and report the average accuracy of our classifier over all leave-out classes. 

We compare the predicted accuracy resulting from this procedure, to the empirical accuracy of a finite-width 2-hidden-layer MLP trained with Adam and a cross-entropy loss (see details of architecture in App.~\ref{app:multiclass_arch}), on a version of rotated-MNIST comprising $N_{\mathrm{seed}}=13$ seeds per class (we limit the number of seeds per class to 13 because of the prohibitive computational cost of computing NTK regression on larger datasets). 

The spectral accuracy qualitatively matches the empirical accuracy of the trained network as we vary the number of points composing the orbits (Fig.~\ref{fig:multiseedclass}B, extended version in App. Fig.~\ref{fig:multiclass}). The spectral accuracy curve also mirrors the empirical curves obtained by training various architectures (MLP, CNN, ViT) on a full version of rotated-MNIST at the beginning of this study (Fig. \ref{fig:symm_learning}).
We note that classification accuracy increases with the number of sampled angles until it saturates to perfect classification accuracy. The saturation effect comes from the fact that we evaluate the network on a classification task. When the elements of the orbits are dense enough in kernel space, the error in prediction becomes small enough that the test point is systematically correctly classified.

\emph{In conclusion, our spectral theory, which only considers pairwise interactions between seeds across classes, and lends itself to a simple geometric interpretation of the factors leading to correct generalization (namely class distance and orbit density), qualitatively recapitulates the lack of generalization of conventionally trained finite-width networks on a realistic dataset, rotated-MNIST, comprising multiple classes and multiple seeds per class. Moreover, both theory and experiments show that the error is a progressive function of these two quantities, and that there isn't a phase transition where the network would suddenly learn a function that captures the symmetry of the problem. Taken together, these results provide strong evidence that the ability of conventional deep networks to generalize on symmetric datasets is essentially dictated by local geometrical properties of the datasets, and that no specific mechanism exists that would allow deep networks to learn from examples the non-local, symmetric structure of a prediction problem.
}

\subsection{Extension to Equivariant Architectures}

We here study how the interplay between dataset symmetries and equivariant architectures affects generalization. 

For simplicity and concreteness, we focus our study on spatially convolutional neural networks. We distinguish two types of convolutional architectures: (1) convolutional architectures where the last layer is fully connected: these architectures do not ensure full invariance to translation; (2) convolutional architectures where the last layer performs global average pooling, ensuring invariance to translation.

We also distinguish two cases of data symmetries: (1) when the symmetry in the data corresponds to the one encoded in the equivariant architecture; (2) when it does not, and the architecture is equivariant to another, different symmetry.
We consider a dataset with translational symmetry to illustrate the first case (a convolutional neural network is equivariant to translations) and a dataset with rotational symmetry to illustrate the second case (a convolutional neural network is not equivariant to image rotations). 

Through these examples, we build a framework that should be easily extensible to characterize how equivariant architectures and dataset symmetries interact in general.

Below, we provide formal definitions for the aforementioned datasets and network architectures, and then proceed to state the results.

\subsubsection{Definitions}

\begin{definition}[Dataset with translational symmetry]
A dataset with translational symmetry is composed of seed images and all their translations.
For the purpose of the proofs to follow, we will focus on a single orbit of this dataset, which consists of a single seed point $x_s \in \R^{n \times n} $ and all of its translations:
\begin{align*}
    \mathcal{O}_T &= \{ g_T^0.x_s, g_T^1.x_s, \cdots g_T^{n-1}.x_s\},
\end{align*}
where the translation operator $g_T$ acts on images by circularly shifting them along one of the dimensions.
For pixel coordinates $(i_x,i_y)$, and the corresponding value of the pixel $x(i_x, i_y)$, we write:
\begin{align*}
    g_T.x\left(\begin{bmatrix}
            i_x \\
            i_y
    \end{bmatrix}\right) = x\left(\begin{bmatrix}
            (i_x + 1) \bmod n\\
            i_y
    \end{bmatrix}\right)
\end{align*}

\end{definition}

\begin{definition}[Dataset with rotational symmetry]
A dataset with rotational symmetry is composed of seed images and all their rotations in $C_4$ (we limit ourselves to 4 cardinal rotations to avoid definitional problems of image rotation on discrete pixel grids).
We will focus on a single orbit of this dataset, which consists in a single seed point $x_s \in \R^{n \times n} $ and all its rotations: 
\begin{align*}
    \mathcal{O}_R &= \{ g_R^0.x_s, g_R^1.x_s,  g_R^{2}.x_s, g_R^{3}.x_s\}
\end{align*}
where the rotation operator $g_R$ permutes pixel coordinates $(i_x,i_y)$ as follows:

\begin{align*}
    g_R. x\left(\begin{bmatrix}
            i_x \\
            i_y
    \end{bmatrix}\right) = x\left(\begin{bmatrix}
            i_y\\
            n-i_x
    \end{bmatrix}\right)
\end{align*}
\end{definition}

\begin{definition}[Fully connected convolutional network (FC)] A fully connected convolutional network is a network \( f : \R^{n \times n} \to \R \) parameterized by:
\begin{align*}
f_\text{FC}(x) = A \frac{1}{\sqrt{k}} \phi(B \circledast x)_v,
\end{align*}
where \( A \in \R^{1 \times n^2k} \), \( B \in \R^{k \times 1 \times 3 \times 3} \), \(\circledast\) denotes the spatial convolution operation, and for any matrix \( u \in \R^{n \times n} \), \( u_v \in \R^{n^2} \) denotes the vectorization (i.e., flattening) of \( u \).
This network first applies a convolutional layer to the data, then flattens the resulting representation into a vector, and passes it through a fully connected layer. We assume circular padding and stride of 1. In experiments we use a filter of size $3\times3$. \end{definition}

\begin{definition}[Global Average Pooling convolutional network (GAP)] A GAP network is a convolutional network with global average pooling at the last layer.
This network is invariant to discrete translations.
The network \( f_\text{GAP} : \R^{n \times n} \to \R \) is parameterized by:
\begin{align*}
f_\text{GAP}(x) = \frac{1}{\sqrt{k}n^2} \sum_{k} \sum_{i_x} \sum_{i_y} A_{1 k} \phi(B_{k,i_x,i_y} \cdot x)
\end{align*}
where \( A \in \R^{1 \times k} \), and \( B \in \R^{k \times 1 \times 3 \times 3} \) with \( B_k \in \R^{1 \times 1 \times 3 \times 3} \) indexing filter \( k \) of \( B \).
$B_{k, i_x, i_y} \in \R^{1\times 1 \times n\times n}$ is obtained by centering $B_k$ at coordinates $(i_x, i_y)$ of an $n\times n$ grid with periodic boundary conditions, and filling the remaining entries with zeros.
Then, the dot product is understood as the sum of the elementwise multiplications of all entries of $x$ and $B_{k, i_x, i_y}$.
We remark that this operation is effectively an alternative way of describing a convolution with filter bank $B$, but this indexing choice proves useful in the proofs.
After applying a convolutional layer to the data, this network averages the resulting representation across each of the \( k \) output channels, and then takes a linear combination of these averages using a fully connected layer.
\end{definition}

\subsubsection{Theoretical Results on Equivariant Architectures}

We showed previously that the generalization behavior of MLPs on a symmetric dataset, rotated-MNIST, is well captured by the spectral error, a quantity computed from the Fourier components of the MLP kernel over the (orbits of the) cyclic group of interest.
Crucially, the derivation of this result relied on the circulant structure of the kernel matrix, which in turn was a consequence of the dot product nature of the MLP kernel.

Convolutional neural networks, however, are \textit{not} dot product kernels, as pixel proximity plays a role in the kernel similarity between two images \citep{arora2019exact}.
Nevertheless, here we show that the kernel matrix of convolutional neural networks \emph{is also circulant} over common representations of the cyclic group, regardless of whether the equivariance matches the symmetry of the data or not.
We thus get that the same spectral theory of generalization applies to these equivariant architectures.
Furthermore, a well-known special case arises when the network is designed to be fully invariant to the symmetry of interest, which is also well captured by our spectral theory.\\

First, we study the interplay between an equivariant architecture and its matching symmetry. 
\begin{proposition}
The kernel matrix of a fully connected convolutional network $K_\text{FC}$ over a translation orbit $O_{T}$ is circulant. Moreover, this kernel matrix is in general not constant or rank-deficient.
\end{proposition}

See App.~\ref{app:equi} for a proof of this proposition and the propositions below.\\

The kernel matrix of a fully connected convolutional network being circulant over a translation orbit, it can be analyzed in the Fourier domain, as done in the previous sections for a MLP. The same spectral formula for generalization error can thus be derived for this architecture and symmetry.
Moreover, the kernel matrix is in general not rank-deficient on the translation orbit. This implies that, \textit{a priori}, no inverse eigenvalue will diverge in the denominator of Eq.~\ref{eq:spectral_error}.
Were this to happen, the spectral error would go to 0, i.e., the network would achieve perfect generalization.
The same conclusion about the inability of a MLP to extrapolate symmetries to partially observed classes thus holds for a fully connected convolutional network on a dataset with translational symmetry.

\begin{proposition}
The kernel matrix of a global average pooling convolutional network $K_\text{GAP}$ over a translation orbit $O_{T}$ is constant.
\end{proposition}

The kernel matrix of a global average pooling convolutional network is not only circulant, but also constant.
The kernel matrix computed over a dataset made of two orbits of two different classes is thus blockwise-constant in 2x2 blocks, one for each orbit, and by consequence rank-deficient.
This rank-deficiency ensures that some of the eigenvalues of the kernel matrix are null, leading the denominator of the spectral error (Eq.~\ref{eq:spectral_error}) to diverge and thus perfect generalization (0 error).\footnote{This is true in the general case where the last frequency (at the numerator of the spectral formula) has non-zero power, i.e., the two classes are not fully collapsed in kernel space.}
We thus recover the well-known fact that a convolutional network with global average pooling at the last layer makes predictions that are by construction invariant to translations, and thus correctly generalizes from a partial view of a class orbit.\\

Second, we study the interplay between an equivariant architecture and a mismatching symmetry. 

\begin{proposition}
The kernel matrix of a fully connected convolutional network $K_\text{FC}$ over a rotation orbit $O_{R}$ is circulant, but in general not constant or rank-deficient. 
\end{proposition}\label{prop:kfc_rot}

For a fully connected convolutional network applied to a dataset with rotational symmetry, the kernel matrix over an orbit remains circulant. Therefore, the same spectral theory of generalization applies as before. 
The kernel matrix is in general neither constant nor rank-deficient, and thus generalization on the missing point is not guaranteed to succeed.\footnote{This result, however, does not preclude a neural \emph{classifier} from classifying the point correctly. Indeed, if the error is less than 1, a classifier would classify the point correctly (since the prediction is closer to the correct label than the other label which is -1 in our setup). Whether the neural classifier will classify the point correctly depends on the specific interplay between orbit density and class separation in kernel space, which depends on geometric properties of the data and of the architecture, as described by the spectral formula.}

\begin{proposition}
The kernel matrix of a global average pooling convolutional network $K_\text{GAP}$ over a rotation orbit $O_{R}$ is circulant, but in general not rank-deficient or constant.
\end{proposition}\label{prop:kgap_rot}

The implication of this last proposition is that (1) our spectral theory applies to this scenario as well, and (2) unlike for translations, the global average pooling layer does not guarantee perfect generalization on a rotation orbit.

\subsubsection{Empirical Results on Equivariant Architectures}

Empirically, we repeat the analyses previously performed on MLPs using convolutional architectures and obtain essentially the same results.

\begin{figure}[t]
    \centering
    \includegraphics[width=\linewidth]{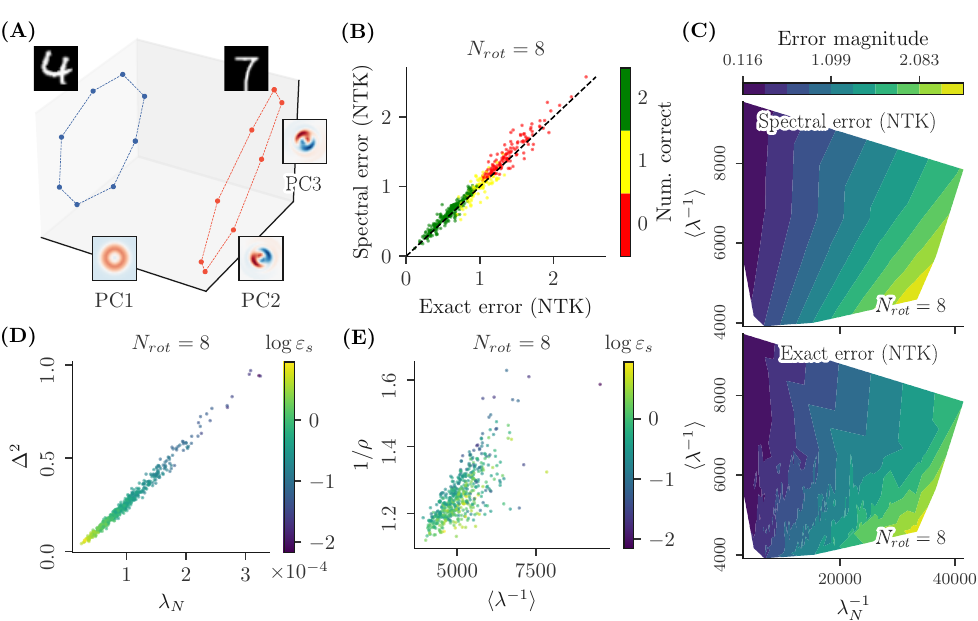}
    \caption{\textbf{Analysis of the prediction behavior of a fully connected convolutional network on pairs of orbits from rotated-MNIST.} Same caption as Fig. \ref{fig:highd_study}.}
    \label{fig:cntk_fc_main}
\end{figure}

\paragraph{Fully connected (FC) convolutional network on rotated-MNIST} We first analyze the behavior of a fully connected (FC) convolutional network, on the task of classifying a missing point from a dataset composed of two orbits, each generated from 8 rotations of a seed MNIST digit. The FC kernel matrix over a rotation orbit is circulant, as predicted by Prop. 3, allowing us to apply the spectral formula as we did for a MLP. The spectral error predicts well the symmetrized NTK error, across all pairs of orbits tested (Fig. \ref{fig:all_ntks}D). As with a MLP, the numerator and denominator values of the spectral error are approximately proportional to the distance between classes in input space and their average orbit density respectively (Fig. ~\ref{fig:cntk_fc_main}), allowing the same interpretation of the factors that lead to successful generalization as for the MLP.

\paragraph{Global average pooling (GAP) convolutional network on rotated-MNIST} We next consider a global average pooling convolutional network, on the same task and datasets. Its kernel matrix over a rotation orbit is also circulant, as predicted by Prop. 4, allowing us to apply the spectral theory.
The prediction errors are in general remarkably low (but, crucially, not 0) for this architecture on these datasets (Fig. \ref{fig:all_ntks}E). 

While we lack a comprehensive theory for this success, we speculate that the GAP convolutional network finds similarities in the local structure of images belonging to the same rotation orbit, allowing it to outperform other architectures. Interestingly, we find that a GAP convolutional network architecture where the filters don't overlap (i.e., the stride matches filter size) fails to correctly classify some pairs of orbits (App. Fig. \ref{fig:cntk_gap_fail}).

Moreover, the agreement between exact NTK error and spectral error that can be observed for other architectures is here worse, although still reasonable (both methods predict low error).
Our interpretation of why spectral error is not a very good approximation of exact NTK error here, is that the post-hoc circularization of the kernel matrix may reshape its structure too drastically. For instance, the geometric interleaving between the two orbits that is assumed by the circularization, may not be a good approximation in this kernel space. 

Finally, as shown in App.~\ref{app:more_cntks_rots}, the numerator and denominator values of the spectral error do not map well to the distance between
classes in input space and their average orbit density. \emph{This illustrates the important fact that, for some neural architectures, the spectral quantities computed in kernel space are not direct correlates of simple geometric quantities in input space, and cannot be interpreted as simply. Nevertheless, even in these cases, the error on a missing point is given by the spectral formula (with the approximations mentioned), where the quantities in the formula represent geometric quantities in kernel space.
}

\paragraph{Convolutional networks on translated-MNIST} We next analyze the behavior of convolutional networks on pairs of orbits of translated-MNIST, a version of MNIST where the digits are translated along the image x-axis, with periodic boundary condition (App.~\ref{app:more_cntks_shifts}).
Using a fully connected convolutional network, we do not see good generalization, as predicted by Prop. 1, and in good quantitative agreement with the spectral formula.
Using a global average pooling convolutional network, we see perfect generalization of the network on the leave-out class (error is 0 to numerical precision), as expected from the fact that this architecture is invariant to translation. 
In our spectral formula, this perfect generalization is realized by the fact that the denominator of Eq.~\ref{eq:spectral_error} diverges, as a consequence of the kernel matrix being rank-deficient (a consequence of Prop. 2).\\

\emph{In conclusion, we find that the generalization behavior of convolutional architectures on symmetric datasets is governed by the same formula and depends on the same geometric quantities as in the MLP case (namely orbit separation and density in kernel space). For highly non-linear architectures such as networks with a global-average pooling at the last layer, it is important to note that the geometric quantities in kernel space entering our spectral error are no longer trivially correlated to the same geometric quantities in input space, making the interpretation of these quantities kernel-specific. Finally, we find that only when the network is designed to be fully invariant to the symmetry of interest (e.g., GAP-CNN for translations), can it be theoretically guaranteed that the network will systematically generalize the symmetry correctly to unseen classes.}

\section{Discussion}\label{sec:discussion}

\textbf{In this work we establish a theory on when---and to what extent---deep networks are capable of learning symmetries from data.}
We find that the generalization behavior of conventional networks trained with supervision on datasets presenting a cyclic-group symmetry is captured by a simple ratio of inverse kernel frequency powers.
Our analysis of this formula sheds light on the limitations of conventional architectures trained with supervision to learn symmetries from data.
In particular, we find that they are generally unable to extrapolate symmetries observed exhaustively on some classes to other partially sampled classes.
Accurate generalization is only possible when the local structure of the data in kernel space (determined by network architecture) allows for correct generalization.
In other words, conventional networks have no mechanism to learn symmetries that have not been embedded in their architecture \emph{a priori} through equivariance. 
\medskip

\textbf{What are the practical implications of our results?} There may be different remedies to the issues described above, which may also depend on the specific constraints of the problem. We outline here these solutions briefly before entering a more thorough discussion. First, resorting to very large datasets or having prior knowledge about the symmetries of the problem so as to build them into the model via data augmentation or equivariance may suffice in some cases. When this is not possible, introducing soft biases encouraging the network to find equivariant structures may be a fruitful alternative. Pretraining with self-supervised learning losses using the symmetries of interest as augmentation may also alleviate the problem of generalization, but it would require further studies to understand how such pretraining affects learned representations. We cannot exclude that in the feature-rich learning regime, symmetries can be learned and generalized---at least in some cases. Finally, there may be some entirely new learning procedures which could somehow respect the symmetries of the data without explicitly building these symmetries into the model.\medskip

\textbf{Could neural networks learn symmetries outside the NTK regime?} Our theory is only valid in the NTK regime, also sometimes referred to as the `lazy regime', which assumes (1) infinite width, and (2) a certain scaling of the parameters with width (weights scale as $1/\sqrt{N}$ where $N$ is layer width). We found in our experiments that the theory also recapitulates well the training dynamics of finite-width networks using Pytorch default scaling of weights (Kaiming Uniform, weights scale as $1/\sqrt{N}$). However, we cannot exclude that for other initialization schemes, datasets (e.g., CIFAR-10, ImageNet) or network architectures (e.g., deeper MLPs), finite-width networks could successfully learn the symmetry of the problem. Furthermore, other regimes of weight scaling have been studied in the infinite limit, for example the `mean-field' regime (where weights scale as $1/N$) \citep{bach2016breakingcursedimensionalityconvex,mei2018meanfield, chizat2018global}, the `high-dimensional' regime \citep{saglietti2022analytical}, and the `dynamical mean-field' regime \citep{bordelon2022selfconsistent, yang2023tensorprogramsvifeature}. These regimes, often referred to as `feature learning' or `active' regimes, are characterized by a kernel that evolves in time, and thus learns features. This leads to different generalization behaviors both in the finite-width \citep{Geiger_2020} and infinite-width limit \citep{bordelon2022selfconsistent}. It would be interesting to study symmetry learning in these regimes. In a recent study, \citet{jacot2025dnnsbreakcursedimensionality} showed that deep networks in the active regime are able to learn data symmetries. However, their definition of symmetry is that of an invariant subspace warped by a non-linear function (through \emph{low-index functions}, i.e., functions of the form $f(x)=g(Ax)$ where $A$ is a low rank matrix). In contrast, we define a symmetry as the \emph{non-local, linear} action of a group on a dataset.\medskip

\textbf{Our results may seem at odds with the notion that deep networks are biased to find simple solutions} \citep{shah2020pitfalls, ortiz2020neural, zhang2021understanding, power2022grokking, humayun2024deep}.
In our study, networks are blind to the simple symmetric structure of the data.
Were they to notice this simple structure, they could achieve data-efficient generalization.
Importantly, results showing that networks have a bias towards simplicity typically define simplicity as a notion of local structure in the data.
These studies show that networks find the smoothest possible function capturing the training set.
Here, however, the symmetric structure of the data is not a local structure: rotated versions of a digit are not close in pixel space to the upright digit, and they may even be closer to a different digit class altogether.
For example, a `4' upside down might look more like a `6' than like another upright `4'.
It is thus expected that networks with a bias for smooth solutions will fail on such task, in the absence of a mechanism to detect symmetries.
Networks biased to find the \emph{shortest description} of a dataset (in the sense of Kolmogorov complexity) may perform better on this task \citep{valle-perez2018deep}, but progress in this area is limited. Bayesian model selection approaches may be a promising avenue to find such shortest description \citep{vanderwilk2018learninginvariancesusingmarginal, immer2022invariancelearningdeepneural,vanderouderaa2022learninginvariantweightsneural,vanderouderaa2023learninglayerwiseequivariancesautomatically}, and have recently been applied to learning data symmetries \citep{vanderouderaa2024noethersrazorlearningconserved}.
\medskip

\textbf{On the other hand, our findings are compatible with the scaling laws observed when training deep networks} \citep{kaplan2020scaling,dyerPNAS2024}.
Scaling laws show that deep networks continuously improve on generalization as the number of samples in the training set increases.
Mapped to our setup, this corresponds to our observation that network performance gradually increases with the number of angles sampled from the group orbits.
We see no abrupt change in generalization performance that would indicate that the network captures or ``groks" the symmetry invariance, even with very long training times.
\textbf{The absence of a mechanism to capture data symmetries could play a role in the notable data-inefficiency of current deep learning approaches}, inefficiency demonstrated by the evergrowing datasets used for training them (e.g., in computer vision JFT-300M \citep{Sun_2017_ICCV}, IG-3.6B \citep{singh2022revisiting} and LAION-5B \citep{schuhmann2022laion}).
As the number of symmetric transformations present in the data increases (e.g., image rotation, translation, scaling, etc.), we only expect the generalization difficulty to increase, as this leads to a combinatorial explosion of possible transformations \citep{schott2022visual}.
This combinatorial explosion could partly be responsible for the brittleness of current deep learning approaches on edge cases, as observed for instance in \cite{Abbas_Deny_2023}, where it is shown that deep networks for vision fare particularly poorly on \emph{combinations} of symmetric transformations such as scaling and rotation.
\medskip

\textbf{Many methods have been proposed to learn symmetries from data.}
For example, some propose to leverage more supervision by pairing images undergoing a symmetric transformation during training.
This is the case, for instance, of self-supervised learning approaches via joint embedding \citep{geiping2023cookbook}, and of autoencoder approaches \citep{dupont2020equivariant, connor2020representing,keller2021topographic, connor2024learning}, the goal of which is to map one sample to a transformed version of itself via an autoencoder equipped with transformation operators in latent space.
Empirically, these methods are found to somewhat generalize group transformations (e.g., 3D rotations) to object classes that were not seen during training.
Other directions include meta-learning approaches \citep{pmlr-v70-finn17a,yang2020feature}, feature learning approaches \citep{bach2016breakingcursedimensionalityconvex,yang2023tensorprogramsvifeature,jacot2025dnnsbreakcursedimensionality} and dynamic architecture choices \citep{stanley:ec02, Chauhan2024}, which could potentially alleviate the problem of having a fixed, frozen, kernel incapable of adapting to problem symmetries. 
Indeed, our understanding is that networks cannot learn symmetries because their kernel is defined by their architecture, and cannot adapt to the symmetries of a given dataset.
Designing mechanisms to adapt the kernel induced by the network to the symmetries of the problem at stake seems a likely path forward to devise more data-efficient deep networks.
In future work, our theory could be extended to these methods, in order to establish whether they are theoretically able to learn symmetries from data, and if so, to what extent.
\medskip

\textbf{Our theory could also be extended in various ways to capture a richer phenomenology.}
First, we have mainly focused  here on the simple discrete one-dimensional cyclic group. However, our theoretical framework is readily extendable to any finite group including non-abelian groups (see Theorem \ref{thm:nonabelian}), and it would be interesting to study the practical implications of the theory for such more complex groups.
It would also be interesting to understand whether and how the theory could be extended to continuous groups, such as $SO(3)$.
Second, our study focuses on symmetries that are native to the space in which the datasets reside (e.g., image rotations).
However, group actions may more realistically exist in a latent space affecting the dataset indirectly (e.g., images of 3D-rotated objects).
This latter case is particularly interesting because architectures cannot easily be designed to be equivariant to symmetries which are not directly acting in dataset space.
We note, however, that given the failure of conventional architectures to correctly generalize simple symmetries native to the dataset space, we do not expect these same architectures to be able to learn more complex, latent space symmetries.
\medskip

\textbf{A source of inspiration for learning symmetries could be found in cognitive science and neuroscience.}
Empirically, there is evidence that humans are superior to deep networks at some tasks necessitating symmetry invariance. In a recent example, \cite{ollikka2025humans} showed that humans beat state-of-the-art deep networks and most vision-language models at recognizing objects in unusual poses.
Humans are also known to be able to reason about problem symmetries, for example in the problem of mental rotation \citep{Shepard1971}.
Interestingly, the time for subjects to compare one 3D object to another in mental rotation tasks is proportional to the angular difference between the two objects, implying that some recurrent processes in the brain may be involved, recurrence lacking from current state-of-the art architectures in deep learning.
There is also evidence that the mental rotation ability is learned through experience (or at least not fully operational at birth), as studies on infants show that there is a critical age when they are able to perform this task \citep{Bambha27052022}. Conversely, children learning to read need to first unlearn mirror symmetries in order to differentiate characters from their mirrored version \citep{perea2011,pegado2011}.
The ability to capture specific problem symmetries may, however, be partially or fully innate in some animals.
A study on chicks show that they have the ability to recognize an unknown synthetic 3D object from birth, across multiple points of view \citep{wood2013}.
This finding prompts us to reconsider whether humans and animals actively learn problem symmetries, or rely on neural architectures which are pre-configured for specific symmetries.
It is at least clear from behavioral experiments that humans have the ability to adapt to unnatural symmetric transformations of their environment.
Classically, Kohler and Erismann \citep{kohler1963formation} showed that a subject wearing goggles reversing the world upside down can adapt to this new environment after a few days of practice.
The extent to which this adaptation is possible for new, previously unseen symmetries is unclear, however (e.g., what if the goggles permanently permuted the location of every pixel?).
Finally, a direction of interest may be to imitate the physicist: by describing the world symbolically, the physicist comes to discover the symmetries of the world and exploit them for predictions.
Language models and other neuro-symbolic approaches may be able to partly replicate this ability, or at least piggy-back on human accumulated knowledge about the symmetries of the world.

\acks{We thank David Radnell, Milo Orlich, Samuel A. Ocko, Jaakko Lehtinen, Tycho F. A. van der Ouderaa, Arthur Jacot, Jan Gerken, Stefano Sarao Mannelli for useful discussions and additional references, Luigi Acerbi and Martin Trapp for comments on the manuscript, Adityanarayanan Radhakrishnan for his clear and concise lecture notes on neural kernels, and Fabio Anselmi for an insight that led to the proof of Theorem 15. We thank the anonymous reviewers and our editor Aapo Hyvärinen for their constructive feedback and suggestions, and for their time.  Funding: Research Council of Finland grant to S.D. under the Project ``Neuroscience-inspired Deep Learning'': 3357590.} 

\bibliography{24-2175}

\begin{thebibliography}{126}
\providecommand{\natexlab}[1]{#1}
\providecommand{\url}[1]{\texttt{#1}}
\expandafter\ifx\csname urlstyle\endcsname\relax
  \providecommand{\doi}[1]{doi: #1}\else
  \providecommand{\doi}{doi: \begingroup \urlstyle{rm}\Url}\fi

\bibitem[Abbas and Deny(2023)]{Abbas_Deny_2023}
Amro Abbas and St\'{e}phane Deny.
\newblock Progress and limitations of deep networks to recognize objects in
  unusual poses.
\newblock In \emph{Conference on Artificial Intelligence (AAAI)}, 2023.

\bibitem[Ainsworth et~al.(2023)Ainsworth, Hayase, and
  Srinivasa]{ainsworth2023git}
Samuel Ainsworth, Jonathan Hayase, and Siddhartha Srinivasa.
\newblock Git {Re-Basin}: {Merging} models modulo permutation symmetries.
\newblock In \emph{International Conference on Learning Representations
  (ICLR)}, 2023.

\bibitem[Alcorn et~al.(2019)Alcorn, Li, Gong, Wang, Mai, Ku, and
  Nguyen]{strikewithapose}
Michael~A Alcorn, Qi~Li, Zhitao Gong, Chengfei Wang, Long Mai, Wei-Shinn Ku,
  and Anh Nguyen.
\newblock Strike (with) a pose: {Neural} networks are easily fooled by strange
  poses of familiar objects.
\newblock In \emph{Conference on Computer Vision and Pattern Recognition
  (CVPR)}, 2019.

\bibitem[Anselmi et~al.(2019)Anselmi, Evangelopoulos, Rosasco, and
  Poggio]{Anselmi2019}
Fabio Anselmi, Georgios Evangelopoulos, Lorenzo Rosasco, and Tomaso Poggio.
\newblock Symmetry-adapted representation learning.
\newblock \emph{Pattern Recognition}, 2019.

\bibitem[Anselmi et~al.(2023)Anselmi, Manzoni, d’Onofrio, Rodriguez,
  Caravagna, Bortolussi, and Cairoli]{anselmi2023data}
Fabio Anselmi, Luca Manzoni, Alberto d’Onofrio, Alex Rodriguez, Giulio
  Caravagna, Luca Bortolussi, and Francesca Cairoli.
\newblock Data symmetries and learning in fully connected neural networks.
\newblock \emph{IEEE Access}, 2023.

\bibitem[Arora et~al.(2019)Arora, Du, Hu, Li, Salakhutdinov, and
  Wang]{arora2019exact}
Sanjeev Arora, Simon~S Du, Wei Hu, Zhiyuan Li, Russ~R Salakhutdinov, and
  Ruosong Wang.
\newblock On exact computation with an infinitely wide neural net.
\newblock In \emph{Advances in Neural Information Processing Systems
  (NeurIPS)}, 2019.

\bibitem[Avidan et~al.(2025)Avidan, Li, and
  Sompolinsky]{avidan2025connectingntknngpunified}
Yehonatan Avidan, Qianyi Li, and Haim Sompolinsky.
\newblock Unified theoretical framework for wide neural network learning
  dynamics.
\newblock \emph{Physical Review E}, 2025.

\bibitem[Azulay and Weiss(2019)]{azulay2019deep}
Aharon Azulay and Yair Weiss.
\newblock Why do deep convolutional networks generalize so poorly to small
  image transformations?
\newblock \emph{Journal of Machine Learning Research (JMLR)}, 2019.

\bibitem[Bach(2017)]{bach2016breakingcursedimensionalityconvex}
Francis Bach.
\newblock Breaking the curse of dimensionality with convex neural networks.
\newblock \emph{Journal of Machine Learning Research (JMLR)}, 2017.

\bibitem[Bahri et~al.(2024)Bahri, Dyer, Kaplan, Lee, and Sharma]{dyerPNAS2024}
Yasaman Bahri, Ethan Dyer, Jared Kaplan, Jaehoon Lee, and Utkarsh Sharma.
\newblock Explaining neural scaling laws.
\newblock \emph{Proceedings of the National Academy of Sciences (PNAS)}, 2024.

\bibitem[Balestriero et~al.(2023)Balestriero, Ibrahim, Sobal, Morcos, Shekhar,
  Goldstein, Bordes, Bardes, Mialon, Tian, Schwarzschild, Wilson, Geiping,
  Garrido, Fernandez, Bar, Pirsiavash, LeCun, and
  Goldblum]{geiping2023cookbook}
Randall Balestriero, Mark Ibrahim, Vlad Sobal, Ari Morcos, Shashank Shekhar,
  Tom Goldstein, Florian Bordes, Adrien Bardes, Gregoire Mialon, Yuandong Tian,
  Avi Schwarzschild, Andrew~Gordon Wilson, Jonas Geiping, Quentin Garrido,
  Pierre Fernandez, Amir Bar, Hamed Pirsiavash, Yann LeCun, and Micah Goldblum.
\newblock A cookbook of self-supervised learning.
\newblock \emph{arXiv preprint}, 2023.

\bibitem[Bambha et~al.(2022)Bambha, Beckner, Shetty, Voss, Xie, Yiu, LoBue,
  Oakes, and Casasola]{Bambha27052022}
Valerie Bambha, Aaron Beckner, Nikita Shetty, Annika Voss, Jinlin Xie, Eunice
  Yiu, Vanessa LoBue, Lisa Oakes, and Marianella Casasola.
\newblock Developmental changes in children’s object insertions during play.
\newblock \emph{Journal of Cognition and Development}, 2022.

\bibitem[Bekkers(2021)]{bekkers2021bsplinecnnsliegroups}
Erik~J Bekkers.
\newblock B-spline {CNNs} on {Lie} groups.
\newblock \emph{arXiv preprint}, 2021.

\bibitem[Benton et~al.(2020)Benton, Finzi, Izmailov, and
  Wilson]{benton2020learnedaugmentations}
Gregory Benton, Marc Finzi, Pavel Izmailov, and Andrew~G Wilson.
\newblock Learning invariances in neural networks from training data.
\newblock In \emph{Advances in Neural Information Processing Systems
  (NeurIPS)}, 2020.

\bibitem[Bordelon and Pehlevan(2022)]{bordelon2022selfconsistent}
Blake Bordelon and Cengiz Pehlevan.
\newblock Self-consistent dynamical field theory of kernel evolution in wide
  neural networks.
\newblock In \emph{Advances in Neural Information Processing Systems
  (NeurIPS)}, 2022.

\bibitem[Bordelon et~al.(2020)Bordelon, Canatar, and
  Pehlevan]{bordelonPMLR2020a}
Blake Bordelon, Abdulkadir Canatar, and Cengiz Pehlevan.
\newblock Spectrum dependent learning curves in kernel regression and wide
  neural networks.
\newblock In \emph{International Conference on Machine Learning (ICML)}, 2020.

\bibitem[Bouchacourt et~al.(2021)Bouchacourt, Ibrahim, and
  Deny]{bouchacourt2021addressingtopologicaldefectsdisentanglement}
Diane Bouchacourt, Mark Ibrahim, and Stéphane Deny.
\newblock Addressing the topological defects of disentanglement via distributed
  operators.
\newblock \emph{arXiv preprint}, 2021.

\bibitem[Brehmer et~al.(2024)Brehmer, Behrends, de~Haan, and
  Cohen]{brehmer2024doesequivariancematterscale}
Johann Brehmer, Sönke Behrends, Pim de~Haan, and Taco Cohen.
\newblock Does equivariance matter at scale?
\newblock \emph{arXiv preprint}, 2024.

\bibitem[Bronstein et~al.(2021)Bronstein, Bruna, Cohen, and
  Veličković]{bronstein2021geometricdeeplearninggrids}
Michael~M. Bronstein, Joan Bruna, Taco Cohen, and Petar Veličković.
\newblock Geometric deep learning: {Grids}, groups, graphs, geodesics, and
  gauges.
\newblock \emph{arXiv preprint}, 2021.

\bibitem[Canatar et~al.(2021)Canatar, Bordelon, and
  Pehlevan]{canatar2021spectral}
Abdulkadir Canatar, Blake Bordelon, and Cengiz Pehlevan.
\newblock Spectral bias and task-model alignment explain generalization in
  kernel regression and infinitely wide neural networks.
\newblock \emph{Nature Communications}, 2021.

\bibitem[Chauhan et~al.(2024)Chauhan, Zhou, Lu, Molaei, and
  Clifton]{Chauhan2024}
Vinod~Kumar Chauhan, Jiandong Zhou, Ping Lu, Soheila Molaei, and David~A.
  Clifton.
\newblock A brief review of hypernetworks in deep learning.
\newblock \emph{Artificial Intelligence Review}, 2024.

\bibitem[Chen et~al.(2020)Chen, Bai, Zhao, Ament, Gregoire, and
  Gomes]{chen2020simple}
Di~Chen, Yiwei Bai, Wenting Zhao, Sebastian Ament, John Gregoire, and Carla
  Gomes.
\newblock Deep reasoning networks for unsupervised pattern de-mixing with
  constraint reasoning.
\newblock In \emph{International Conference on Machine Learning (ICML)}, 2020.

\bibitem[Chirikjian and Kyatkin(2021)]{Chirikjian2021}
Gregory~S. Chirikjian and Alexander~B. Kyatkin.
\newblock \emph{Engineering Applications of Noncommutative Harmonic Analysis}.
\newblock CRC Press, 2021.

\bibitem[Chizat and Bach(2018)]{chizat2018global}
L\'{e}na\"{\i}c Chizat and Francis Bach.
\newblock On the global convergence of gradient descent for over-parameterized
  models using optimal transport.
\newblock In \emph{Advances in Neural Information Processing Systems
  (NeurIPS)}, 2018.

\bibitem[Chung et~al.(2018)Chung, Lee, and
  Sompolinsky]{chung2018classification}
SueYeon Chung, Daniel~D. Lee, and Haim Sompolinsky.
\newblock Classification and geometry of general perceptual manifolds.
\newblock \emph{Physical Review X}, 2018.

\bibitem[Cohen and Welling(2016)]{cohenICML2016}
Taco Cohen and Max Welling.
\newblock Group equivariant convolutional networks.
\newblock In \emph{International Conference on Machine Learning (ICML)}, 2016.

\bibitem[Cohen et~al.(2019{\natexlab{a}})Cohen, Weiler, Kicanaoglu, and
  Welling]{cohen2019gauge}
Taco Cohen, Maurice Weiler, Berkay Kicanaoglu, and Max Welling.
\newblock Gauge equivariant convolutional networks and the icosahedral {CNN}.
\newblock In \emph{International Conference on Machine Learning (ICML)},
  2019{\natexlab{a}}.

\bibitem[Cohen et~al.(2019{\natexlab{b}})Cohen, Geiger, and
  Weiler]{cohen2020generaltheoryequivariantcnns}
Taco~S Cohen, Mario Geiger, and Maurice Weiler.
\newblock A general theory of equivariant {CNNs} on homogeneous spaces.
\newblock In \emph{Advances in Neural Information Processing Systems
  (NeurIPS)}, 2019{\natexlab{b}}.

\bibitem[Cohen et~al.(2020)Cohen, Chung, Lee, and
  Sompolinsky]{cohen2020separability}
Uri Cohen, SueYeon Chung, Daniel~D. Lee, and Haim Sompolinsky.
\newblock Separability and geometry of object manifolds in deep neural
  networks.
\newblock \emph{Nature Communications}, 2020.

\bibitem[Connor and Rozell(2020)]{connor2020representing}
Marissa Connor and Christopher Rozell.
\newblock Representing closed transformation paths in encoded network latent
  space.
\newblock \emph{Conference on Artificial Intelligence (AAAI)}, 2020.

\bibitem[Connor et~al.(2024)Connor, Olshausen, and Rozell]{connor2024learning}
Marissa Connor, Bruno Olshausen, and Christopher Rozell.
\newblock Learning internal representations of {3D} transformations from {2D}
  projected inputs.
\newblock \emph{Neural Computation}, 2024.

\bibitem[Cranmer et~al.(2020)Cranmer, Greydanus, Hoyer, Battaglia, Spergel, and
  Ho]{cranmer2020lagrangian}
Miles Cranmer, Sam Greydanus, Stephan Hoyer, Peter Battaglia, David Spergel,
  and Shirley Ho.
\newblock Lagrangian neural networks.
\newblock In \emph{ICLR Workshop on Integration of Deep Neural Models and
  Differential Equations}, 2020.

\bibitem[Culpepper and Olshausen(2009)]{culpepper2009learning}
Benjamin Culpepper and Bruno Olshausen.
\newblock Learning transport operators for image manifolds.
\newblock In \emph{Advances in Neural Information Processing Systems
  (NeurIPS)}, 2009.

\bibitem[D'Ascoli et~al.(2021)D'Ascoli, Touvron, Leavitt, Morcos, Biroli, and
  Sagun]{d2021convit}
St{\'e}phane D'Ascoli, Hugo Touvron, Matthew~L Leavitt, Ari~S Morcos, Giulio
  Biroli, and Levent Sagun.
\newblock Convit: {Improving} vision transformers with soft convolutional
  inductive biases.
\newblock In \emph{International Conference on Machine Learning (ICML)}, 2021.

\bibitem[Dosovitskiy et~al.(2021)Dosovitskiy, Beyer, Kolesnikov, Weissenborn,
  Zhai, Unterthiner, Dehghani, Minderer, Heigold, Gelly, Uszkoreit, and
  Houlsby]{dosovitskiy2020image}
Alexey Dosovitskiy, Lucas Beyer, Alexander Kolesnikov, Dirk Weissenborn,
  Xiaohua Zhai, Thomas Unterthiner, Mostafa Dehghani, Matthias Minderer, Georg
  Heigold, Sylvain Gelly, Jakob Uszkoreit, and Neil Houlsby.
\newblock An image is worth 16x16 words: {Transformers} for image recognition
  at scale.
\newblock In \emph{International Conference on Learning Representations
  (ICLR)}, 2021.

\bibitem[Dupont et~al.(2020)Dupont, Martin, Colburn, Sankar, Susskind, and
  Shan]{dupont2020equivariant}
Emilien Dupont, Miguel~Bautista Martin, Alex Colburn, Aditya Sankar, Josh
  Susskind, and Qi~Shan.
\newblock Equivariant neural rendering.
\newblock In \emph{International Conference on Machine Learning (ICML)}, 2020.

\bibitem[Elsayed et~al.(2020)Elsayed, Ramachandran, Shlens, and
  Kornblith]{elsayed2020revisiting}
Gamaleldin Elsayed, Prajit Ramachandran, Jonathon Shlens, and Simon Kornblith.
\newblock Revisiting spatial invariance with low-rank local connectivity.
\newblock In \emph{International Conference on Machine Learning (ICML)}, 2020.

\bibitem[Esmaeili et~al.(2023)Esmaeili, Walters, Zimmermann, and van~de
  Meent]{esmaeili2023topologicalobstructionsavoid}
Babak Esmaeili, Robin Walters, Heiko Zimmermann, and Jan-Willem van~de Meent.
\newblock Topological obstructions and how to avoid them.
\newblock In \emph{Advances in Neural Information Processing Systems
  (NeurIPS)}, 2023.

\bibitem[Farrell et~al.(2022)Farrell, Bordelon, Trivedi, and
  Pehlevan]{Farrell2021}
Matthew Farrell, Blake Bordelon, Shubhendu Trivedi, and Cengiz Pehlevan.
\newblock Capacity of group-invariant linear readouts from equivariant
  representations: {How} many objects can be linearly classified under all
  possible views?
\newblock In \emph{International Conference on Learning Representations
  (ICLR)}, 2022.

\bibitem[Finn et~al.(2017)Finn, Abbeel, and Levine]{pmlr-v70-finn17a}
Chelsea Finn, Pieter Abbeel, and Sergey Levine.
\newblock Model-agnostic meta-learning for fast adaptation of deep networks.
\newblock In \emph{International Conference on Machine Learning (ICML)}, 2017.

\bibitem[Finzi et~al.(2020)Finzi, Stanton, Izmailov, and
  Wilson]{finzi2020generalizing}
Marc Finzi, Samuel Stanton, Pavel Izmailov, and Andrew~Gordon Wilson.
\newblock Generalizing convolutional neural networks for equivariance to {Lie}
  groups on arbitrary continuous data.
\newblock In \emph{International Conference on Machine Learning (ICML)}, 2020.

\bibitem[Fort et~al.(2020)Fort, Dziugaite, Paul, Kharaghani, Roy, and
  Ganguli]{fortNEURIPS2020}
Stanislav Fort, Gintare~Karolina Dziugaite, Mansheej Paul, Sepideh Kharaghani,
  Daniel~M Roy, and Surya Ganguli.
\newblock Deep learning versus kernel learning: {An} empirical study of loss
  landscape geometry and the time evolution of the neural tangent kernel.
\newblock In \emph{Advances in Neural Information Processing Systems
  (NeurIPS)}, 2020.

\bibitem[Garrido et~al.(2024)Garrido, Assran, Ballas, Bardes, Najman, and
  LeCun]{garrido2024learningleveragingworldmodels}
Quentin Garrido, Mahmoud Assran, Nicolas Ballas, Adrien Bardes, Laurent Najman,
  and Yann LeCun.
\newblock Learning and leveraging world models in visual representation
  learning.
\newblock \emph{arXiv preprint}, 2024.

\bibitem[Geiger et~al.(2020)Geiger, Spigler, Jacot, and Wyart]{Geiger_2020}
Mario Geiger, Stefano Spigler, Arthur Jacot, and Matthieu Wyart.
\newblock Disentangling feature and lazy training in deep neural networks.
\newblock \emph{Journal of Statistical Mechanics: Theory and Experiment}, 2020.

\bibitem[Gens and Domingos(2014)]{gensNIPS2014}
Robert Gens and Pedro Domingos.
\newblock Deep symmetry networks.
\newblock In \emph{Advances in Neural Information Processing Systems
  (NeurIPS)}, 2014.

\bibitem[Gerace et~al.(2022)Gerace, Saglietti, Sarao~Mannelli, Saxe, and
  Zdeborová]{gerace2022probing}
Federica Gerace, Luca Saglietti, Stefano Sarao~Mannelli, Andrew Saxe, and Lenka
  Zdeborová.
\newblock Probing transfer learning with a model of synthetic correlated
  datasets.
\newblock \emph{Machine Learning: Science and Technology}, 2022.

\bibitem[Gerken and Kessel(2024)]{gerkenPMLR2024a}
Jan~E Gerken and Pan Kessel.
\newblock Emergent equivariance in deep ensembles.
\newblock In \emph{International Conference on Machine Learning (ICML)}, 2024.

\bibitem[Goldt et~al.(2020)Goldt, M\'ezard, Krzakala, and
  Zdeborov\'a]{goldt2020modeling}
Sebastian Goldt, Marc M\'ezard, Florent Krzakala, and Lenka Zdeborov\'a.
\newblock Modeling the influence of data structure on learning in neural
  networks: {The} hidden manifold model.
\newblock \emph{Physical Review X}, 2020.

\bibitem[Greydanus et~al.(2019)Greydanus, Dzamba, and
  Yosinski]{greydanus2019hamiltonian}
Samuel Greydanus, Misko Dzamba, and Jason Yosinski.
\newblock Hamiltonian neural networks.
\newblock In \emph{Advances in Neural Information Processing Systems
  (NeurIPS)}, 2019.

\bibitem[Gross(1996)]{gross1996role}
David~J. Gross.
\newblock The role of symmetry in fundamental physics.
\newblock \emph{Proceedings of the National Academy of Sciences (PNAS)}, 1996.

\bibitem[Gruver et~al.(2023)Gruver, Finzi, Goldblum, and Wilson]{gruver2023the}
Nate Gruver, Marc~Anton Finzi, Micah Goldblum, and Andrew~Gordon Wilson.
\newblock The {Lie} derivative for measuring learned equivariance.
\newblock In \emph{International Conference on Learning Representations
  (ICLR)}, 2023.

\bibitem[Hajij et~al.(2023)Hajij, Zamzmi, Papamarkou, Miolane, Guzmán-Sáenz,
  Ramamurthy, Birdal, Dey, Mukherjee, Samaga, Livesay, Walters, Rosen, and
  Schaub]{hajij2023topologicaldeeplearninggoing}
Mustafa Hajij, Ghada Zamzmi, Theodore Papamarkou, Nina Miolane, Aldo
  Guzmán-Sáenz, Karthikeyan~Natesan Ramamurthy, Tolga Birdal, Tamal~K. Dey,
  Soham Mukherjee, Shreyas~N. Samaga, Neal Livesay, Robin Walters, Paul Rosen,
  and Michael~T. Schaub.
\newblock Topological deep learning: {Going} beyond graph data.
\newblock \emph{arXiv preprint}, 2023.

\bibitem[Higgins et~al.(2018)Higgins, Amos, Pfau, Racaniere, Matthey, Rezende,
  and Lerchner]{higgins2018definitiondisentangledrepresentations}
Irina Higgins, David Amos, David Pfau, Sebastien Racaniere, Loic Matthey,
  Danilo Rezende, and Alexander Lerchner.
\newblock Towards a definition of disentangled representations.
\newblock \emph{arXiv preprint}, 2018.

\bibitem[Humayun et~al.(2024)Humayun, Balestriero, and
  Baraniuk]{humayun2024deep}
Ahmed~Imtiaz Humayun, Randall Balestriero, and Richard Baraniuk.
\newblock Deep networks always grok and here is why.
\newblock In \emph{International Conference on Machine Learning (ICML)}, 2024.

\bibitem[Ibrahim et~al.(2022)Ibrahim, Bouchacourt, and
  Morcos]{ibrahim2022robustselfsupervisedlearninglie}
Mark Ibrahim, Diane Bouchacourt, and Ari Morcos.
\newblock Robust self-supervised learning with {Lie} groups.
\newblock \emph{arXiv preprint}, 2022.

\bibitem[Ibrahim et~al.(2023)Ibrahim, Garrido, Morcos, and
  Bouchacourt]{ibrahim2023robustness}
Mark Ibrahim, Quentin Garrido, Ari~S. Morcos, and Diane Bouchacourt.
\newblock The robustness limits of so{TA} vision models to natural variation.
\newblock \emph{Transactions on Machine Learning Research (TMLR)}, 2023.

\bibitem[Immer et~al.(2022)Immer, van~der Ouderaa, R\"{a}tsch, Fortuin, and
  van~der Wilk]{immer2022invariancelearningdeepneural}
Alexander Immer, Tycho~F.A. van~der Ouderaa, Gunnar R\"{a}tsch, Vincent
  Fortuin, and Mark van~der Wilk.
\newblock Invariance learning in deep neural networks with differentiable
  laplace approximations.
\newblock In \emph{Advances in Neural Information Processing Systems
  (NeurIPS)}, 2022.

\bibitem[Jacot et~al.(2018)Jacot, Gabriel, and Hongler]{jacotNIPS2018}
Arthur Jacot, Franck Gabriel, and Clement Hongler.
\newblock Neural tangent kernel: {Convergence} and generalization in neural
  networks.
\newblock In \emph{Advances in Neural Information Processing Systems
  (NeurIPS)}, 2018.

\bibitem[Jacot et~al.(2025)Jacot, Choi, and
  Wen]{jacot2025dnnsbreakcursedimensionality}
Arthur Jacot, Seok~Hoan Choi, and Yuxiao Wen.
\newblock How {DNN}s break the curse of dimensionality: {Compositionality} and
  symmetry learning.
\newblock In \emph{International Conference on Learning Representations
  (ICLR)}, 2025.

\bibitem[Jaderberg et~al.(2015)Jaderberg, Simonyan, Zisserman, and
  kavukcuoglu]{jaderberg2015spatial}
Max Jaderberg, Karen Simonyan, Andrew Zisserman, and koray kavukcuoglu.
\newblock Spatial transformer networks.
\newblock In \emph{Advances in Neural Information Processing Systems
  (NeurIPS)}, 2015.

\bibitem[Kaba and Ravanbakhsh(2023)]{kaba2023symmetry}
S{\'e}kou-Oumar Kaba and Siamak Ravanbakhsh.
\newblock Symmetry breaking and equivariant neural networks.
\newblock In \emph{NeurIPS Workshop on Symmetry and Geometry in Neural
  Representations}, 2023.

\bibitem[Kaplan et~al.(2020)Kaplan, McCandlish, Henighan, Brown, Chess, Child,
  Gray, Radford, Wu, and Amodei]{kaplan2020scaling}
Jared Kaplan, Sam McCandlish, Tom Henighan, Tom~B. Brown, Benjamin Chess, Rewon
  Child, Scott Gray, Alec Radford, Jeffrey Wu, and Dario Amodei.
\newblock Scaling laws for neural language models.
\newblock \emph{arXiv preprint}, 2020.

\bibitem[Keller and Welling(2021{\natexlab{a}})]{keller2021topographic}
T.~Anderson Keller and Max Welling.
\newblock Topographic vaes learn equivariant capsules.
\newblock In \emph{Advances in Neural Information Processing Systems
  (NeurIPS)}, 2021{\natexlab{a}}.

\bibitem[Keller and Welling(2021{\natexlab{b}})]{kellerNEURIPS2021}
T.~Anderson Keller and Max Welling.
\newblock Topographic vaes learn equivariant capsules.
\newblock In \emph{Advances in Neural Information Processing Systems
  (NeurIPS)}, 2021{\natexlab{b}}.

\bibitem[Kohler(1963)]{kohler1963formation}
Ivo Kohler.
\newblock The formation and transformation of the perceptual world.
\newblock \emph{Psychological issues}, 1963.

\bibitem[LeCun et~al.(1998)LeCun, Bottou, Bengio, and Haffner]{mnist}
Y.~LeCun, L.~Bottou, Y.~Bengio, and P.~Haffner.
\newblock Gradient-based learning applied to document recognition.
\newblock \emph{Proceedings of the IEEE}, 1998.

\bibitem[Lee et~al.(2018)Lee, Sohl-dickstein, Pennington, Novak, Schoenholz,
  and Bahri]{lee2018deep}
Jaehoon Lee, Jascha Sohl-dickstein, Jeffrey Pennington, Roman Novak, Sam
  Schoenholz, and Yasaman Bahri.
\newblock Deep neural networks as {Gaussian} processes.
\newblock In \emph{International Conference on Learning Representations
  (ICLR)}, 2018.

\bibitem[Liang et~al.(2024)Liang, Liu, Ostrow, and Fiete]{liang2024how}
Qiyao Liang, Ziming Liu, Mitchell Ostrow, and Ila Fiete.
\newblock How diffusion models learn to factorize and compose.
\newblock In \emph{Advances in Neural Information Processing Systems
  (NeurIPS)}, 2024.

\bibitem[Lippl and Stachenfeld(2025)]{lippl2024doescompositionalstructureyield}
Samuel Lippl and Kim Stachenfeld.
\newblock When does compositional structure yield compositional generalization?
  {A} kernel theory.
\newblock In \emph{International Conference on Learning Representations
  (ICLR)}, 2025.

\bibitem[Madan et~al.(2022)Madan, Henry, Dozier, Ho, Bhandari, Sasaki, Durand,
  Pfister, and Boix]{madan2022ood}
Spandan Madan, Timothy Henry, Jamell Dozier, Helen Ho, Nishchal Bhandari,
  Tomotake Sasaki, Frédo Durand, Hanspeter Pfister, and Xavier Boix.
\newblock When and how convolutional neural networks generalize to
  out-of-distribution category–viewpoint combinations.
\newblock \emph{Nature Machine Intelligence}, 2022.

\bibitem[Madan et~al.(2025)Madan, Sasaki, Pfister, Li, and
  Boix]{madan2021small}
Spandan Madan, Tomotake Sasaki, Hanspeter Pfister, Tzu-Mao Li, and Xavier Boix.
\newblock In-distribution adversarial attacks on object recognition models
  using gradient-free search.
\newblock \emph{arXiv preprint}, 2025.

\bibitem[Marchetti et~al.(2024)Marchetti, Hillar, Kragic, and
  Sanborn]{marchetti2024harmonicslearninguniversalfourier}
Giovanni~Luca Marchetti, Christopher~J Hillar, Danica Kragic, and Sophia
  Sanborn.
\newblock Harmonics of learning: {Universal} fourier features emerge in
  invariant networks.
\newblock In \emph{Conference on Learning Theory (COLT)}, 2024.

\bibitem[Mei et~al.(2018)Mei, Montanari, and Nguyen]{mei2018meanfield}
Song Mei, Andrea Montanari, and Phan-Minh Nguyen.
\newblock A mean field view of the landscape of two-layer neural networks.
\newblock \emph{Proceedings of the National Academy of Sciences (PNAS)}, 2018.

\bibitem[Mel and Ganguli(2021)]{mel2021theory}
Gabriel Mel and Surya Ganguli.
\newblock A theory of high dimensional regression with arbitrary correlations
  between input features and target functions: {Sample} complexity, multiple
  descent curves and a hierarchy of phase transitions.
\newblock In \emph{International Conference on Machine Learning (ICML)}, 2021.

\bibitem[Mercatali et~al.(2022)Mercatali, Freitas, and
  Garg]{mercatali2022symmetry}
Giangiacomo Mercatali, Andre Freitas, and Vikas Garg.
\newblock Symmetry-induced disentanglement on graphs.
\newblock In \emph{Advances in Neural Information Processing Systems
  (NeurIPS)}, 2022.

\bibitem[Moskalev et~al.(2023)Moskalev, Sepliarskaia, Bekkers, and
  Smeulders]{moskalevPMLR2023a}
Artem Moskalev, Anna Sepliarskaia, Erik~J Bekkers, and Arnold~W.M. Smeulders.
\newblock On genuine invariance learning without weight-tying.
\newblock In \emph{Proceedings of 2nd Annual Workshop on Topology, Algebra, and
  Geometry in Machine Learning (TAG-ML)}, 2023.

\bibitem[Neal(1996)]{Neal1996}
Radford~M. Neal.
\newblock \emph{Bayesian Learning for Neural Networks}.
\newblock Springer-Verlag, 1996.

\bibitem[Noether(1918)]{Noether1918}
E.~Noether.
\newblock Invariante variationsprobleme.
\newblock \emph{Nachrichten von der Gesellschaft der Wissenschaften zu
  Göttingen, Mathematisch-Physikalische Klasse}, 1918.

\bibitem[Nordenfors and
  Flinth(2024)]{nordenfors2024ensemblesprovablylearnequivariance}
Oskar Nordenfors and Axel Flinth.
\newblock Ensembles provably learn equivariance through data augmentation.
\newblock \emph{arXiv preprint}, 2024.

\bibitem[Novak et~al.(2020)Novak, Xiao, Hron, Lee, Alemi, Sohl-Dickstein, and
  Schoenholz]{Novak2020Neural}
Roman Novak, Lechao Xiao, Jiri Hron, Jaehoon Lee, Alexander~A. Alemi, Jascha
  Sohl-Dickstein, and Samuel~S. Schoenholz.
\newblock Neural tangents: {Fast} and easy infinite neural networks in python.
\newblock In \emph{International Conference on Learning Representations
  (ICLR)}, 2020.

\bibitem[Ollikka et~al.(2025)Ollikka, Abbas, Perin, Kilpel{\"a}inen, and
  Deny]{ollikka2025humans}
Netta Ollikka, Amro Kamal~Mohamed Abbas, Andrea Perin, Markku Kilpel{\"a}inen,
  and Stephane Deny.
\newblock A comparison between humans and {AI} at recognizing objects in
  unusual poses.
\newblock \emph{Transactions on Machine Learning Research (TMLR)}, 2025.

\bibitem[Ortiz-Jimenez et~al.(2020)Ortiz-Jimenez, Modas, Moosavi, and
  Frossard]{ortiz2020neural}
Guillermo Ortiz-Jimenez, Apostolos Modas, Seyed-Mohsen Moosavi, and Pascal
  Frossard.
\newblock Neural anisotropy directions.
\newblock In \emph{Advances in Neural Information Processing Systems
  (NeurIPS)}, 2020.

\bibitem[Pegado et~al.(2011)Pegado, Nakamura, Cohen, and Dehaene]{pegado2011}
Felipe Pegado, Kimihiro Nakamura, Laurent Cohen, and Stanislas Dehaene.
\newblock Breaking the symmetry: {Mirror} discrimination for single letters but
  not for pictures in the visual word form area.
\newblock \emph{NeuroImage}, 2011.

\bibitem[Perea et~al.(2011)Perea, Moret-Tatay, and Panadero]{perea2011}
Manuel Perea, Carmen Moret-Tatay, and Victoria Panadero.
\newblock Suppression of mirror generalization for reversible letters:
  {Evidence} from masked priming.
\newblock \emph{Journal of Memory and Language}, 2011.

\bibitem[P{\'e}rez~Rey et~al.(2023)P{\'e}rez~Rey, Marchetti, Kragic, Jarnikov,
  and Holenderski]{rey2023equivariantrepresentationlearningpresence}
Luis~Armando P{\'e}rez~Rey, Giovanni~Luca Marchetti, Danica Kragic, Dmitri
  Jarnikov, and Mike Holenderski.
\newblock Equivariant representation learning in the presence of stabilizers.
\newblock In \emph{Machine Learning and Knowledge Discovery in Databases:
  Research Track}, 2023.

\bibitem[Pfau et~al.(2020)Pfau, Higgins, Botev, and
  Racani\`{e}re]{pfauNEURIPS2020}
David Pfau, Irina Higgins, Alex Botev, and S\'{e}bastien Racani\`{e}re.
\newblock Disentangling by subspace diffusion.
\newblock In \emph{Advances in Neural Information Processing Systems
  (NeurIPS)}, 2020.

\bibitem[Power et~al.(2022)Power, Burda, Edwards, Babuschkin, and
  Misra]{power2022grokking}
Alethea Power, Yuri Burda, Harri Edwards, Igor Babuschkin, and Vedant Misra.
\newblock Grokking: {Generalization} beyond overfitting on small algorithmic
  datasets.
\newblock \emph{arXiv preprint}, 2022.

\bibitem[Sabour et~al.(2017)Sabour, Frosst, and Hinton]{sabour2017dynamic}
Sara Sabour, Nicholas Frosst, and Geoffrey~E Hinton.
\newblock Dynamic routing between capsules.
\newblock In \emph{Advances in Neural Information Processing Systems
  (NeurIPS)}, 2017.

\bibitem[Saglietti et~al.(2022)Saglietti, Mannelli, and
  Saxe]{saglietti2022analytical}
Luca Saglietti, Stefano Mannelli, and Andrew Saxe.
\newblock An analytical theory of curriculum learning in teacher-student
  networks.
\newblock In \emph{Advances in Neural Information Processing Systems
  (NeurIPS)}, 2022.

\bibitem[Sanborn et~al.(2023)Sanborn, Shewmake, Olshausen, and
  Hillar]{sanborn2023bispectral}
Sophia Sanborn, Christian~A Shewmake, Bruno Olshausen, and Christopher~J.
  Hillar.
\newblock Bispectral neural networks.
\newblock In \emph{International Conference on Learning Representations
  (ICLR)}, 2023.

\bibitem[Saxe et~al.(2019)Saxe, McClelland, and Ganguli]{saxe2019mathematical}
Andrew~M. Saxe, James~L. McClelland, and Surya Ganguli.
\newblock A mathematical theory of semantic development in deep neural
  networks.
\newblock \emph{Proceedings of the National Academy of Sciences (PNAS)}, 2019.

\bibitem[Schott et~al.(2022)Schott, K{\"u}gelgen, Tr{\"a}uble, Gehler, Russell,
  Bethge, Sch{\"o}lkopf, Locatello, and Brendel]{schott2022visual}
Lukas Schott, Julius~Von K{\"u}gelgen, Frederik Tr{\"a}uble, Peter~Vincent
  Gehler, Chris Russell, Matthias Bethge, Bernhard Sch{\"o}lkopf, Francesco
  Locatello, and Wieland Brendel.
\newblock Visual representation learning does not generalize strongly within
  the same domain.
\newblock In \emph{International Conference on Learning Representations
  (ICLR)}, 2022.

\bibitem[Schuhmann et~al.(2022)Schuhmann, Beaumont, Vencu, Gordon, Wightman,
  Cherti, Coombes, Katta, Mullis, Wortsman, Schramowski, Kundurthy, Crowson,
  Schmidt, Kaczmarczyk, and Jitsev]{schuhmann2022laion}
Christoph Schuhmann, Romain Beaumont, Richard Vencu, Cade Gordon, Ross
  Wightman, Mehdi Cherti, Theo Coombes, Aarush Katta, Clayton Mullis, Mitchell
  Wortsman, Patrick Schramowski, Srivatsa Kundurthy, Katherine Crowson, Ludwig
  Schmidt, Robert Kaczmarczyk, and Jenia Jitsev.
\newblock Laion-5b: {An} open large-scale dataset for training next generation
  image-text models.
\newblock In \emph{Advances in Neural Information Processing Systems
  (NeurIPS)}, 2022.

\bibitem[Shah et~al.(2020)Shah, Tamuly, Raghunathan, Jain, and
  Netrapalli]{shah2020pitfalls}
Harshay Shah, Kaustav Tamuly, Aditi Raghunathan, Prateek Jain, and Praneeth
  Netrapalli.
\newblock The pitfalls of simplicity bias in neural networks.
\newblock In \emph{Advances in Neural Information Processing Systems
  (NeurIPS)}, 2020.

\bibitem[Shepard and Metzler(1971)]{Shepard1971}
Roger~N. Shepard and Jacqueline Metzler.
\newblock Mental rotation of three-dimensional objects.
\newblock \emph{Science}, 1971.

\bibitem[Siddiqui et~al.(2023)Siddiqui, Krueger, and
  Breuel]{siddiqui2023investigatingnature3dgeneralization}
Shoaib~Ahmed Siddiqui, David Krueger, and Thomas Breuel.
\newblock Investigating the nature of {3D} generalization in deep neural
  networks.
\newblock \emph{arXiv preprint}, 2023.

\bibitem[Simsek et~al.(2021)Simsek, Ged, Jacot, Spadaro, Hongler, Gerstner, and
  Brea]{pmlr-v139-simsek21a}
Berfin Simsek, Fran{\c{c}}ois Ged, Arthur Jacot, Francesco Spadaro, Clement
  Hongler, Wulfram Gerstner, and Johanni Brea.
\newblock Geometry of the loss landscape in overparameterized neural networks:
  {Symmetries} and invariances.
\newblock In \emph{International Conference on Machine Learning (ICML)}, 2021.

\bibitem[Singh et~al.(2022)Singh, Gustafson, Adcock, De~Freitas~Reis, Gedik,
  Kosaraju, Mahajan, Girshick, Dollar, and Van Der~Maaten]{singh2022revisiting}
Mannat Singh, Laura Gustafson, Aaron Adcock, Vinicius De~Freitas~Reis, Bugra
  Gedik, Raj~Prateek Kosaraju, Dhruv Mahajan, Ross Girshick, Piotr Dollar, and
  Laurens Van Der~Maaten.
\newblock Revisiting weakly supervised pre-training of visual perception
  models.
\newblock In \emph{Conference on Computer Vision and Pattern Recognition
  (CVPR)}, 2022.

\bibitem[Sohl-Dickstein et~al.(2017)Sohl-Dickstein, Wang, and
  Olshausen]{sohl2010unsupervised}
Jascha Sohl-Dickstein, Ching~Ming Wang, and Bruno~A. Olshausen.
\newblock An unsupervised algorithm for learning {Lie} group transformations.
\newblock \emph{arXiv preprint}, 2017.

\bibitem[Sorscher et~al.(2022)Sorscher, Ganguli, and
  Sompolinsky]{sorscher2022neural}
Ben Sorscher, Surya Ganguli, and Haim Sompolinsky.
\newblock Neural representational geometry underlies few-shot concept learning.
\newblock \emph{Proceedings of the National Academy of Sciences (PNAS)}, 2022.

\bibitem[Stanley and Miikkulainen(2002)]{stanley:ec02}
Kenneth~O. Stanley and Risto Miikkulainen.
\newblock Evolving neural networks through augmenting topologies.
\newblock \emph{Evolutionary Computation}, 2002.

\bibitem[Sun et~al.(2017)Sun, Shrivastava, Singh, and Gupta]{Sun_2017_ICCV}
Chen Sun, Abhinav Shrivastava, Saurabh Singh, and Abhinav Gupta.
\newblock Revisiting unreasonable effectiveness of data in deep learning era.
\newblock In \emph{International Conference on Computer Vision (ICCV)}, 2017.

\bibitem[Sundaram et~al.(2022)Sundaram, Sinha, Groth, Sasaki, and
  Boix]{sundaram2022symmetry}
Shobhita Sundaram, Darius Sinha, Matthew Groth, Tomotake Sasaki, and Xavier
  Boix.
\newblock Symmetry perception by deep networks: {Inadequacy} of feed-forward
  architectures and improvements with recurrent connections.
\newblock \emph{arXiv preprint}, 2022.

\bibitem[Tanaka and Kunin(2021)]{tanakaNIPS2021}
Hidenori Tanaka and Daniel Kunin.
\newblock Noether{\textquoteright}s learning dynamics: {Role} of symmetry
  breaking in neural networks.
\newblock In \emph{Advances in Neural Information Processing Systems
  (NeurIPS)}, 2021.

\bibitem[Vadgama et~al.(2025)Vadgama, Islam, Buracas, Shewmake, Moskalev, and
  Bekkers]{vadgama2025probingequivariancesymmetrybreaking}
Sharvaree Vadgama, Mohammad~Mohaiminul Islam, Domas Buracas, Christian
  Shewmake, Artem Moskalev, and Erik Bekkers.
\newblock Probing equivariance and symmetry breaking in convolutional networks.
\newblock \emph{arXiv preprint}, 2025.

\bibitem[Valle-Perez et~al.(2019)Valle-Perez, Camargo, and
  Louis]{valle-perez2018deep}
Guillermo Valle-Perez, Chico~Q. Camargo, and Ard~A. Louis.
\newblock Deep learning generalizes because the parameter-function map is
  biased towards simple functions.
\newblock In \emph{International Conference on Learning Representations
  (ICLR)}, 2019.

\bibitem[van~der Linden et~al.(2024)van~der Linden, Garc\'{\i}a-Castellanos,
  Vadgama, Kuipers, and
  Bekkers]{vanderlinden2024learningsymmetriesweightsharingdoubly}
Putri~A. van~der Linden, Alejandro Garc\'{\i}a-Castellanos, Sharvaree Vadgama,
  Thijs~P. Kuipers, and Erik~J. Bekkers.
\newblock Learning symmetries via weight-sharing with doubly stochastic
  tensors.
\newblock In \emph{Advances in Neural Information Processing Systems
  (NeurIPS)}, 2024.

\bibitem[van~der Ouderaa et~al.(2023)van~der Ouderaa, Immer, and van~der
  Wilk]{vanderouderaa2023learninglayerwiseequivariancesautomatically}
Tycho F.~A. van~der Ouderaa, Alexander Immer, and Mark van~der Wilk.
\newblock Learning layer-wise equivariances automatically using gradients.
\newblock In \emph{Advances in Neural Information Processing Systems
  (NeurIPS)}, 2023.

\bibitem[van~der Ouderaa et~al.(2024)van~der Ouderaa, van~der Wilk, and
  de~Haan]{vanderouderaa2024noethersrazorlearningconserved}
Tycho F.~A. van~der Ouderaa, Mark van~der Wilk, and Pim de~Haan.
\newblock Noether's razor: {Learning} conserved quantities.
\newblock In \emph{Advances in Neural Information Processing Systems
  (NeurIPS)}, 2024.

\bibitem[van~der Ouderaa and van~der
  Wilk(2022)]{vanderouderaa2022learninginvariantweightsneural}
Tycho~F.A. van~der Ouderaa and Mark van~der Wilk.
\newblock Learning invariant weights in neural networks.
\newblock In \emph{Conference on Uncertainty in Artificial Intelligence}, 2022.

\bibitem[van~der Wilk et~al.(2018)van~der Wilk, Bauer, John, and
  Hensman]{vanderwilk2018learninginvariancesusingmarginal}
Mark van~der Wilk, Matthias Bauer, ST~John, and James Hensman.
\newblock Learning invariances using the marginal likelihood.
\newblock In \emph{Advances in Neural Information Processing Systems
  (NeurIPS)}, 2018.

\bibitem[Wang et~al.(2022)Wang, Walters, and Yu]{wang2022approximately}
Rui Wang, Robin Walters, and Rose Yu.
\newblock Approximately equivariant networks for imperfectly symmetric
  dynamics.
\newblock In \emph{International Conference on Machine Learning (ICML)}, 2022.

\bibitem[Weiler et~al.(2021)Weiler, Forré, Verlinde, and
  Welling]{weiler2021coordinateindependentconvolutionalnetworks}
Maurice Weiler, Patrick Forré, Erik Verlinde, and Max Welling.
\newblock Coordinate independent convolutional networks -- isometry and gauge
  equivariant convolutions on riemannian manifolds.
\newblock \emph{arXiv preprint}, 2021.

\bibitem[Wiedemer et~al.(2023)Wiedemer, Mayilvahanan, Bethge, and
  Brendel]{wiedemer2024compositional}
Thadd\"{a}us Wiedemer, Prasanna Mayilvahanan, Matthias Bethge, and Wieland
  Brendel.
\newblock Compositional generalization from first principles.
\newblock In \emph{Advances in Neural Information Processing Systems
  (NeurIPS)}, 2023.

\bibitem[Wood(2013)]{wood2013}
Justin~N. Wood.
\newblock Newborn chickens generate invariant object representations at the
  onset of visual object experience.
\newblock \emph{Proceedings of the National Academy of Sciences (PNAS)}, 2013.

\bibitem[Worrall et~al.(2017)Worrall, Garbin, Turmukhambetov, and
  Brostow]{Worrall2017}
Daniel~E. Worrall, Stephan~J. Garbin, Daniyar Turmukhambetov, and Gabriel~J.
  Brostow.
\newblock Harmonic networks: {Deep} translation and rotation equivariance.
\newblock In \emph{Conference on Computer Vision and Pattern Recognition
  (CVPR)}, 2017.

\bibitem[Yang and Hu(2021)]{yang2020feature}
Greg Yang and Edward~J. Hu.
\newblock Tensor programs {IV}: {Feature} learning in infinite-width neural
  networks.
\newblock In \emph{International Conference on Machine Learning (ICML)}, 2021.

\bibitem[Yang et~al.(2024{\natexlab{a}})Yang, Yu, Zhu, and
  Hayou]{yang2023tensorprogramsvifeature}
Greg Yang, Dingli Yu, Chen Zhu, and Soufiane Hayou.
\newblock Tensor programs {VI}: {Feature} learning in infinite depth neural
  networks.
\newblock In \emph{International Conference on Learning Representations
  (ICLR)}, 2024{\natexlab{a}}.

\bibitem[Yang et~al.(2024{\natexlab{b}})Yang, Dehmamy, Walters, and
  Yu]{yang2024latentspacesymmetrydiscovery}
Jianke Yang, Nima Dehmamy, Robin Walters, and Rose Yu.
\newblock Latent space symmetry discovery.
\newblock In \emph{International Conference on Machine Learning (ICML)},
  2024{\natexlab{b}}.

\bibitem[Yang et~al.(2025)Yang, Park, Lubana, Okawa, Hu, and
  Tanaka]{yang2025swingbydynamicsconceptlearning}
Yongyi Yang, Core~Francisco Park, Ekdeep~Singh Lubana, Maya Okawa, Wei Hu, and
  Hidenori Tanaka.
\newblock Dynamics of concept learning and compositional generalization.
\newblock In \emph{International Conference on Learning Representations
  (ICLR)}, 2025.

\bibitem[Yeh et~al.(2022)Yeh, Hu, Hasegawa-Johnson, and Schwing]{yehPMLR2022}
Raymond~A. Yeh, Yuan-Ting Hu, Mark Hasegawa-Johnson, and Alexander Schwing.
\newblock Equivariance discovery by learned parameter-sharing.
\newblock In \emph{International Conference on Artificial Intelligence and
  Statistics (AISTATS)}, 2022.

\bibitem[Zbontar et~al.(2021)Zbontar, Jing, Misra, LeCun, and
  Deny]{zbontar2021barlow}
Jure Zbontar, Li~Jing, Ishan Misra, Yann LeCun, and Stephane Deny.
\newblock Barlow {Twins}: {Self}-supervised learning via redundancy reduction.
\newblock In \emph{International Conference on Machine Learning (ICML)}, 2021.

\bibitem[Zemel and Hinton(1990)]{zemelNIPS1990}
Richard Zemel and Geoffrey~E Hinton.
\newblock Discovering viewpoint-invariant relationships that characterize
  objects.
\newblock In \emph{Advances in Neural Information Processing Systems
  (NeurIPS)}, 1990.

\bibitem[Zhang et~al.(2021)Zhang, Bengio, Hardt, Recht, and
  Vinyals]{zhang2021understanding}
Chiyuan Zhang, Samy Bengio, Moritz Hardt, Benjamin Recht, and Oriol Vinyals.
\newblock Understanding deep learning (still) requires rethinking
  generalization.
\newblock \emph{Commun. ACM}, 2021.

\bibitem[Zhou et~al.(2021)Zhou, Knowles, and Finn]{zhou2021metalearning}
Allan Zhou, Tom Knowles, and Chelsea Finn.
\newblock Meta-learning symmetries by reparameterization.
\newblock In \emph{International Conference on Learning Representations
  (ICLR)}, 2021.

\bibitem[Ziyin et~al.(2020)Ziyin, Hartwig, and Ueda]{ziyin2020neural}
Liu Ziyin, Tilman Hartwig, and Masahito Ueda.
\newblock Neural networks fail to learn periodic functions and how to fix it.
\newblock In \emph{Advances in Neural Information Processing Systems
  (NeurIPS)}, 2020.

\end{thebibliography}

\clearpage

\appendix

{\centering \LARGE \bfseries Appendix\par}

\section*{Contents}
\addcontentsline{toc}{section}{Appendix Contents}
\startcontents
\printcontents{a}{1}{\setcounter{tocdepth}{2}}\vspace{0.5cm}

All code is available at \href{https://github.com/Andrea-Perin/gpsymm}{\texttt{https://github.com/Andrea-Perin/gpsymm}.}

\section{A (Very Brief) Introduction to Neural Kernel Theories}\label{app:intro_ntk}

Here we briefly introduce the mathematical connection that exists between deep neural networks and kernel methods. We reuse and modify the words of Jan Gerken (\url{https://gapindnns.github.io/_pages/wide_nns.html}), with permission:
\begin{quote}
\itshape
The number of neurons in each layer of a neural network is called the \emph{width} $N$ of that layer. When considering the limit in which the width of all hidden layers goes to infinity (assuming that the weights of the network scale as $1/\sqrt{N}$ to avoid activity divergence), the neural network simplifies dramatically. By an argument using the central limit theorem, one can show that in the infinite width limit, the neurons follow a collective Gaussian distribution known as a Gaussian process. Intuitively, the fluctuations from all the neurons cancel out. This effect was known for a long time as the \emph{Neural Network Gaussian Process (NNGP)} theory \citep{Neal1996} and characterizes the neural network at initialization, i.e. before training has begun. Furthermore, the neural kernel---characterizing dot product similarity of data points in the last infinitely-wide layer of the network---becomes independent of the specific initialization chosen and can be computed using recursive layer-by-layer expressions.\\\\
In 2018, a seminal paper by \citet{jacotNIPS2018} showed that the simplifications go even further than that: In the infinite width limit (again assuming that the weights scale as $1/\sqrt{N}$), the dynamics remain manageable even during training. Specifically, they proved that the kernel of an infinitely wide network remains constant throughout training. This kernel, called the \emph{Neural Tangent Kernel (NTK)}, is equal to the NNGP kernel with additional corrective terms. Briefly, this simplification is possible because all weights in the inner layers remain in the vicinity of their initial values during training, allowing to linearize the network around the initial condition. Under the simple but realistic training paradigm of gradient descent of the mean-squared-error loss, the training dynamics can then be solved analytically in closed form and the prediction of the trained network on any input be computed. The output of the trained network is again a Gaussian process. \\ \\These simplifications in the infinite width limit give powerful insights into the behavior of neural networks.  
\end{quote}

\section{Derivation of the Spectral Error}\label{app:spectral_error}

In this section, we report the computations that give the spectral error in Eq.~\ref{eq:spectral_error}.
We start by describing general Gaussian process (GP) regression, and then move to the actual derivation of our formula, that applies specifically to a group-cyclic dataset.

We consider a dataset $\mathcal{D} = \{ (x_i, y_i)\}_{i=1}^N \subset \R^d\times \R$, so that the samples $x_i$ are $d$ dimensional vectors and labels $y_i$ are scalars.

GP regression is a Bayesian method which, starting from the user's prior information (namely, a mean function $m: \mathcal{X} \to\R$, and a covariance function $k: \mathcal{X}\times \mathcal{X} \to \R$), produces a distribution of regression functions of the type $f: \mathcal{X}\to \R$.
We are interested in the value of such a function at a given test point $x_t \in \mathcal{X}$, conditioned on the values $(x_i, y_i)$ contained in the dataset.
To obtain these pointwise results, one can focus on a finite collection of points $(x_i, y_i) \in\mathcal{X}\times \R$.
Any such finite collection has the property (inherited by the GP) of being jointly Gaussian-distributed as follows:
\begin{align*}
    (y_1, \dots, y_N) \sim \mathcal{N}(\vec{m}, K),
\end{align*}
where $\vec{m}_i=m(x_i)\in\R$ is the mean vector, obtained by evaluating the mean function $m$ over the given points, and $K_{ij} = k(x_i, x_j) \in \R$ is the covariance matrix, obtained by evaluating the covariance function $k$ over all pairs of points $x_i, x_j$ that can be formed in the collection.
We set $m(x)=0$ everywhere.

Equipped with these tools, we can formulate regression on an additional test point $x_t$ as a \textit{conditioning} over a $N+1$ dimensional multivariate Gaussian distribution (MVG), in order to obtain a probability distribution for the value of $y_t$.
In other words, we can consider the set of $N+1$ samples ($N$ from the training set, and the additional one being the test value) to be distributed according to a $N+1$ dimensional MVG:
\begin{align*}
    (y_1, \dots, y_N, y_t) \sim \mathcal{N}\left(0, 
    \begin{bmatrix}
        K & k(\vec{x}, x_t)\\
        k(\vec{x}, x_t)^T & k(x_t, x_t)
    \end{bmatrix}\right), 
\end{align*}
where we use the notation $k(\vec{x}, x_t) \in \R^N$ to denote the column vector that, at $i$-th entry, contains $k(x_i, x_t)$ with $x_i$ being the $i$-th sample in the dataset.
The conditional mean and variance for $y_t$ are described by the formulas:
\begin{align*}
    \mu_{t|\mathcal{D}} &= (K^{-1}k(\vec{x}, x_t))^T \vec{y},\\
    \sigma_{t|\mathcal{D}} &= k(x_t, x_t) - (K^{-1}k(\vec{x}, x_t))^Tk(\vec{x}, x_t),
\end{align*}
where we define $\vec{y}$ to be the vector of the training labels.
The expectation of this distribution (i.e., $\mu_{t|\mathcal{D}}$) will then serve as the result of the Gaussian process regression.
This procedure can be extended for $p$ test points, with all the opportune dimensionality changes (i.e., the result of the conditioning will be a $p$ dimensional MVG).

In Section~\ref{sec:gaussian_toy}, we take the Gaussian (RBF) kernel as the kernel for Gaussian Process regression.
In later sections, we consider various deep network architectures in the infinite width limit.
These also behave like Gaussian Processes, with a kernel that is specified by their architecture \citep{Neal1996, jacotNIPS2018}.

Our proof for the spectral formula of the generalization error does not depend on the specific kernel used, as long as the kernel (Gram) matrix is \emph{circulant over a group of points generated by the action of a cyclic group}.
We demonstrate elsewhere that this is the case both for the Gaussian kernel, as well as all the deep neural kernels that we analyze in this paper (MLP and ConvNets).

We are now ready to present the derivation of Eq.~\ref{eq:spectral_error}. This equation describes the \emph{most likely} error in prediction resulting from Gaussian Process regression on a dataset generated by cyclic group action. We refer to this \emph{most likely} error as \emph{the result of kernel regression} in main text. We do not study the fluctuations of errors around this mean.

We start by presenting a depiction of the conditioning procedure for the case in which the training set consists of a single point (Fig.~\ref{fig:ellipse}).
In the following, we denote labels as $y_0$ and $y_1$, and their respective values as $\mu_0$ and $\mu_1$.
In practice, the $y$s are the training set labels, and the $\mu$s are their values. 
In this case, the training label is $y_0$, while the test label is $y_1$.
The two labels are jointly distributed according to a 2 dimensional Gaussian, with mean 0 and covariance matrix $K$.
This matrix depends on the values of $x_0, x_1$, and the kernel function $k$.
We represent the distribution as an ellipse, which is understood as a level set of the covariance matrix $K$, centered at the mean, 0, in a way similar to the usual representation of confidence levels of a multivariate Gaussian distribution.
Conditioning over the known training label $y_0$ can be interpreted geometrically as slicing the ellipse with a vertical line at coordinate $\mu_0$.
We thus obtain a segment (i.e., a 1-dimensional ellipse), the midpoint of which is the point $(\mu_0, \mu_{1|0})$, where $\mu_{1|0}$ denotes the conditional value of $y_1$ given $x_0, x_1$ and the value of $y_0$.
As a consequence, the result of GP regression can thus be interpreted geometrically as finding the $y_1$ coordinate of the center of the slice.
Doing this will reward us with Eq.~\ref{eq:spectral_error}.

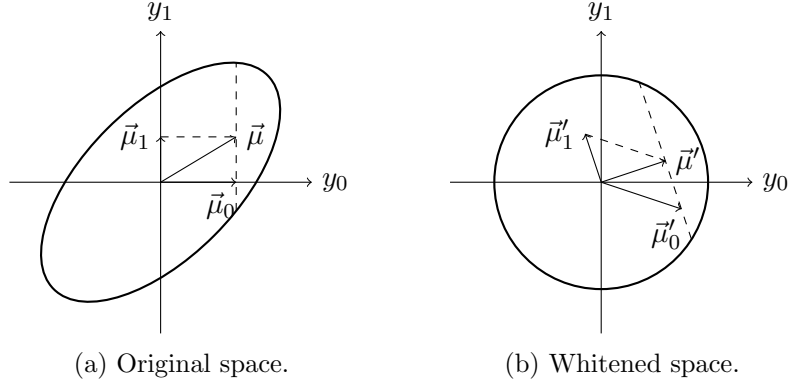
\begin{figure}[H]
    \centering
\begin{subfigure}{0.3\textwidth}
        \centering
\begin{tikzpicture}
\def\A{2} \def\B{1} \def\Angle{45} 

\draw[rotate around={\Angle:(0,0)}, thick] (0,0) ellipse (\A cm and \B cm);

\def\muOne{1}

\pgfmathsetmacro{\AA}{cos(\Angle)*cos(\Angle)/(\A*\A) + sin(\Angle)*sin(\Angle)/(\B*\B)}
    \pgfmathsetmacro{\BB}{sin(\Angle)*sin(\Angle)/(\A*\A) + cos(\Angle)*cos(\Angle)/(\B*\B)}
    \pgfmathsetmacro{\CC}{-2*cos(\Angle)*sin(\Angle)/(\B*\B) + 2*cos(\Angle)*sin(\Angle)/(\A*\A)}

    \pgfmathsetmacro{\discriminant}{\CC*\CC*\muOne*\muOne - 4*\BB*(\AA*\muOne*\muOne - 1)}
    \pgfmathsetmacro{\yplus}{(-\CC*\muOne + sqrt(\discriminant))/(2*\BB)}
    \pgfmathsetmacro{\yminus}{(-\CC*\muOne - sqrt(\discriminant))/(2*\BB)}

\draw[dashed] (\muOne, \yminus) -- (\muOne, \yplus);

\pgfmathsetmacro{\muTwo}{(\yplus + \yminus) / 2}

\draw[dashed] (0, \muTwo) -- (\muOne, \muTwo);

\draw[->] (-\B-1,0) -- (\B+1,0) node[right] {$y_0$};
    \draw[->] (0,-\B-1) -- (0,\B+1) node[above] {$y_1$};

\draw[->] (0,0) -- (\muOne,0) node[left,below,shift={(-0.2,0)}] {$\vec{\mu}_0$};
    \draw[->] (0,0) -- (0,\muTwo) node[left] {$\vec{\mu}_1$};
    \draw[->] (0,0) -- (\muOne,\muTwo) node[right] {$\vec{\mu}$};
    
\end{tikzpicture}         \caption{Original space.}
        \label{fig:ellipse}
    \end{subfigure}
    \hspace{1cm}
\begin{subfigure}{0.3\textwidth}
        \centering
\begin{tikzpicture}
\def\A{2} \def\B{1} \def\Angle{45} \def\radius{1.414}

\def\muOne{1}
\pgfmathsetmacro{\AA}{cos(\Angle)*cos(\Angle)/(\A*\A) + sin(\Angle)*sin(\Angle)/(\B*\B)}
    \pgfmathsetmacro{\BB}{sin(\Angle)*sin(\Angle)/(\A*\A) + cos(\Angle)*cos(\Angle)/(\B*\B)}
    \pgfmathsetmacro{\CC}{-2*cos(\Angle)*sin(\Angle)/(\B*\B) + 2*cos(\Angle)*sin(\Angle)/(\A*\A)}
    \pgfmathsetmacro{\discriminant}{\CC*\CC*\muOne*\muOne - 4*\BB*(\AA*\muOne*\muOne - 1)}
    \pgfmathsetmacro{\yplus}{(-\CC*\muOne + sqrt(\discriminant))/(2*\BB)}
    \pgfmathsetmacro{\yminus}{(-\CC*\muOne - sqrt(\discriminant))/(2*\BB)}
    \pgfmathsetmacro{\muTwo}{(\yplus + \yminus) / 2}

\pgfmathsetmacro{\a}{sin(\Angle)*sin(\Angle)/\B + cos(\Angle)*cos(\Angle)/\A}
    \pgfmathsetmacro{\b}{sin(\Angle)*cos(\Angle) * ((\B-\A)/(\A*\B))}
    \pgfmathsetmacro{\c}{sin(\Angle)*cos(\Angle) * ((\B-\A)/(\A*\B))}
    \pgfmathsetmacro{\d}{cos(\Angle)*cos(\Angle)/\B + sin(\Angle)*sin(\Angle)/\A}

\pgfmathsetmacro{\muOnePx}{\a * \muOne}
    \pgfmathsetmacro{\muOnePy}{\c * \muOne}
    \pgfmathsetmacro{\muTwoPx}{\b * \muTwo}
    \pgfmathsetmacro{\muTwoPy}{\d * \muTwo}
    \pgfmathsetmacro{\yminusx}{\a * \muOne + \b*\yminus}
    \pgfmathsetmacro{\yminusy}{\c * \muOne + \d*\yminus}
    \pgfmathsetmacro{\yplusx}{\a * \muOne + \b*\yplus}
    \pgfmathsetmacro{\yplusy}{\c * \muOne + \d*\yplus}
    \pgfmathsetmacro{\mupx}{\muOnePx + \muTwoPx}
    \pgfmathsetmacro{\mupy}{\muOnePy + \muTwoPy}

\draw[->] (0, 0) -- (\radius*\muOnePx, \radius*\muOnePy) node[below,shift={(-0.2,0)}] {$\vec{\mu}'_0$};
    \draw[->] (0, 0) -- (\radius*\muTwoPx, \radius*\muTwoPy) node[left] {$\vec{\mu}'_1$};
    \draw[->] (0, 0) -- (\radius*\mupx, \radius*\mupy) node[right] {$\vec{\mu}'$};

\draw[dashed] (\radius*\yminusx, \radius*\yminusy) -- (\radius*\yplusx, \radius*\yplusy);
    \draw[dashed] (\radius*\muTwoPx, \radius*\muTwoPy) -- (\radius*\mupx, \radius*\mupy);
\draw[thick] (0,0) circle (\radius cm);

\draw[->] (-\B-1,0) -- (\B+1,0) node[right] {$y_0$};
    \draw[->] (0,-\B-1) -- (0,\B+1) node[above] {$y_1$};

\end{tikzpicture}         \caption{Whitened space.}
        \label{fig:circle}
    \end{subfigure}
    \caption{2D sketch of the geometric interpretation, showing the effect of the whitening transformation.}
    \label{fig:ellipse_to_circle}
\end{figure}

We can make this geometric task easier by turning the ellipse into a circle; we can do this by performing a \textit{whitening} operation on this vector space.
Such whitening is effectively obtained by the inverse of the square root of the covariance matrix $K$.
Denoting the transformed values by priming them, our task is now to find $\vec{\mu}_1'$ given $\vec{\mu}_0'$.\footnote{We have now shifted to a vector notation for these values as, unlike in the original space, both components can be nonzero.}
It is clear that the segment produced by the conditioning turns now into a chord of the circle.
This makes the whitened vector $\Vec{\mu}'$, which is defined so that it points to the midpoint of the segment, \textit{perpendicular} to the segment itself.
This means we can write the following geometric condition:
\begin{align}
    (\Vec{\mu}_0' + \Vec{\mu}_1') \cdot \Vec{\mu}_1' = 0. \label{eq:ortho_2d}
\end{align}

The same properties hold in the case where we have $2N$ samples, as in the main text, and we condition on $2N-1$ of them.
Let us now denote the value we wish to infer as $\mu_0$; the multidimensional extension of Eq.~\ref{eq:ortho_2d} becomes
\begin{align}
    (\sum_{i=0}^{2N-1} \Vec{\mu}_i') \cdot \Vec{\mu}_0' = 0. \label{eq:ortho}
\end{align}

We now move to solve Eq.~\ref{eq:ortho}.
To do so, we make the justified assumption that $K$ is \emph{circulant} (as a consequence of the group-symmetric structure of the dataset).
As such, it is diagonalized by the Discrete Fourier Transform (DFT) basis, $E$:
\begin{align*}
    K = E^{-1} \Lambda E.
\end{align*}

Whitening is obtained by applying the transform given by:
\begin{align*}
    Z = K^{-1/2} = E \Lambda^{-1/2} E^{-1},
\end{align*}
to all the vectors $\Vec{\mu_i}$.

To make computations easier, we collect our vectors in a data matrix $M$ so that the $i$-th column is $\Vec{\mu}_i$, we can collectively transform our samples with a matrix operation.
The collection of transformed (i.e., whitened) samples is denoted as $M'$, and
\begin{align*}
    M' = ZM.
\end{align*}

We now write down a few properties of the matrices at play that will make computations quicker.
Denoting $\delta$ as the Kronecker delta, $\lambda_i$ as the $i$-th eigenvalue of $K$, and $\omega = \exp(-2\pi i/(2N))$, we have
\begin{align*}
    M_{ij} &= \mu_i \delta_{ij},\\
    E_{ij} &= \frac{1}{\sqrt{2N}} \omega^{ij},\\
    E^{-1}_{ij} &= \frac{1}{\sqrt{2N}} \omega^{-ij},\\
    \Lambda^{-1/2}_{ij} &= \frac{1}{\sqrt{\lambda_{i}}} \delta_{ij}.
\end{align*}

We can now compute the value of the $j$-th component of the $i$-th sample in whitened space, that is, $M'_{ji}$:
\begin{align*}
    M'_{ji} &= (ZM)_{ji}\\
    &= \sum_k Z_{jk}M_{ki}\\
    &= \sum_{k} (E \Lambda^{-1} E^{-1})_{jk}M_{ki}\\
    &= \sum_{k l m} E_{jl}\Lambda^{-1/2}_{lm} E^{-1}_{mk} M_{ki}\\
    &= \sum_{k l m} \frac{1}{\sqrt{2N}} \omega^{jl}\frac{1}{\sqrt{\lambda_{l}}} \delta_{lm} \frac{1}{\sqrt{2N}} \omega^{-mk} \mu_k \delta_{ki}\\
    &= \sum_l \frac{1}{2N} \omega^{jl}\frac{1}{\sqrt{\lambda_{l}}} \omega^{-li} \mu_i\\
    &= \sum_l \frac{1}{2N} \frac{\mu_i\omega^{l(j-i)}}{\sqrt{\lambda_{l}}}.
\end{align*}

We can now take Eq.~\ref{eq:ortho} and distribute the dot product over the sum.
We obtain:
\begin{align}
    \sum_{i=0}^{2N-1} \sum_j M_{ji}'M_{j0}' = 0. \label{eq:ortho_indexes}
\end{align}
In general,
\begin{align*}
    \sum_j M_{ji}'M_{j0}' 
    &= \sum_j (\sum_l \frac{1}{2N} \frac{\mu_i\omega^{l(j-i)}}{\sqrt{\lambda_l}})(\sum_k \frac{1}{2N} \frac{\mu_0\omega^{kj}}{\sqrt{\lambda_k}})\\
    &= \sum_{lkj} \frac{1}{(2N)^2} \mu_i\mu_0\frac{\omega^{l(j-i)}\omega^{kj}}{\sqrt{\lambda_l\lambda_k}}\\
    &= \sum_{lkj} \frac{1}{2N} \frac{\mu_i\mu_0\omega^{-il}}{\sqrt{\lambda_l\lambda_k}} \frac{1}{2N}\omega^{j(l+k)}.
\end{align*}
We isolate the rightmost term and use the equality
\begin{align*}
    \sum_j \frac{1}{2N}\omega^{j(l+k)} = \delta_{l,2N-k}.
\end{align*}
Thus, we get
\begin{align*}
    \sum_j M_{ji}'M_{j0}' 
    &= \sum_{kl} \frac{1}{2N} \frac{\mu_i\mu_0\omega^{-il}}{\sqrt{\lambda_l\lambda_k}} \delta_{l,2N-k}\\
    &= \sum_{k} \frac{1}{2N} \frac{\mu_i\mu_0\omega^{-i(2N-k)}}{\sqrt{\lambda_{2N-k}\lambda_k}}\\
    &= \sum_{k} \frac{1}{2N} \frac{\mu_i\mu_0\omega^{ik}}{\lambda_k}.
\end{align*}
Thus, Eq.~\ref{eq:ortho_indexes} becomes, after simplifying the common term $\mu_0$ away,
\begin{align*}
    \sum_{i=0}^{2N-1} \sum_k \frac{1}{2N} \frac{\mu_i\omega^{ik}}{\lambda_k} = 0.
\end{align*}
So far, we have only used the circularity of $K$ in our derivation.
We now make an additional assumption, this time on the values and ordering of the labels.
We require $\mu_{2i}=\mu$, and $\mu_{2i+1}=-\mu$ for $i\in [0, \dots, N-1]$.
This amounts to having our samples be interleaved and have opposing labels.
We can then write the general expression for the value of $\mu_i$, valid for all $i$:
\begin{align}
    \mu_i = \mu \omega^{iN} + (\mu_0-\mu)\delta_{i0}.\label{eq:mu_expression}
\end{align}
Note how, for $i=0$ (that is, the missing point) we recover the unknown value $\mu_0$ for which we want to solve.

With this substitution, we get
\begin{align*}
    \sum_{i=0}^{2N-1}\frac{1}{2N} (\mu \omega^{iN} + (\mu_0-\mu)\delta_{i0})\sum_k \frac{\omega^{ik}}{\lambda_k} = 0.
\end{align*}
The summation thus splits in three parts:
\begin{align*}
    \sum_{i=0}^{2N-1}\frac{\mu}{N} \sum_k \frac{\omega^{i(k+N)}}{\lambda_k} 
    + \sum_{i=0}^{2N-1}\frac{\mu_0}{2N}\sum_k \frac{\omega^{ik}}{\lambda_k}\delta_{i0}
    - \sum_{i=0}^{2N-1}\frac{\mu}{2N}\delta_{i0}\sum_k \frac{\omega^{ik}}{\lambda_k} = 0
\end{align*}
We now go over each of them.
\begin{itemize}
    \item For the first term,
    \begin{align*}
    \sum_{i=0}^{2N-1}\sum_k\frac{\mu}{2N}\frac{\omega^{i(k+N)}}{\lambda_k} &= \mu\sum_k \frac{1}{\lambda_k}\sum_{i=0}^{2N-1}\frac{1}{2N}\omega^{i(k+N)}\\
    &= \mu\sum_k \frac{1}{\lambda_k}\delta_{k,-N}\\
    &= \mu\frac{1}{\lambda_{N}}.
    \end{align*}
    \item The second term contains a Kronecker delta that selects the 0-th component in the sum over index $i$.
    Thus, we get
    \begin{align*}
        \mu_0 \sum_k \frac{1}{2N}\frac{1}{\lambda_k} = \mu_0 \langle \lambda^{-1}\rangle,
    \end{align*}
    where the angled brackets denote the average.
    \item The last term, like the previous one, contains a selection on $i$ due to the Kronecker delta.
    It becomes
    \begin{align*}
        \mu \langle \lambda^{-1}\rangle.
    \end{align*}
\end{itemize}
Putting all together, we get
\begin{align*}
    \mu\frac{1}{\lambda_{N}} + \mu_0 \langle \lambda^{-1}\rangle - \mu \langle \lambda^{-1}\rangle = 0,
\end{align*}
from which, after rearranging, we get a formula for our missing point as a function of the eigenvalues of $K$:
\begin{align}
    \mu_0 = \mu \left(1 - \frac{\lambda_{N}^{-1}}{\langle \lambda^{-1} \rangle}\right).\label{eq:mu0_theory}
\end{align}

The formula for the spectral error, then, is
\begin{align}
    \varepsilon_s \triangleq \frac{\mu_0-\mu}{\mu} \color{black}{= \frac{\lambda_N^{-1}}{\langle \lambda^{-1} \rangle}}.
    \label{eq:eps_theory}
\end{align}

\subsection{Extension to Multiple Missing Points}\label{app:spectral_error_multiple_points}
In the case of multiple missing points, we can find an equivalent of Eq.~\ref{eq:eps_theory}.
It requires solving a $p$ dimensional linear system, where $p$ is the number of missing points.

We denote by $[p]$ the set of indexes for the missing points.
Then, for the missing point with index $j\in[p]$, we have
\begin{align*}
    \sum_{m\in[p]} \mu_m \frac{1}{2N}\sum_l \frac{\omega^{l(j-m)}}{\lambda_l} = \mu \left(\frac{1}{2N}\sum_{n\in[p]} \omega^{nN} \sum_k \frac{\omega^{k(j-n)}}{\lambda_k} - \frac{\omega^{jN}}{\lambda_{N}}\right).
\end{align*}
We notice that this equation can be expressed as the following linear system:
\begin{align*}
    \sum_{m\in[p]} A_{jm}\mu_m = b_j,
\end{align*}
where
\begin{align*}
    A_{jm} = \frac{1}{2N}\sum_l \frac{\omega^{l(j-m)}}{\lambda_l}, \quad b_j=\mu \left(\frac{1}{2N}\sum_{n\in[p]} \omega^{nN} \sum_k \frac{\omega^{k(j-n)}}{\lambda_k} - \frac{\omega^{jN}}{\lambda_{N}}\right).
\end{align*}
One can check that, if we set $[p]=\{0\}$, this formula recovers the expression in Eq.~\ref{eq:eps_theory}.

\section{Extension of the Spectral Error to General Finite Groups}\label{app:nonabelian}

We now present a general formula for the spectral error which holds true for arbitrary finite groups, including non-abelian groups.\\

We start by re-deriving the error for the simple cyclic group using an algebraic formalism which will be easier to generalize for arbitrary finite groups.

\subsection{Alternative Derivation of Spectral Error on Cyclic Groups}\label{app:spectral_nice}

We re-derive here the formula for the spectral error in Eq.~\ref{eq:spectral_error}.
We start by restating the conditions under which the formula holds.

Assume we have an orbit of the cyclic group $C_{2N}$ (where $N$ is the number of points in one class), a labeling function $y: C_{2N}\to\R$ and a kernel function $k: V\times V \to \R$, where $V \cong \R^n$ is some vector space. 
Additionally, assume that $\rho: C_{2N} \to GL(V)$ is a representation of $C_{2N}$, and consider its linear action on elements of $V$. We assume $k$ to be \emph{stationary} with respect to this action.
We can then consider, for a ``seed point" $x\in V$ the $C_{2N}$ orbit $\mathcal{O}x$ generated by the action of $\rho$:
\begin{align*}
    \mathcal{O}x = \{\rho(g)x \mid g\in C_{2N}\} = \{\rho(r^i)x \mid i\in [2N]\},
\end{align*}
where $r\in C_{2N}$ is the generator of the cyclic group.
We collect the kernel function's values on pairs of points in $\mathcal{O}x$, and store them in the kernel matrix $K\in\R^{2N\times 2N}$, such that
\begin{align*}
    K_{ij} = k(x_i, x_j) = k(\rho(r^i)x, \rho(r^j)x) = k(x, \rho(r^{j-i})x), 
\end{align*}
where the last equality is allowed by the stationarity of the kernel function.

Consider now the situation in which we condition kernel regression on all but one point of the dataset. We assume the labels to follow the formula
\begin{align*}
    y_{true}(r^i) = (-1)^i, \quad \mathrm{for}\quad i\in[2N],
\end{align*}
which produces an alternating labeling scheme.
Given the symmetry of the dataset, we can choose the missing point to be the first point (indexed at 0) without loss of generality. We want to quantify the \textit{expected error} of kernel regression on the missing point.
We can express the prediction given by kernel regression using in the Fourier basis:
\begin{align*}
    y = \mathcal{F}\mu - \varepsilon_s \delta_0 = \mathcal{F}(\mu - \varepsilon_s \hat{\delta}_0),
\end{align*}
where $\mathcal{F} \in \C^{2N\times 2N}$ is the inverse DFT matrix, $\mu\in\R^{2N}$ is the vector of Fourier coefficients for the labeling function, $\hat{\delta}_0$ is the vector of Fourier coefficients for the indicator at 0, and $\varepsilon_s$ is the error on the missing point.

Furthermore, for a stationary kernel over an orbit of the group, the kernel matrix is circulant, which allows us to diagonalize it using the same inverse DFT matrix $\mathcal{F}$ (and its transpose):
\begin{align*}
    K = \mathcal{F} \Lambda \mathcal{F}^T.
\end{align*}
Then, as per usual Gaussian Process (GP) regression, we assume a jointly gaussian distribution for the predictions with covariance given by $K$:
\begin{align*}
    p(y) \propto \exp\left(-y^TK^{-1}y\right).
\end{align*}
Substituting the Fourier expressions,
\begin{align*}
    y^TK^{-1}y &= (\mathcal{F}(\mu - \varepsilon_s \hat{\delta}_0))^T \mathcal{F} \Lambda^{-1} \mathcal{F}^T (\mathcal{F}(\mu - \varepsilon_s\hat{\delta}_0)) \\
    &= \mu^T \Lambda^{-1} \mu - \varepsilon_s \mu^T \Lambda^{-1} \hat{\delta}_0 -  \varepsilon_s \hat{\delta}_0^T \Lambda^{-1} \mu + \varepsilon_s^2 \hat{\delta}_0^T \Lambda^{-1}\hat{\delta}_0.
\end{align*}
The maximum likelihood estimate for $\varepsilon_s$ is obtained by differentiating this expression with respect to it, and setting the result to 0.
We obtain
\begin{align*}
    \frac{\partial p(y)}{\partial \varepsilon_s} = 0 \implies
    \mu^T\Lambda^{-1}\hat{\delta}_0 - \varepsilon_s \hat{\delta}_0^T\Lambda^{-1}\hat{\delta}_0 = 0 \implies \varepsilon_s = \frac{\mu^T\Lambda^{-1}\hat{\delta}_0}{\hat{\delta}_0^T\Lambda^{-1}\hat{\delta}_0}.
\end{align*}
Now, we can compute $\hat{\delta}_0$ (note that we use the orthogonal scaling for the DFT):
\begin{align*}
    [\hat{\delta}_0]_i = [\mathcal{F}^T\delta_0]_i = \frac{1}{\sqrt{2N}} \sum_{j=1}^{2N} \omega^{-ij}\delta_{j0} = \frac{1}{\sqrt{2N}}.
\end{align*}
The denominator becomes
\begin{align*}
    \hat{\delta}^T_0 \Lambda^{-1}\hat{\delta}_0 = \sum_{ij} \frac{\delta_{ij}}{2N\lambda_i} =  \langle \lambda^{-1} \rangle.
\end{align*}
As for the numerator, since the labeling function is $(y_{true})_i=(-1)^i$, we have
\begin{align*}
    \mu_i = [\mathcal{F}^T y_{true}]_i = \frac{1}{\sqrt{2N}}\sum_{j=1}^{2N} \omega^{-ij}(-1)^j = 
    \frac{1}{\sqrt{2N}}\sum_{j=1}^{2N} \omega^{-ij}\omega^{jN} = \sqrt{2N}\delta_{iN}.
\end{align*}
Thus, the numerator becomes
\begin{align*}
    \mu^T\Lambda^{-1}\hat{\delta}_0 = \sum_{i,j=1}^{2N} \sqrt{2N}\delta_{iN} \frac{\delta_{ij}}{\lambda_i} \frac{1}{\sqrt{2N}} = \lambda_N^{-1}.
\end{align*}
Thus, the formula for the spectral error becomes
\begin{align*}
    \varepsilon_s = \frac{\lambda_N^{-1}}{\langle \lambda^{-1}\rangle}.
\end{align*}

\subsection{Extension to Arbitrary Groups}\label{app:arbitrary}

The ideas behind the derivation are the same as in App.~\ref{app:spectral_nice}, but require the introduction of additional group-theoretic concepts.
For a good practical introduction to those topics in representation theory and noncommutative harmonic analysis, we refer to \cite{Chirikjian2021}.

Let us consider a finite group $G$ of order $|G|$, and representation $\rho$ on a $d$-dimensional vector space $V$.
For a point $x\in V$, consider the orbit generated by the linear action of the representation $\rho$:
\begin{align*}
    \mathcal{O}_x = \{\rho(g)x \mid g\in G\}.
\end{align*}
Let us also consider a kernel function $k: V \times V \to \R$ that is \textit{stationary} with respect to the action of $G$, meaning that, for $x, x'\in V$ and for any $g\in G$,
\begin{align*}
    k(x, x') = k(g.x, g.x'),
\end{align*}
where the dot notation stands for the action of the group $G$ onto elements of $V$.
We compute the associated kernel matrix $K\in \R^{|G| \times |G|}$ on the $G$-orbit we defined above as follows:
\begin{align*}
    K_{ij} = k(x_i, x_j) = k(\rho(g_i)x, \rho(g_j)x).
\end{align*}
Kernel regression produces predictions on input data points. These predictions are given by the function $y: G \to \C$. 
Under the assumptions of kernel regression, the collection of the $|G|$ predicted values will be jointly Gaussian distributed:
\begin{align}
    p(y) \propto \exp(-\langle y, K^{-1}y\rangle)\label{eq:jointprob},
\end{align}
where the angled brackets denote an inner product in $\C^{|G|}$.

In the following, we are interested in computing the prediction error of a linear kernel regressor on a single missing point in the orbit at $m\in G$.
The predictions, conditioned on all but the missing point $m\in G$, can be written:
\begin{align*}
    y(g) = y_{true}(g) + \varepsilon \delta_m(g),
\end{align*}
where we define the $G$-Kronecker delta function $\delta_h: G\to\C$ as follows:
\begin{align*}
    \delta_h(g) = \begin{cases}
			1 & \text{if $g=h$,}\\
            0 & \text{otherwise.}
		 \end{cases}
\end{align*}
We now express both terms of $y(g)$ in the \emph{generalized Fourier basis}. 
We start with the true labels $y_{true}$:
\begin{align*}
    y_{true}(g) &= \sum_{\rho \in \hat{G}} \sqrt{\frac{d_\rho}{|G|}} \mathrm{Tr}[\hat{y}_{true}(\rho)\rho(g)],
\end{align*}
where $\hat{G}$ denotes the set of equivalent \textit{unitary} irreps of $G$, $d_\rho$ is the dimension of irrep $\rho$, and the matrix functions $\hat{y}_{true}(\rho)$ are the generalized Fourier coefficients for the labeling function $y_{true}$.\footnote{Note that we use the \textit{orthogonal convention} for the normalization constants.}
We then move to the generalized Fourier expression for the $G$-Kronecker delta.
We start from the inverse Fourier transform formula,
\begin{align*}
    \delta_h(g) = \sum_{\rho \in \hat{G}} \sqrt{\frac{d_\rho}{|G|}} \mathrm{Tr}[\hat{\delta}_h(\rho) \rho(g)].
\end{align*}
We compute the Fourier coefficients $\hat{\delta}_h(\rho)$ by projecting the function onto the basis given by the irreps:
\begin{align*}
    \hat{\delta}_h(\rho) = \sum_{g\in G}\sqrt{\frac{d_\rho}{|G|}} \delta_h(g)\rho(g^{-1}) = \sqrt{\frac{d_\rho}{|G|}}\rho(h^{-1}).
\end{align*}
Then,
\begin{align*}
    \delta_h(g) = \sum_{\rho \in \hat{G}} \sqrt{\frac{d_\rho}{|G|}} \mathrm{Tr}\left[\sqrt{\frac{d_\rho}{|G|}}\rho(h^{-1}) \rho(g)\right] = \frac{1}{|G|}\sum_{\rho\in\hat{G}} d_\rho \mathrm{Tr}[\rho(h^{-1}g)].
\end{align*}
Analogously to the cyclic case, the Kronecker delta can be seen as a sum of all irreps in the group.
We thus obtain the following expression for $y$:
\begin{align}
    y(g) &= y_{true}(g) + \varepsilon \delta_m(g) \nonumber \\
    &=\sum_{\rho \in \hat{G}} \sqrt{\frac{d_\rho}{|G|}} \mathrm{Tr}[\hat{y}_{true}(\rho)\rho(g)] + \varepsilon \frac{1}{|G|}\sum_{\rho\in\hat{G}} d_\rho \mathrm{Tr}[\rho(m^{-1}g)] \nonumber \\
&=\sum_{\rho \in \hat{G}} \sqrt{\frac{d_\rho}{|G|}} \mathrm{Tr}\left[\left(\hat{y}_{true}(\rho)\rho(m) + \varepsilon \sqrt{\frac{d_\rho}{|G|}}\rho(e)\right)\rho(m^{-1}g)\right].\label{eq:y_ready_for_shift}
\end{align}

We are looking for the most likely prediction error $\varepsilon$ on the missing point, which is given by solving the equation:
\begin{align}
    \frac{\partial p(y)}{\partial\varepsilon} = \frac{\partial\exp(-\langle y, K^{-1}y\rangle)}{\partial\varepsilon} = 0.\label{eq:maxlikelihood_general}
\end{align}

An intermediate step is computing the inner product $\langle y, K^{-1}y\rangle$. We first need to find an expression for $K^{-1}$.
We start by noticing that, due to the stationarity of the kernel function, any one row of the kernel matrix is a \textit{group function}.
Indeed, for group elements $g,h\in G$, we have
\begin{align*}
    K_{gh} = k(g . x, h . x) = k(x, g^{-1}h . x) \triangleq \kappa_x(g^{-1}h).
\end{align*}
Thus, for a fixed element $x\in V$ (which we drop from the notation in the following), we can write a row of the kernel matrix in its Fourier form:
\begin{align*}
    K_{gh} = \kappa(g^{-1}h) = \sum_{\rho\in\hat{G}} \sqrt{\frac{d_\rho}{|G|}}\mathrm{Tr}[\hat{\kappa}(\rho)\rho(g^{-1})\rho(h)].
\end{align*}
We notice that, due to the symmetry of $K$, it must hold that
\begin{align*}
    \mathrm{Tr}[\hat{\kappa}(\rho)\rho(g^{-1})\rho(h)] = \mathrm{Tr}[\hat{\kappa}(\rho)\rho(h^{-1})\rho(g)],
\end{align*}
which is only possible when $\hat{\kappa}(\rho)$ is real and symmetric for all $\rho$.

Before we compute Eq.~\ref{eq:maxlikelihood_general}, we prove that the specific missing point $m$ can always be set to be the identity element $e$, by correctly permuting the labeling of $y$.
To do so, we introduce the shift operator $L_m$, which acts on group functions $f: G\to \C$ as
\begin{align*}
    L_m[f](g) = f(m^{-1}g).
\end{align*}
Its Fourier expression is
\begin{align*}
    L_m[f](g) = \sum_{\rho\in\hat{G}} \sqrt{\frac{d_\rho}{|G|}} \mathrm{Tr}[\widehat{L_m[f]}(\rho)\rho(g)]= \sum_{\rho\in\hat{G}}\sqrt{\frac{d_\rho}{|G|}}\mathrm{Tr}[\hat{f}(\rho)\rho(m^{-1}g)].
\end{align*}
By comparing this formula with Eq.~\ref{eq:y_ready_for_shift}, we recognize that
\begin{align*}
    y(g) = L_{m}[y_e](g) = y_e(m^{-1}g),
\end{align*}
where we introduce the labeling function $y_e$, which indexes the missing point by the group element $e$:
\begin{align*}
    y_e(g) = \sum_{\rho\in\hat{G}} \sqrt{\frac{d_\rho}{|G|}} \mathrm{Tr}\left[\left(\hat{y}_{true}(\rho)\rho(e) + \varepsilon\sqrt{\frac{d_\rho}{|G|}}\rho(e)\right)\rho(g)\right].
\end{align*}
Thus, we can always set $m=e$ by shifting the labeling function by $m^{-1}$.
The same shifting must be performed on the labels of the kernel matrix.
Indeed, if we look at the $m$-th row of the kernel matrix and shift all group elements by $m^{-1}$, we get
\begin{align*}
    K_{mh} = k(m.x, h.x) \to k(x, (m^{-1}h).x),
\end{align*}
which means that, after shifting, what used to be the $m$-th row is now the first row, and that the elements in the row are also shifted by $m^{-1}$.
We conclude that, without loss of generality, we can set up the problem so that the missing point is at group element $e$.
This is what we will do in the rest of the derivation, denoting
\begin{align}
    \hat{y}_e(\rho) = \hat{y}_{true}(\rho)+\varepsilon\sqrt{\frac{d_\rho}{|G|}}\rho(e).\label{eq:perturbed}
\end{align}

We now prove that we can block-diagonalize the kernel matrix by using a matrix composed of the unitary irreps.
Let us define the unitary matrix
\begin{align*}
    U_{g, (\rho,a,b)} = \sqrt{\frac{d_\rho}{|G|}}\rho_{ab}(g)
\end{align*}
to be the matrix whose columns are the values taken by the (matrix elements of the) irreps of the group.
We then compute $\Tilde{K} = U^*KU$, making use of the Schur orthogonality relations:
\begin{align*}
    &\Tilde{K}_{(\sigma, a, b), (\pi, c, d)} = 
    \sum_{g,h} U^*_{(\sigma, a, b), g}K_{gh}U_{h, (\pi, c, d)} \\
    &= \sum_{g,h} \sqrt{\frac{d_\sigma}{|G|}}\overline{\sigma}_{ab}(g) \left(\sum_{\rho\in\hat{G}} \sqrt{\frac{d_\rho}{|G|}}\mathrm{Tr}[\hat{\kappa}(\rho)\rho(h^{-1})\rho(g)]\right) \sqrt{\frac{d_\pi}{|G|}}\pi_{cd}(h) \\
    &= \sum_{g,h} \sqrt{\frac{d_\sigma}{|G|}}\overline{\sigma}_{ab}(g) \left(\sum_{\rho\in\hat{G}} \sqrt{\frac{d_\rho}{|G|}}\sum_{ijk} \hat{\kappa}(\rho)_{ij}\rho(h^{-1})_{jk}\rho(g)_{ki}\right) \sqrt{\frac{d_\pi}{|G|}}\pi_{cd}(h) \\
    &= \sum_{\rho\in\hat{G}} \sum_{ijk} \sqrt{\frac{d_\rho}{|G|}}\left(\sum_{g} \sqrt{\frac{d_\sigma}{|G|}}\overline{\sigma}_{ab}(g) \rho(g)_{ki}\right) \hat{\kappa}(\rho)_{ij}\left(\sum_h \sqrt{\frac{d_\pi}{|G|}}\rho(h^{-1})_{jk} \pi_{cd}(h) \right)\\
    &= \sum_{\rho\in\hat{G}} \sqrt{\frac{d_\rho d_\pi d_\sigma}{|G|^3}}\sum_{ijk} \frac{|G|}{d_\rho}\delta_{\sigma\rho}\delta_{bk}\delta_{ai} \hat{\kappa}(\rho)_{ij} \frac{|G|}{d_\rho} \delta_{\rho\pi}\delta_{cj}\delta_{kd}\\
    &= \sqrt{\frac{|G|}{d_\sigma}}\delta_{\sigma\pi}\delta_{bd}\hat{\kappa}(\sigma)_{ac}.
\end{align*}
The two Kronecker deltas highlight the doubly block-diagonal structure of the matrix.
Rewriting this equation in its matrix form makes this fact immediately clear:
\begin{align*}
   \Tilde{K} = \bigoplus_{\sigma\in\hat{G}} \sqrt{\frac{|G|}{d_\sigma}}\bigoplus_{i=1}^{d_\sigma} \hat{\kappa}(\sigma).
\end{align*}
This fact is useful because the inverse of a block-diagonal matrix is the block-diagonal matrix of the inverses.
Then,
\begin{align*}
    \Tilde{K}^{-1}_{(\sigma, a, b), (\pi, c, d)} = \sqrt{\frac{d_\sigma}{|G|}}\delta_{\sigma\pi}\delta_{bd}\hat{\kappa}^{-1}(\sigma)_{ac}.
\end{align*}

We can now compute the inner product in Eq.~\ref{eq:jointprob}.
For notational convenience, we temporarily write $\bullet_\star$ instead of $\bullet(\star)$ for quantity $\bullet$ indexed by irrep $\star$.
\begin{align*}
    &\langle y,K^{-1}y\rangle = \sum_{g, h} \overline{y(g)} K^{-1}_{gh}y(h)=\\
    &
    =\sum_{g,h} 
    \sum_{\rho\in\hat{G}} \sqrt{\frac{d_\rho}{|G|}} \mathrm{Tr}[\hat{y}_{e,\rho}^*\rho^*(g)]
    \sum_{\sigma\in\hat{G}}
    \sqrt{\frac{d_\sigma^3}{|G|^3}}\mathrm{Tr}[\hat{\kappa}^{-1}_\sigma \sigma(h^{-1})\sigma(g)] 
    \sum_{\pi\in\hat{G}} \sqrt{\frac{d_\pi}{|G|}} \mathrm{Tr}[\hat{y}_{e,\pi}\pi(h)] 
    \\
    &
    =\sum_{\rho,\sigma,\pi}
    \frac{d_\sigma}{|G|}\sqrt{\frac{d_\rho d_\sigma d_\pi}{|G|^3}}
    \sum_{g,h}
    \mathrm{Tr}[\hat{y}_{e,\rho}^*\rho^*(g)] 
    \mathrm{Tr}[\sigma(g)\hat{\kappa}^{-1}_\sigma\sigma(h^{-1})] 
    \mathrm{Tr}[\hat{y}_{e,\pi}\pi(h)]  \\
    &
    =\sum_{\rho,\sigma,\pi} \frac{d_\sigma}{|G|}\sqrt{\frac{d_\rho d_\sigma d_\pi}{|G|^3}} 
    \sum_{g,h}
    \sum_{a,b}[\hat{y}_{e,\rho}^*]_{ab}\rho^*(g)_{ba}
    \sum_{c,d,f}\sigma(g)_{cd}[\hat{\kappa}^{-1}_\sigma]_{df}\sigma(h^{-1})_{fc}
    \sum_{i,j}[\hat{y}_{e,\pi}]_{ij}\pi(h)_{ji}
    \\
    &
    =\sum_{\rho,\sigma,\pi} \frac{d_\sigma}{|G|}\sqrt{\frac{d_\rho d_\sigma d_\pi}{|G|^3}} 
    \sum_{a,b,c,d,f,i,j}
    \left(\sum_g \rho^*(g)_{ba}\sigma(g)_{cd}
    \sum_h \sigma(h^{-1})_{fc}\pi(g)_{ji}\right)
    [\hat{y}_{e,\rho}^*]_{ab}
    [\hat{\kappa}^{-1}_\sigma]_{df}
    [\hat{y}_{e,\pi}]_{ij}
    \\
    &
    =\sum_{\rho,\sigma,\pi} \frac{d_\sigma}{|G|}\sqrt{\frac{d_\rho d_\sigma d_\pi}{|G|^3}} 
    \sum_{a,b,c,d,f,i,j}
    \left(\frac{|G|}{d_\rho}\delta_{\rho\sigma}\delta_{bc}\delta_{ad}
    \frac{|G|}{d_\sigma}\delta_{\pi\sigma}\delta_{fj}\delta_{ci}\right)
    [\hat{y}_{e,\rho}^*]_{ab}
    [\hat{\kappa}^{-1}_\sigma]_{df}
    [\hat{y}_{e,\pi}]_{ij}
    \\
    &
    =\sum_{\sigma} \sqrt{\frac{d_\sigma}{|G|}} 
    \sum_{a,b,c,d,f,i,j}
    \delta_{bc}\delta_{ad}
    \delta_{fj}\delta_{ci}
    [\hat{y}_{e,\sigma}^*]_{ab}
    [\hat{\kappa}^{-1}_\sigma]_{df}
    [\hat{y}_{e,\sigma}]_{ij}
    \\&
    =\sum_{\sigma} \sqrt{\frac{d_\sigma}{|G|}} 
    \sum_{a,b,j}
    [\hat{y}_{e,\sigma}^*]_{ab}
    [\hat{\kappa}^{-1}_\sigma]_{aj}
    [\hat{y}_{e,\sigma}]_{bj}
    \\&
    =\sum_{\sigma} \sqrt{\frac{d_\sigma}{|G|}} 
    \mathrm{Tr}[
    \hat{y}_{e,\sigma}^*
    \hat{\kappa}^{-1}_\sigma
    \hat{y}_{e,\sigma}
    ].
    \end{align*}

We can now compute Eq.~\ref{eq:maxlikelihood_general}:
\begin{align*}
    \sum_\sigma \sqrt{\frac{d_\sigma}{|G|}} \left(\frac{d_\sigma}{|G|}\varepsilon \mathrm{Tr}[\hat{\kappa}^{-1}_\sigma] + \sqrt{\frac{d_\sigma}{|G|}}\mathrm{Tr}[\sigma^*(e)\hat{\kappa}^{-1}_\sigma\hat{y}_{true,\sigma}]\right) = 0
\end{align*}
Now, remembering that $\sigma(e)=\mathrm{Id}_{d_\sigma}$, we get the formula for the spectral error for a general finite group $G$.
We state the result as the following theorem:
\newpage

\begin{theorem}\label{thm:nonabelian}
    Let $G$ be a finite group, $y: G\to\C$ a labeling function on elements of a $G$-orbit, $k: G\times G\to\R$ a $G$-stationary kernel function that operates on pairs of points belonging to a $G$-orbit.
    The prediction error over element $e\in G$ by kernel regression is given by
    \begin{align}
        \varepsilon = -\frac{\sqrt{|G|}\sum_{\sigma\in\hat{G}} d_\sigma \mathrm{Tr}[\hat{\kappa}^{-1}_\sigma\hat{y}_{true, \sigma}^*] }{\sum_{\sigma\in\hat{G}} d_\sigma^{3/2}\mathrm{Tr}[\hat{\kappa}^{-1}_\sigma]}\label{eq:general_spectral_error},
    \end{align}
    where $\hat{G}$ denotes the set of all equivalent unitary irreps of $G$, $\sigma$ denotes an irrep and $d_\sigma$ its dimension, $\hat{\kappa}^{-1}_\sigma$ are the Fourier coefficients of the first row of the kernel matrix, and $\hat{y}_{true,\sigma}$ are the Fourier coefficients of the ground-truth labeling function.
\end{theorem}

\paragraph{Remark} The generalized spectral error formula differs from the cyclic one in that the alignment between labeling function and inverse kernel is measured as the Frobenius inner product of their generalized Fourier representations (which are stored in matrix coefficients) instead of a simple scalar product in the cyclic case between the Fourier coefficients of the labeling function and of the Gram matrix.

\begin{figure}[h]
    \centering
    \includegraphics{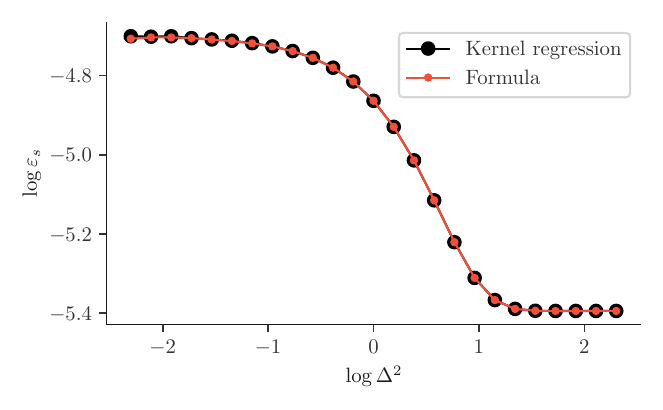}
    \caption{Comparison between the errors as given by Eq.~\ref{eq:general_spectral_error} and by standard numerical kernel regression. Used kernel: RBF. Slight discrepancies in the curves are induced by numerical precision errors in the matrix inversion operation required for standard kernel regression.}
    \label{fig:D4C2}
\end{figure}

\paragraph{Experiment} In Fig.~\ref{fig:D4C2}, we include an example of application of the generalized spectral error formula.
In this example, the group is the direct product $D_4\times C_2$, and it is acting on a 2x2 pixel random image to generate the dataset. The action of $D_4$ is composed of 90 degree rotations and flips of the image, while the $C_2$ action is a sign flip of all pixel values. The labeling (+1 and -1) follows the sign of the action of the $C_2$ subgroup.
We use the RBF kernel to predict the label on a missing point, as we vary the distance between the positive and the negative portions of the orbit, corresponding to the positive and negative actions of the $C_2$ subgroup.
We compute the prediction error using both our formula and standard kernel regression, and verify that they are equivalent.

\section{Architecture Details}\label{app:arch}
All the relevant code can be found at \href{https://github.com/Andrea-Perin/gpsymm}{\texttt{https://github.com/Andrea-Perin/gpsymm}}.
\subsection{Architectures Used in Fig. \ref{fig:symm_learning}}
For Fig.~\ref{fig:symm_learning}, we used the following architectures. We used the default initialization scheme provided by PyTorch (a version of Kaiming uniform). 
\begin{itemize}
    \item a MLP (the \texttt{Rearrange} layer is provided by the \href{https://einops.rocks}{\texttt{einops}} library): 
    \begin{lstlisting}[language=python]
    m = nn.Sequential(
        Rearrange('b 1 h w -> b (h w)'),
        nn.Linear(1*28*28, 512, bias=True),
        nn.ReLU(),
        nn.Dropout(),
        nn.Linear(512, 128, bias=True),
        nn.ReLU(),
        nn.Dropout(),
        nn.Linear(128, 10, bias=False),
        nn.LogSoftmax(dim=1)
    )
    \end{lstlisting}
    
    \item a ConvNet:
    \begin{lstlisting}[language=python]
    c = nn.Sequential(
        nn.Conv2d(1, 24, 5, 1),
        nn.MaxPool2d(kernel_size=2),
        nn.ReLU(),
        nn.Conv2d(24, 32, kernel_size=3),
        nn.MaxPool2d(kernel_size=2),
        nn.ReLU(),
        Rearrange('b c h w -> b (c h w)'),
        nn.Linear(800, 256),
        nn.ReLU(),
        nn.Linear(256, 10),
        nn.LogSoftmax(dim=1)
    )
    \end{lstlisting}
    
    \item a ViT-Simple (using the implementation provided by the library \href{https://github.com/lucidrains/vit-pytorch}{\texttt{vit-pytorch}}):
    \begin{lstlisting}[language=python]
    v = SimpleViT(
        image_size = 28,
        patch_size = 4,
        num_classes = 10,
        dim = 256,
        depth = 2,
        heads = 4,
        mlp_dim = 256,
        channels=1
    )
    \end{lstlisting}
\end{itemize}
The models were trained for 20 epochs using the Adam optimizer, using a learning rate of 1e-3 and $(\beta_1, \beta_2) = (0.7, 0.9).$
Further details on the implementation can be found in the provided code.

\subsection{Architectures Trained on Pairs of Orbits from MNIST}

For all tests involving two orbits from MNIST, we used the following MLP architecture, where the parameters \texttt{args.n\_hidden} defines the depth of the MLP (either 1 or 5).
\begin{lstlisting}[language=python]
W_std, b_std = 1., 1.
layer = nt.stax.serial(
    nt.stax.Dense(512, W_std=W_std, b_std=b_std),
    nt.stax.Relu(),
)
init_fn, apply_fn, kernel_fn = nt.stax.serial(
    nt.stax.serial(*([layer] * args.n_hidden)),
    nt.stax.Dense(1, W_std=W_std, b_std=b_std)
)
\end{lstlisting}

For all tests involving two orbits from MNIST, we used the following Convolutional architecture, where the parameter \texttt{args.kernel\_size} defines the size of the conv kernel, and \texttt{IS\_GAP} defines whether or not to include the \texttt{GlobalAvgPool} layer.
\begin{lstlisting}[language=python]
W_std, b_std = 1., 1.
conv = nt.stax.serial(
    nt.stax.Conv(
        out_chan=64,
        filter_shape=(args.kernel_size, args.kernel_size),
        padding='CIRCULAR',
        W_std=W_std, b_std=None),
    nt.stax.Relu()
)
pool = nt.stax.GlobalAvgPool() if IS_GAP else nt.stax.Identity()
init_fn, apply_fn, kernel_fn = nt.stax.serial(
    conv,
    pool,
    nt.stax.Flatten(),
    nt.stax.Dense(1, W_std=W_std, b_std=None)
)
\end{lstlisting}

\subsection{Architecture Trained on Multiple Seeds and Multiple Classes of Rotated-MNIST}\label{app:multiclass_arch}
For Section~\ref{sec:multiclass}, we use the following network architecture and optimizer.
Note how this architecture is not in principle trainable with a cross-entropy loss, as it outputs a scalar value.
This is needed for the computation of our kernel function.
The actual training, which indeed is based on the cross-entropy loss, involves \emph{all but the last two layers}, i.e., the last \texttt{Relu} and \texttt{Dense} layers.
This is done by basic model surgery techniques, made possible by the simple sequential structure of the model.
\begin{lstlisting}[language=python]
def net_maker(
    W_std: float = 1.,
    b_std: float = 1.,
    dropout_rate: float = 0.5,
    mode: str = 'train'
):
    return nt.stax.serial(
        nt.stax.Dense(512, W_std=W_std, b_std=b_std),
        nt.stax.Relu(),
        nt.stax.Dropout(rate=dropout_rate, mode=mode),
        nt.stax.Dense(128, W_std=W_std, b_std=b_std),
        nt.stax.Relu(),
        nt.stax.Dropout(rate=dropout_rate, mode=mode),
        nt.stax.Dense(10, W_std=W_std, b_std=None),
        nt.stax.Relu(),
        nt.stax.Dense(1, W_std=W_std, b_std=b_std)
    )

optim = optax.adam(learning_rate=1e-3, b1=0.7, b2=0.9)
\end{lstlisting}

\section{Extension of the Spectral Theory to Equivariant Architectures}\label{app:equi}

We study how the symmetries of the dataset interact with equivariant architectures.  We limit our analysis to the Neural Network Gaussian Process (NNGP), as the Neural Tangent Kernel (NTK) adds extra terms that would cloud the derivations. Essentially, the same proofs should hold for the NTK.

We distinguish two cases: (1) when the neural network is equivariant to the symmetry of interest; (2) when the neural network is equivariant, but not to the symmetry of interest. The number of possible symmetries to consider is vast, as well as the number of possible equivariant architectures. Here, we study the interplay between equivariance and symmetries in a limited setup: we consider spatially convolutional neural networks, and we study their interplay with a dataset with translation symmetry (the symmetry to which the network is equivariant) and with a dataset with rotational symmetry. These examples should give the reader an intuition as to how these two things interact in general.

 \paragraph{Dataset with translational symmetry.}
We consider a dataset composed of seed images and all their translations.
For the purpose of the proofs to follow, we will focus on a single orbit of this dataset, which consists of a single seed point $x_s \in \R^{n \times n} $ and all of its translations:
\begin{align*}
    \mathcal{O}_T &= \{ g_T^0.x_s, g_T^1.x_s, \cdots g_T^{n-1}.x_s\},
\end{align*}
where the translation operator $g_T$ acts on images by circularly shifting them along one of the dimensions.
For pixel coordinates $(i_x,i_y)$, and the corresponding value of the pixel $x(i_x, i_y)$, we write:
\begin{align*}
    g_T.x\left(\begin{bmatrix}
            i_x \\
            i_y
    \end{bmatrix}\right) = x\left(\begin{bmatrix}
            (i_x + 1) \bmod n\\
            i_y
    \end{bmatrix}\right).
\end{align*}

\paragraph{Dataset with rotational symmetry.}
We consider a dataset composed of seed images and all their rotations in $C_4$ (we limit ourselves to 4 cardinal rotations to avoid definitional problems of image rotation on discrete pixel grids).
We will focus on a single orbit of this dataset, which consists in a single seed point $x_s \in \R^{n \times n} $ and all its rotations: 
\begin{align*}
    \mathcal{O}_R &= \{ g_R^0.x_s, g_R^1.x_s,  g_R^{2}.x_s, g_R^{3}.x_s\}
\end{align*}
where the rotation operator $g_R$ permutes pixel coordinates $(i_x,i_y)$ as follows:

\begin{align*}
    g_R. x\left(\begin{bmatrix}
            i_x \\
            i_y
    \end{bmatrix}\right) = x\left(\begin{bmatrix}
            i_y\\
            n-i_x
    \end{bmatrix}\right).
\end{align*}

We also distinguish between two types of convolutional architectures: (1) convolutional architectures where the last layer is fully connected: these architectures do not ensure full invariance to translation; (2) convolutional architectures where the last layer does global average pooling, ensuring invariance to translation.

\paragraph{Fully connected convolutional network (FC).} We consider a fully connected convolutional network with one hidden layer, filters of size 3x3, circular padding and stride of 1.
The network \( f : \R^{n \times n} \to \R \) is parameterized by:
\begin{align*}
f_\text{FC}(x) = A \frac{1}{\sqrt{k}} \phi(B \circledast x)_v,
\end{align*}
where \( A \in \R^{1 \times n^2k} \), \( B \in \R^{k \times 1 \times 3 \times 3} \), \(\circledast\) denotes the spatial convolution operation, and for any matrix \( u \in \R^{n \times n} \), \( u_v \in \R^{n^2} \) denotes the vectorization (i.e., flattening) of \( u \).
This network first applies a convolutional layer to the data, then flattens the resulting representation into a vector, and passes it through a fully connected layer.

Let \( K \in \mathbb{R}^{n^2 \times n^2} \) denote the Conv-NNGP for a 1-layer fully convolutional network operating on a pair of images \( x, x' \) of size \( n \times n \). This kernel contains an entry for every coordinate quadruplet \( (i_x, i_y, i_x', i_y') \) across the images \( x, x' \). On the other hand, the Conv-NNGP of the network above \(f_\text{FC}\) reduces this representation to a kernel of size \( K_\text{FC} \in \mathbb{R} \), by letting \( K_\text{FC} = \operatorname{Tr}(K) \) \citep{arora2019exact}: 

\begin{align*}
    K_\text{FC}(x,x') &= \sum_{i_x, i_y} K_{i_x, i_y, i_x, i_y}(x,x')\\
    &= \sum_{i_x, i_y} \sum_k \Big\langle
    \phi(B_{k,i_x,i_y} \cdot x), \phi(B_{k,i_x,i_y} \cdot x')
    \Big\rangle_\Theta
\end{align*}
where $k$ denotes the filter index, and $\langle.,.\rangle_\Theta$ the average dot product of the embeddings over the randomly sampled weights of the models.

\paragraph{Global Average Pooling convolutional network (GAP).} We consider a convolutional network with global average pooling at the last layer.
This network is invariant to discrete translations.
We consider the network \( f_\text{GAP} : \R^{n \times n} \to \R \) be parameterized by:
\begin{align*}
f_\text{GAP}(x) = \frac{1}{\sqrt{k}n^2} \sum_{k} \sum_{i_x} \sum_{i_y} A_{1 k} \phi(B_{k,i_x,i_y} \cdot x)
\end{align*}
where \( A \in \R^{1 \times k} \), and \( B \in \R^{k \times 1 \times 3 \times 3} \) with \( B_k \in \R^{1 \times 1 \times 3 \times 3} \) indexing filter \( k \) of \( B \).
$B_{k, i_x, i_y} \in \R^{1\times 1 \times n\times n}$ is obtained by centering $B_k$ at coordinates $(i_x, i_y)$ of an $n\times n$ grid with periodic boundary conditions, and filling the remaining entries with zeros.
Then, the dot product is understood as the sum of the elementwise multiplications of all entries of $x$ and $B_{k, i_x, i_y}$.
We remark that this operation is effectively an alternative way of describing a convolution with filter bank $B$, but this indexing choice proves useful in the proofs.
After applying a convolutional layer to the data, this network averages the resulting representation across each of the \( k \) output channels, and then takes a linear combination of these averages using a fully connected layer.

The Conv-NNGP of \( f_\text{GAP} \) reduces \( K \) to \( K_\text{GAP} \) by letting \( K_\text{GAP} = \frac{1}{n^4} \sum_{i_x,i_y,i_x',i_y'} K_{i_x,i_y,i_x',i_y'} \) (i.e., averaging over all elements of \( K \)). 
In the following, we omit this prefactor $1/n^4$ for notational convenience.
\begin{align*}
    K_\text{GAP}(x,x') 
    &= \sum_{\substack{i_x, i_y \\ i_x', i_y'}} K_{i_x,i_y,i_x',i_y'}(x,x')  \\
    &= \sum_{\substack{i_x, i_y \\ i_x', i_y'}} \sum_k \Big\langle
    \phi(B_{k,i_x,i_y} \cdot x), \phi(B_{k,i_x',i_y'} \cdot x')
    \Big\rangle_\Theta
\end{align*}

\setcounter{theorem}{10}

\begin{proposition}
The kernel matrix of a fully connected convolutional network $K_\text{FC}$ over a translation orbit $O_{T}$ is circulant. Moreover, this kernel matrix is in general not constant or rank-deficient.
\end{proposition}

\begin{proof}

We first show that the kernel of a fully connected network is circulant over an orbit of image rotations. 
Consider a pair of images from this orbit \( (x, x') \in O_T \). We show below that applying the same image rotation $g_T$ to both these images leaves the kernel unaffected, which is equivalent to showing that the kernel matrix is circulant over the orbit $O_T$. First, we write:

\begin{align*}
    K_\text{FC}(g_T.x,g_T.x') 
    &= \sum_{i_x, i_y} \sum_k \Big\langle
    \phi(B_{k,i_x,i_y} \cdot g_T.x), \phi(B_{k,i_x,i_y} \cdot g_T.x)
    \Big\rangle_\Theta \\
    &= \sum_{i_x, i_y} \sum_k \Big\langle
    \phi(g_T^{-1}.B_{k,i_x,i_y} \cdot x), \phi(g_T^{-1}.B_{k,i_x,i_y} \cdot x)
    \Big\rangle_\Theta
\end{align*}

Applying a translation to image $x$ is equivalent to applying a change of spatial index to the convolutional filters, $g_T^{-1}.B_{k,i_x,i_y}$ =  $B_{k,i_x-1 \bmod n,i_y}$, such that:
\begin{align*}
    K_\text{FC}(g_T\cdot x,g_T\cdot x') 
    &= \sum_{i_x, i_y} \sum_k \Big\langle
    \phi(B_{k,i_x -1 \bmod n} \cdot x), \phi(B_{k,i_x-1 \bmod n,i_y}  \cdot x)
    \Big\rangle_\Theta
\end{align*}

We operate the changes of variable with renaming $i_x \gets i_x -1 \bmod n$: 
\begin{align*}
    K_\text{FC}(g_T.x,g_T.x') &= \sum_{i_x, i_y} \sum_k \Big\langle
    \phi(B_{k,i_x,i_y} \cdot x), \phi(B_{k,i_x,i_y} \cdot x)
    \Big\rangle_\Theta\\
    &= K_\text{FC}(x,x')
\end{align*}

This concludes our proof that the kernel of a fully connected convolutional network is circulant over an orbit of image translations.

Another question of interest is whether the kernel is constant over pairs of images belonging to the same orbit. Consider an image $x$ and its transformation $g_T.x$.
\begin{align*}
    K_\text{FC}(x,g_T.x) 
    &= \sum_{i_x, i_y} \sum_k \Big\langle
    \phi(B_{k,i_x,i_y} \cdot x), \phi(B_{k,i_x,i_y} \cdot g_T.x)
    \Big\rangle_\Theta\\
    &= \sum_{i_x, i_y} \sum_k \Big\langle
    \phi(B_{k,i_x,i_y} \cdot x), \phi(g_T^{-1}.B_{k,i_x,i_y} \cdot x)
    \Big\rangle_\Theta\\
    &= \sum_{i_x, i_y} \sum_k \Big\langle
    \phi(B_{k,i_x,i_y} \cdot x), \phi(B_{k,i_x-1 \bmod n,i_y} \cdot x)
    \Big\rangle_\Theta
\end{align*}

It is clear that there is no change of variable that could make this kernel equal to $K_\text{FC}(x,x)$. In general, 
\begin{align*}
    K_\text{FC}(x,g_T.x) \neq K_\text{FC}(x,x)
\end{align*}

In other words, the fully connected convolutional network does not in general produce invariance to the translation transformation. Its kernel over an orbit is in general not constant.

To prove that $K_\text{FC}$ is not, in general, rank-deficient, one can consider a seed image where only one pixel is active, and consider shifts in the images of size larger than the filter's size.
Doing so results in a kernel matrix over the orbit that is a multiple of the identity, which is full rank.

\end{proof}

\begin{proposition}
 The kernel matrix of a global average pooling convolutional network $K_\text{GAP}$ over a translation orbit $O_{T}$ is constant.
\end{proposition}
\begin{proof}

We now study how a global average pooling layer affects the kernel. We consider two successive images from the translation orbit $(x, g_T.x) \in O_T$: 
\begin{align*}
    K_\text{GAP}(x,g_T.x) &= \sum_{\substack{i_x, i_y \\ i_x', i_y'}} \sum_k \Big\langle
    \phi(B_{k,i_x,i_y} \cdot x), \phi(B_{k,i_x',i_y'} \cdot (g_T.x))
    \Big\rangle_\Theta \\
    &= \sum_{\substack{i_x, i_y \\ i_x', i_y'}} \sum_k \Big\langle
    \phi(B_{k,i_x,i_y} \cdot x), \phi(g_T^{-1}.B_{k,i_x',i_y'} \cdot (x))
    \Big\rangle_\Theta \\
    &= \sum_{\substack{i_x, i_y \\ i_x', i_y'}} \sum_k \Big\langle
    \phi(B_{k,i_x,i_y} \cdot x), \phi(B_{k,i_x' -1 \bmod n,i_y'} \cdot (x))
    \Big\rangle_\Theta
\end{align*}

By the change of variable with renaming $i_x' \gets i_x'-i \mod n$, we recover: 
\begin{align*}
    K_\text{GAP}(x,g_T.x) &= \sum_{\substack{i_x, i_y \\ i_x', i_y'}} \sum_k \Big\langle
    \phi(B_{k,i_x,i_y} \cdot x), \phi(B_{k,i_x',i_y'} \cdot x)
    \Big\rangle_\Theta \\
    &= K_\text{GAP}(x,x)
\end{align*}

The kernel of a global average pooling convolutional network computed over a translation orbit is thus constant.

In our spectral error formula, the kernel over a pair of orbits will thus be rank-deficient (each orbit-wise block being constant), such that some inverse eigenvalues in the denominator will diverge. The spectral error will thus go to 0 (perfect generalization).
We recover the well-known fact that a global average pooling convolutional network is invariant to translations. 

\end{proof}

\begin{proposition}
The kernel matrix of a fully connected convolutional network $K_\text{FC}$ over a rotation orbit $O_{R}$ is circulant, but in general not rank-deficient or constant. 
\end{proposition}

\begin{proof}
Now we show that the kernel of a fully connected convolutional network is circulant over an orbit of image rotations. Consider a pair of images from this orbit \( (x, x') \in O_R \). We show next that applying the same image rotation $g_R$ to both these images leaves the kernel unaffected: 
\begin{align*}
    K_\text{FC}(g_R.x,g_R.x') 
    &= \sum_{i_x, i_y} \sum_k \Big\langle
    \phi(B_{k,i_x,i_y} \cdot g_R.x), \phi(B_{k,i_x,i_y} \cdot g_R.x)
    \Big\rangle_\Theta \\
    &= \sum_{i_x, i_y} \sum_k \Big\langle
    \phi(g_R^{-1}.B_{k,i_x,i_y} \cdot x), \phi(g_R^{-1}.B_{k,i_x,i_y} \cdot x)
    \Big\rangle_\Theta
\end{align*}

Applying a rotation to image $x$ is equivalent to applying a change of index to the convolutional filters $B_{k,i_x,i_y}$ jointly with rotating the filters. We denote $B'_{k,n-i_y,i_x}$ = $g_R^{-1}.B_{k,i_x,i_y}$.
\begin{align*}
    K_\text{FC}(g_R\cdot x,g_R\cdot x') 
    &= \sum_{i_x, i_y} \sum_k \Big\langle
    \phi(B'_{k,n-i_y,i_x} \cdot x), \phi(B'_{k,n-i_y,i_x}  \cdot x)
    \Big\rangle_\Theta
\end{align*}

We operate the changes of variable with renaming $i_x \gets n-i_y$, $i_y \gets i_x$. 
\begin{align*}
    K_\text{FC}(g_R.x,g_R.x') &= \sum_{i_x, i_y} \sum_k \Big\langle
    \phi(B_{k,i_x,i_y}' \cdot x), \phi(B_{k,i_x,i_y}' \cdot x)
    \Big\rangle_\Theta
\end{align*}

Assuming that the weights $\Theta$ of the filters $B_k$ are drawn from a random normal distribution \( B_k \sim \mathcal{N}(\mathbf{0}, \mathbf{I}) \), then $B_k'=g_R^{-1}.B_k$ will have the same distribution of weights \( B_k' \sim \mathcal{N}(\mathbf{0}, \mathbf{I}) \). Thus:
\begin{align*}
    K_\text{FC}(g_R.x,g_R.x') &= K_\text{FC}(x,x')
\end{align*}

This concludes our proof that the kernel of a fully connected convolutional network is circulant over an orbit of image rotations.

Another question of interest is whether the kernel is constant or rank-deficient over pairs of images belonging to the same orbit. Consider an image $x$ and its transformation $g_R.x$.
\begin{align*}
    K_\text{FC}(x,g_R.x) 
    &= \sum_{i_x, i_y} \sum_k \Big\langle
    \phi(B_{k,i_x,i_y} \cdot x), \phi(B_{k,i_x,i_y} \cdot g_R.x)
    \Big\rangle_\Theta\\
    &= \sum_{i_x, i_y} \sum_k \Big\langle
    \phi(B_{k,i_x,i_y} \cdot x), \phi(g_R^{-1}.B_{k,i_x,i_y} \cdot x)
    \Big\rangle_\Theta\\
    &= \sum_{i_x, i_y} \sum_k \Big\langle
    \phi(B_{k,i_x,i_y} \cdot x), \phi(B'_{k,n-i_y,i_x} \cdot x)
    \Big\rangle_\Theta
\end{align*}

There is no change of variable that could make this kernel equal to $K_\text{FC}(x,x)$. In general, 
\begin{align*}
    K_\text{FC}(x,g_R.x) \neq K_\text{FC}(x,x)
\end{align*}

In other words, the fully connected network does not produce invariance to the rotation transformation. Its kernel over an orbit is in general not constant.

A similar argument as in the case of shifts can be made to prove that the kernel is not, in general, rank-deficient.
This can be proven by choosing a ``Dirac-like" image whose only active pixel is not at the center of the image.
\end{proof}

\begin{proposition}
The kernel matrix of a global average pooling convolutional network $K_\text{GAP}$ over a rotation orbit $O_{R}$ is circulant, but in general not rank-deficient or constant.
\end{proposition}

\begin{proof}
We study how changing the last layer from a fully connected to a global average pooling layer affects the results of the previous proposition.  
\begin{align*}
    K_\text{GAP}(g_R.x,g_R.x') 
    &= \sum_{\substack{i_x, i_y \\ i_x', i_y'}} \sum_k \Big\langle
    \phi(B_{k,i_x,i_y} \cdot g_R.x), \phi(B_{k,i_x',i_y'} \cdot g_R.x')
    \Big\rangle_\Theta \\
    &= \sum_{\substack{i_x, i_y \\ i_x', i_y'}} \sum_k \Big\langle
    \phi(g_R^{-1}.B_{k,i_x,i_y} \cdot x), \phi(g_R^{-1}.B_{k,i_x',i_y'} \cdot x')
    \Big\rangle_\Theta \\
    &= \sum_{\substack{i_x, i_y \\ i_x', i_y'}} \sum_k \Big\langle
    \phi(B'_{k,n-i_y,i_x} \cdot x), \phi(B'_{k,n-i_y',i_x'} \cdot x')
    \Big\rangle_\Theta
\end{align*}
where we denote $B'_{k,n-i_y,i_x}$ = $g_R^{-1}.B_{k,i_x,i_y}$ the filter rotated and displaced by the rotation. 
Changing variables as $i_x \gets n-i_y$, $i_y \gets i_x$, $i_x' \gets n-i_y'$, $i_y' \gets i_x'$, we get
\begin{align*}
    K_\text{GAP}(g_R.x,g_R.x') 
    &= \sum_{\substack{i_x, i_y \\ i_x', i_y'}} \sum_k \Big\langle
    \phi(B'_{k,i_x,i_y} \cdot x), \phi(B'_{k,i_x',i_y'} \cdot x')
    \Big\rangle_\Theta \\
    &= K_\text{GAP}(x,x') 
\end{align*}
The kernel of this network is thus circulant on a single orbit, as before with the fully connected convolutional network. 

We now compute the kernel value for pairs of images belonging to the same orbit, in order to check whether the kernel is constant over an orbit:

\begin{align*}
    K_\text{GAP}(x,g_R.x) 
    &= \sum_{\substack{i_x, i_y \\ i_x', i_y'}} \sum_k \Big\langle
    \phi(B_{k,i_x,i_y} \cdot x), \phi(B_{k,i_x',i_y'} \cdot g_R.x)
    \Big\rangle_\Theta \\
    &= \sum_{\substack{i_x, i_y \\ i_x', i_y'}} \sum_k \Big\langle
    \phi(B_{k,i_x,i_y} \cdot x), \phi(g_R^{-1}.B_{k,i_x',i_y'} \cdot x)
    \Big\rangle_\Theta \\
    &= \sum_{\substack{i_x, i_y \\ i_x', i_y'}} \sum_k \Big\langle
    \phi(B_{k,i_x,i_y} \cdot x), \phi(B'_{k,n-i_y',i_x'} \cdot x)
    \Big\rangle_\Theta
\end{align*}
Changing variables as $i_x' \gets n-i_y'$, $i_y' \gets i_x'$, we get:
\begin{align*}
    K_\text{GAP}(x,g_R.x) 
    &= \sum_{\substack{i_x, i_y \\ i_x', i_y'}} \sum_k \Big\langle
    \phi(B_{k,i_x,i_y} \cdot x), \phi(B'_{k,i_x',i_y'} \cdot x)
    \Big\rangle_\Theta \\
    &\neq K_\text{GAP}(x,x) 
\end{align*}

In general, the kernel is not constant over an orbit.

To prove that this kernel over an orbit is in general not rank deficient, one can resort to using a seed image that is the sum of two ``Dirac-like'' images, each with a single nonzero pixel, and both of them in an area contained within the filter size.
\end{proof}

\section{Varying the Number of Angles in an Orbit}\label{app:many_angles}
We report here plots comparing spectral and exact NTK errors for varying number of rotations for a 1 hidden layer MLP. 

\begin{figure}[H]
    \centering
    \includegraphics[width=0.8\linewidth]{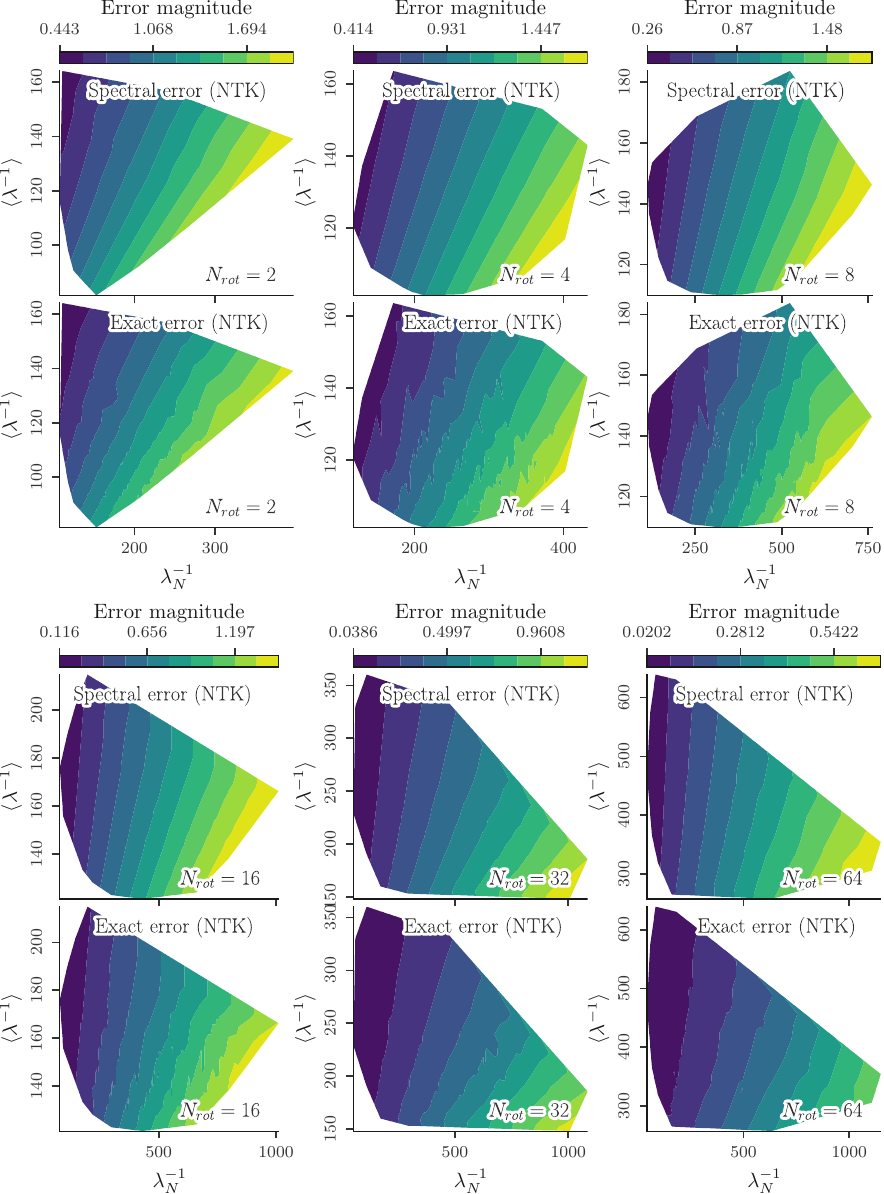}
    \caption{Comparison of spectral and exact NTK errors across a range of values for $\lambda_N$ and $\langle \lambda^{-1}\rangle$.}
    \label{fig:all_errors}
\end{figure}

\section{Further MLP Analyses}\label{app:more_mlps}
We report here plots of the NTK analysis for a deeper MLP (5 hidden layers, Fig.~\ref{fig:mlp_5}), and a MLP that is trained via Adam (Fig.~\ref{fig:mlp_train}).
\begin{figure}[H]
    \centering
    \includegraphics[width=0.85\linewidth]{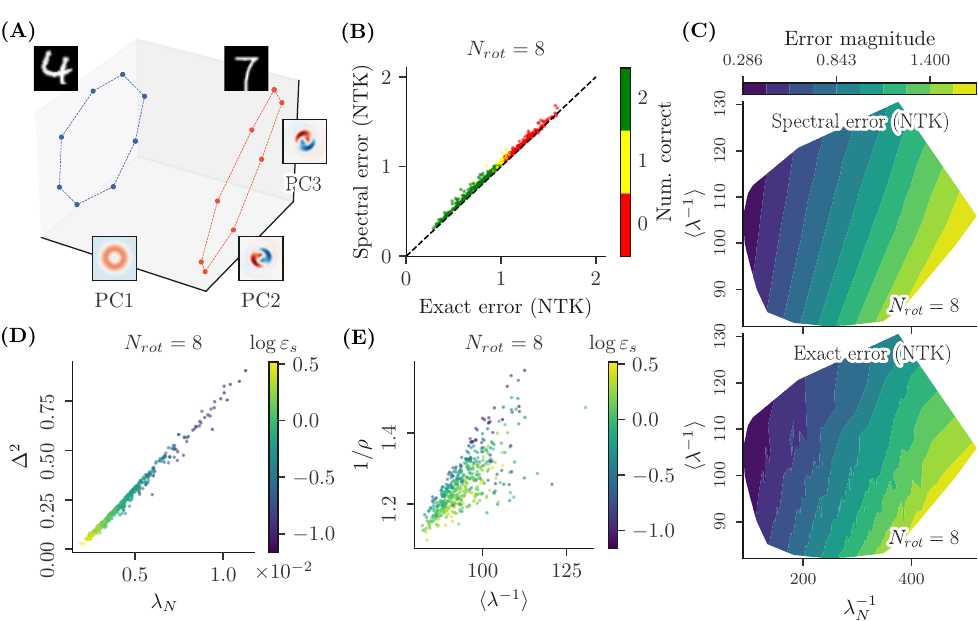}
    \caption{MLP, 5 hidden layers.}
    \label{fig:mlp_5}
\end{figure}

\begin{figure}[H]
    \centering
    \includegraphics[width=0.85\linewidth]{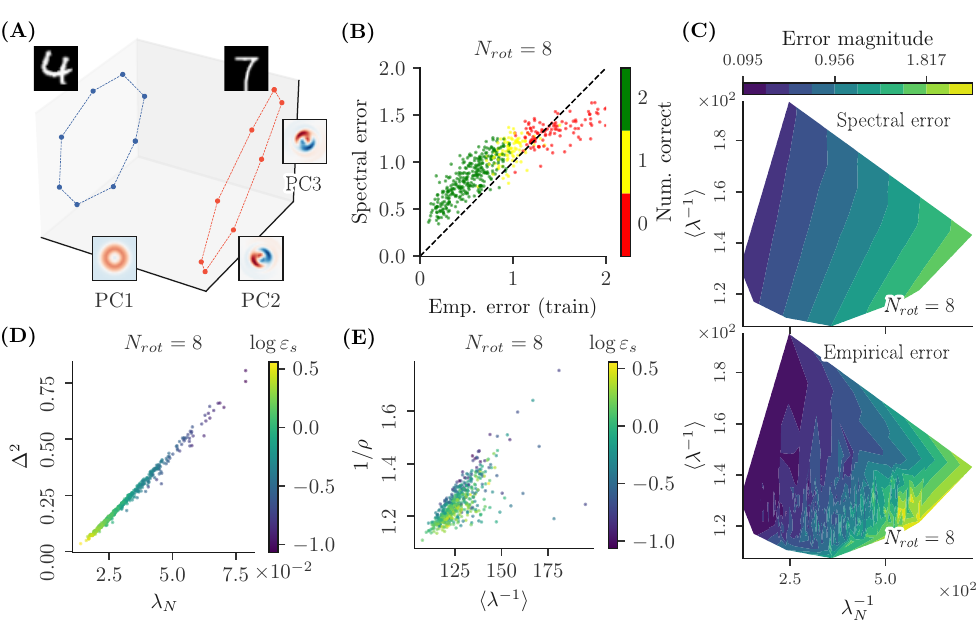}
    \caption{Trained MLP.}
    \label{fig:mlp_train}
\end{figure}

\section{Multiple Seeds, Multiple Classes - Additional Figures}\label{app:multi_plots}
\begin{figure}[H]
    \centering
    \includegraphics[width=\linewidth]{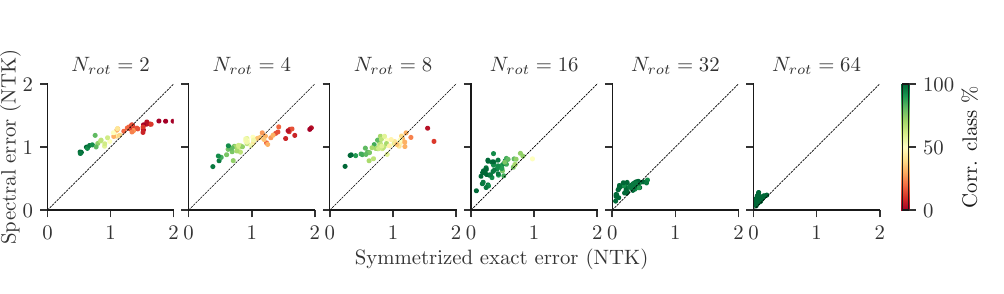}
    \caption{\textbf{Comparison of spectral and exact NTK errors for the multiple seeds case.}
    We compare the spectral error, averaged over all the pairings of orbits for a given number of pairs, and the symmetrized NTK prediction error, over a range of number of rotations in the orbits, $N_{rot}$.
    We superimpose the bisector as a visual reference.
    The color coding reflects the percentage of seeds (of both classes) for which the NTK regression gives a correct prediction, understood as agreeing with the +1 label of the missing points.
    As the number of points in the orbits increases, both errors decrease as expected, and the percentage of correct NTK predictions increases.
    }
    \label{fig:multiseed}
\end{figure}

\begin{figure}[H]
    \centering
    \includegraphics[width=\linewidth]{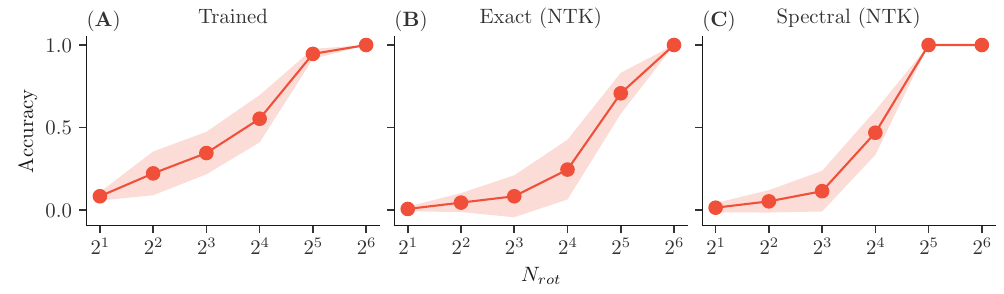}
    \caption{
    \textbf{Our spectral method correctly reproduces the generalization behavior of a trained finite-width MLP on rotated-MNIST, as we vary the number of sampled angles.}
    On a version of rotated-MNIST comprising all 10 classes and 13 seed images per class, we compare (A) a cross-entropy trained MLP (2 hidden layers) (B) a one-versus-all NTK regression strategy for this same architecture (we exclude one angle from one class and use NTK regression on the missing points, assuming label +1 for this class and label -1 for all other classes, repeat over all leave-out classes and average classification accuracy over all missing points), and (C) our multi-class adapted spectral error (see text).
    As the number of sampled angles in the orbits increases, the accuracies of all methods increase similarly and gradually on the missing points, showing that no mechanism for symmetry learning is present for the finite-width network that would be missing from our spectral theory, or from simple NTK regression.
    }
    \label{fig:multiclass}
\end{figure}

\section{Further Convolutional NTK (CNTK) Analyses}\label{app:more_cntks}
\subsection{Rotation Orbits}\label{app:more_cntks_rots}

\begin{figure}[H]
    \centering
    \includegraphics[width=0.85\linewidth]{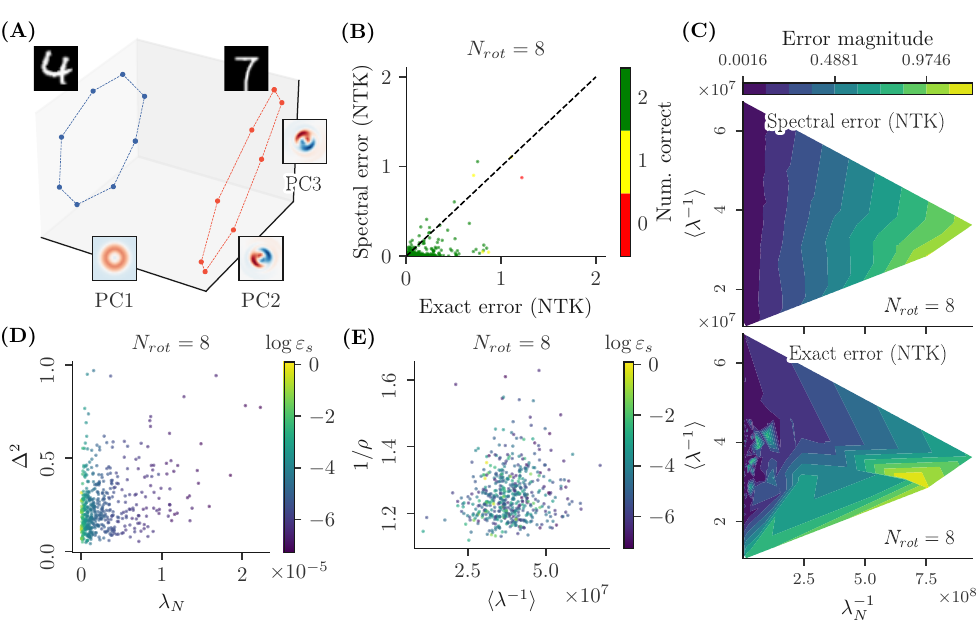}
    \caption{CNTK, Global Average Pooling, on rotation orbit.}
    \label{fig:cntk_gap}
\end{figure}

\begin{figure}[H]
    \centering
    \includegraphics[width=0.85\linewidth]{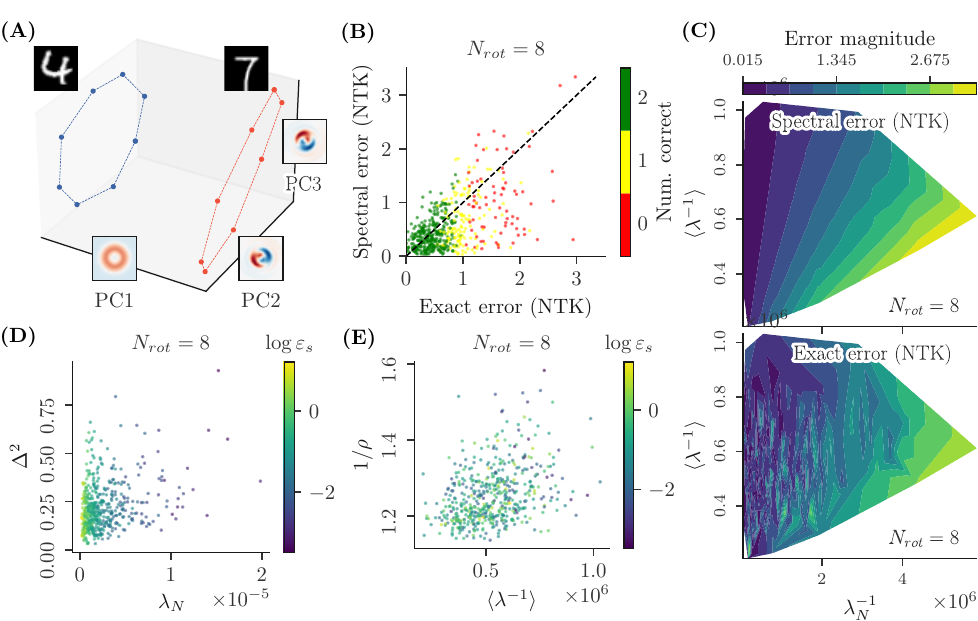}
    \caption{CNTK, Global Average Pooling, on rotation orbit. Kernel size (4, 4), strides (4, 4).}
    \label{fig:cntk_gap_fail}
\end{figure}

\subsection{Translation Orbits}\label{app:more_cntks_shifts}
\begin{figure}[H]
    \centering
    \includegraphics[width=0.85\linewidth]{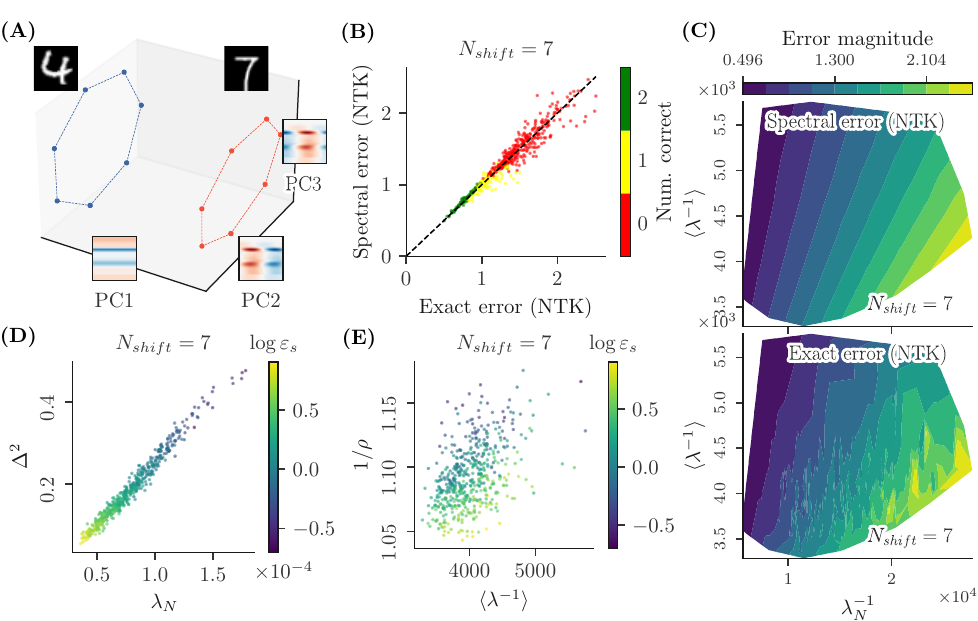}
    \caption{CNTK, Fully Connected, on translation orbit.}
    \label{fig:cntk_fc_shift}
\end{figure}

\begin{figure}[H]
    \centering
    \includegraphics[width=0.85\linewidth]{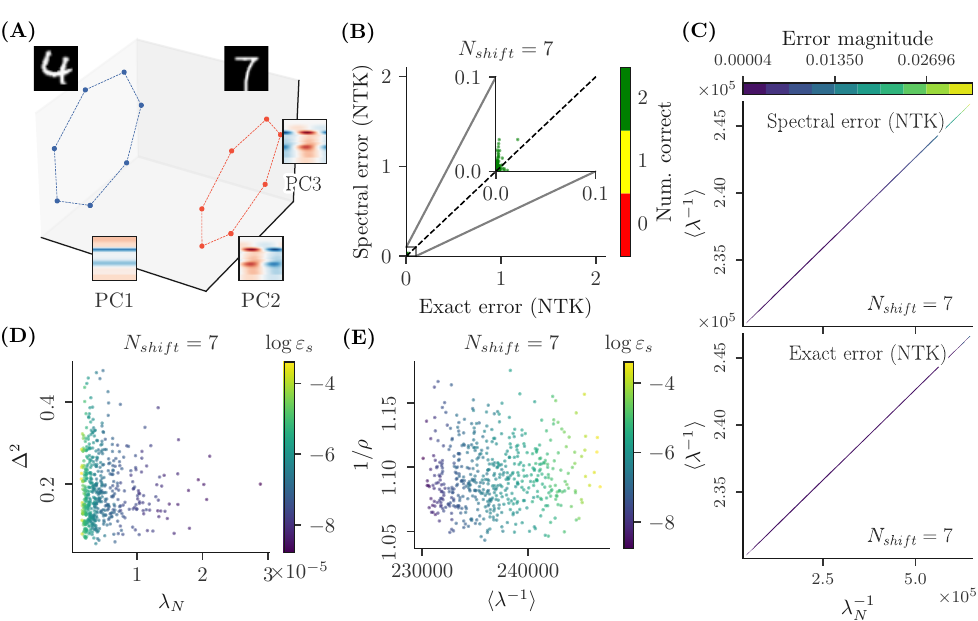}
    \caption{CNTK, Global Average Pooling, on translation orbit.}
    \label{fig:cntk_gap_shift}
\end{figure}

\end{document}